%% file: main.tex
\documentclass[10pt]{article} 

\usepackage[preprint]{tmlr}

\input{math_commands.tex}

\usepackage{hyperref}
\usepackage{url}

\usepackage[utf8]{inputenc}
\usepackage[T1]{fontenc}

\usepackage{amsthm}
\usepackage{mathtools}
\usepackage{bm}
\usepackage{nicefrac}

\usepackage{graphicx}
\usepackage{tikz}
\usepackage{subcaption}

\usepackage{booktabs}
\usepackage{tabularx}

\usepackage{algorithm}
\usepackage{algorithmic}
\usepackage{enumitem}
\usepackage{amssymb}   
\usepackage{xcolor}
\usepackage{soul}
\usepackage{microtype}

\usepackage{hyperref}
\usepackage{url}
\usepackage[capitalize,noabbrev]{cleveref}

\newtheorem{theorem}{Theorem}[section]
\newtheorem{proposition}[theorem]{Proposition}
\newtheorem{lemma}[theorem]{Lemma}
\newtheorem{corollary}[theorem]{Corollary}

\theoremstyle{remark}
\newtheorem{remark}[theorem]{Remark}

\DeclareMathOperator{\Id}{Id}

\newcommand{\Pc}{\mathcal{P}}

\newcommand{\sg}{\mathrm{sg}}

\title{Generative Drifting is Secretly Score Matching: a Spectral and Variational Perspective}

\author{\name Erkan Turan \email turan@lix.polytechnique.fr \\
      LIX, Ecole Polytechnique, IP Paris
      \AND
      \name Nicolas Dufour \email nicolas.dufour@kyutai.org \\
       Kyutai
      \AND
      \name Maks Ovsjanikov \email maks@lix.polytechnique.fr\\
      \addr LIX, Ecole Polytechnique, IP Paris
     }

\def\month{MM}  
\def\year{YYYY} 
\def\openreview{\url{https://openreview.net/forum?id=XXXX}} 

\begin{document}

\maketitle
\let\AND\relax
\begin{abstract}
Generative Modeling via Drifting~\citep{deng2026drifting} has recently achieved
state-of-the-art one-step image generation through a kernel-based drift operator, yet the success is largely empirical and its
theoretical foundations remain poorly understood. In this paper, we make the following observation: \emph{under a Gaussian kernel, the drift
operator is exactly a score difference on smoothed distributions}. 
This insight allows us to answer all three key questions, which were left open in the original work: (1) whether a vanishing drift guarantees equality of distributions
($V_{p,q}=0\Rightarrow p=q$), (2) how to choose between kernels, and (3) why the
stop-gradient operator is indispensable for stable training.   Our observations position drifting within the well-studied score-matching family and enable a rich theoretical perspective for subsequent analysis.  By linearizing the  McKean-Vlasov
dynamics resulting from our formulation and probing these dynamics in Fourier space, we reveal frequency-dependent convergence
timescales comparable to \emph{Landau damping} in plasma kinetic theory: the
Gaussian kernel suffers an exponential high-frequency bottleneck, potentially explaining
the empirical preference for the Laplacian kernel.  Our analysis also suggests a fix: an exponential bandwidth annealing schedule
$\sigma(t)=\sigma_0 e^{-rt}$ that reduces convergence time from
$\exp(O(K_{\max}^2))$ to $O(\log K_{\max})$.  Finally, by formalizing drifting as
a Wasserstein gradient flow of the smoothed KL divergence, we prove that the
stop-gradient operator is not a heuristic but is derived directly from the frozen-field discretization mandated by the Jordan, Kinderlehrer and Otto (JKO)
scheme, and removing it severs training from any gradient-flow guarantee.
This variational perspective further provides a general template for
constructing novel drift operators, which we demonstrate with a Sinkhorn
divergence drift. We validate our analysis on toy datasets and scale it up to ImageNet. Code is available \href{https://github.com/nicolas-dufour/drifting_pytorch}{here}\footnote{\url{https://github.com/nicolas-dufour/drifting_pytorch}}.
\end{abstract}

\section{Introduction}
\label{sec:intro}
The dominant paradigm in continuous generative modeling relies on learning the \textit{score function} of data distributions and then using the learned score to guide the generation process. Energy-Based Models exploit the score to bypass the intractable partition function  \cite{grathwohl2019your,du2019implicit,song2021train}; Score-Based Generative Models and Diffusion Models \citep{pmlr-v37-sohl-dickstein15, ho2020denoising, song2021scorebased} are unified under this framework via Denoising Score Matching \citep{vincent2011connection} which established that training a denoiser is theoretically equivalent to score matching; furthermore, Flow Matching \citep{lipman2023flow} has recently been shown to fit within the same family \citep{gao2025diffusionmeetsflow}. Bridging these paradigms has allowed researchers to develop robust mathematical machinery and principled training practices across all of these models.

A recent approach, \emph{Generative Modeling via Drifting}
\citep{deng2026drifting}, appears to depart from this formalism.  Instead
of learning a score or velocity field, drifting prescribes a
\emph{kernel-based drift operator} $V_{p,q}$ that pulls generated samples toward data
while pushing them away from one another.  The generator is trained to match
its own drifted outputs until the drift vanishes, achieving impressive one-step image
generation without distillation, teacher models, or adversarial training.

Despite its impressive empirical results, the mathematical structure of drifting remains largely underexplored. Specifically, \citet{deng2026drifting} left three
foundational open questions:
\begin{enumerate}[leftmargin=*, itemsep=1pt, topsep=3pt]
  \item \emph{Identifiability.} Does $V_{p,q}=0$ guarantee $p=q$?
  \item \emph{Kernel selection.} The drift operator depends on the choice of the kernel; how should they be defined and selected?
  \item \emph{Algorithmic stability.} Is the stop-gradient operator
        essential, and, if so, what are its theoretical justifications?
\end{enumerate}
All three questions share a common root: we do not know what the drift
operator \emph{actually computes}.

In this paper we show that the drift operator, under the Gaussian kernel,
\emph{is} a score difference on smoothed distributions:
\begin{equation}
\label{eq:key_identity_intro}
  V_{p,q}^{(\sigma)}(x) \;=\; \sigma^2\,\nabla_x\log\frac{p_\sigma(x)}{q_\sigma(x)}.
\end{equation}


\begin{figure*}[t]
  \centering
  \includegraphics[width=\textwidth]{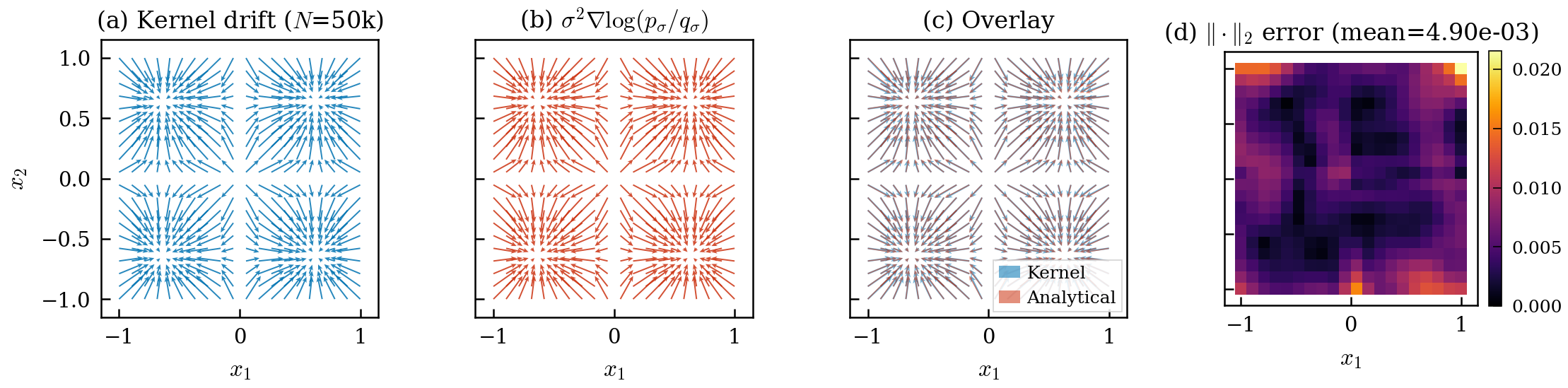}
  \caption{
    Numerical confirmation of Theorem~\ref{thm:score_matching} on a 4-mode
    Gaussian mixture.
    \textbf{(a)}~Empirical kernel mean-shift drift ($N=50$k samples).
    \textbf{(b)}~Analytical score-difference form
    $\sigma^2\nabla\log(p_\sigma/q_\sigma)$.
    \textbf{(c)}~Overlay: the two fields are visually indistinguishable.
    \textbf{(d)}~Pointwise $\ell_2$ error heatmap (mean $4.9\times10^{-3}$).
    Details are provided in Section \ref{app:score_check}.
  }
  \label{fig:score_matching_verification}
\end{figure*}

This identity, which we derive from first principles by direct kernel substitution,
immediately positions drifting within the score-matching family (see Figure~\ref{fig:score_matching_verification} for direct numerical confirmation). More importantly, it provides the theoretical framework required to resolve all three open questions. Specifically, we prove identifiability via the injectivity of Gaussian convolution (Sec. \ref{sec:score_connection}), kernel selection via a Fourier-space stability analysis of the linearized McKean-Vlasov equation (Sec. \ref{sec:fourier_stability}), and stop-gradient necessity via the JKO discretization of a Wasserstein gradient flow (Sec. \ref{sec:gradient_flow}).

To summarize, our key contributions are:
\begin{enumerate}[leftmargin=*, itemsep=2pt, topsep=4pt]
  \item \textbf{Score-matching identity and identifiability} (\S\ref{sec:score_connection}). Direct kernel substitution gives $V_{p,q}^{(\sigma)}=\sigma^2\nabla\log(p_\sigma/q_\sigma)$, placing drifting in the score-matching family and reducing identifiability to Fourier injectivity.
  \item \textbf{Landau damping and kernel diagnosis} (\S\ref{sec:fourier_stability}). Linearizing the McKean--Vlasov dynamics around equilibrium yields mode-resolved convergence timescales: exponential slowdown for the Gaussian kernel, polynomial for the Laplacian, a first principled explanation of the empirical kernel preference.
  \item \textbf{Wasserstein gradient flow and stop-gradient necessity} (\S\ref{sec:gradient_flow}). Drifting is the JKO-discretized Wasserstein gradient flow of the smoothed KL divergence, and the stop-gradient operator is precisely the frozen-field structure this discretization requires; removing it severs training from any descent guarantee.
  \item \textbf{Algorithmic improvements} (\S\ref{sec:practical}). An exponential bandwidth schedule $\sigma(t)=\sigma_0 e^{-rt}$ provably reduces convergence time from $\exp(\mathcal{O}(K_{\max}^2))$ to $\mathcal{O}(\log K_{\max})$ while retaining Gaussian identifiability, and the variational template $V=-\nabla(\delta\mathcal{F}/\delta q)$ yields new operators such as a Sinkhorn-divergence drift.
\end{enumerate}


\section{Background}
\label{sec:background}

We briefly recall the basic principles behind score-matching and discuss how Diffusion and Flow Models are unified under this lens.

\paragraph{Score matching and Energy-Based Models.}
The score function $\nabla_x\log p_\theta(x)$ bypasses the intractable normalization constant in energy-based models and enables Langevin sampling:
$X_{t+1} = X_t + \frac{\eta^2}{2}\nabla_x\log p_\theta(X_t) + \eta z$, where $z\sim\mathcal{N}(0,I)$ and $\eta$ is the step size.
Because the raw data score is unavailable, \citet{vincent2011connection} introduced denoising score matching: by working with the Gaussian-perturbed distribution $q(\tilde x)=\int\mathcal{N}(\tilde x;x,\sigma^2 I)p_D(x)\,dx$, the objective is replaced by matching the score of the transition kernel,  which simplifies to $-(\tilde x - x)/\sigma^2$. This perturbed score will reappear at the heart of our analysis.

\paragraph{Diffusion and Flow Models.}
Diffusion models \citep{pmlr-v37-sohl-dickstein15, ho2020denoising, song2021scorebased} connect the noise-predicting network $\epsilon_\theta$ to the score via Tweedie's formula, $\epsilon_\theta(x_t)\approx -\sigma_t\nabla_{x_t}\log p_t(x_t)$. This connection is most clear when viewing the forward perturbation as a Stochastic Differential Equation (SDE), $dx = f(x,t)\,dt + g(t)\,dw$, where $f(x, t)$ is the drift coefficient, $g(t)$ is the diffusion coefficient, and $w$ is standard Brownian motion. The generative process is then given by the exact reverse-time SDE:
\begin{equation}
\label{eq:reverse}
 dx=[f(x,t)-g(t)^2\nabla_{x}\log p_t(x)]\,dt + g(t)\,d\bar{w},
\end{equation}
where $\bar{w}$ is standard Brownian motion flowing backward in time, and $p_t(x)$ is the marginal distribution of the perturbed data at time $t$. 

Flow models \citep{lipman2023flow} learn a velocity field $v_\theta$ via
$\dot x_t=v_\theta(x_t)$, sampling by ODE integration
\begin{equation}
\label{eq:ode_sampling}
x_1=x_0+\int_0^1 v_\theta(x_\tau)\,d\tau
\end{equation}
Recent work \citep{gao2025diffusionmeetsflow} has shown that these two families are equivalent, completing the unification under score matching.  In both cases, \emph{the dynamics operate at inference time}: the generator is a learned field that is integrated to produce samples.

\section{Drifting Models}
\label{sec:drifting}

Drifting \citep{deng2026drifting} departs from the above paradigm by
shifting all dynamics from inference to training.  Let $p$ be a data
distribution on $\R^d$ and $f_\theta:\R^k\to\R^d$ a generator whose
pushforward $q=(f_\theta)_\#\mathcal{N}(0,I)$ should approximate $p$.
The \emph{drift operator} $V_{p,q}=V_p^+-V_q^-$ is defined as:
\begin{equation}
\label{eq:drift}
  V_p^+(x) = \frac{\E_{y\sim p}[k(x,y)(y-x)]}{\E_{y\sim p}[k(x,y)]},
  \qquad
  V_q^-(x) = \frac{\E_{y\sim q}[k(x,y)(y-x)]}{\E_{y\sim q}[k(x,y)]},
\end{equation}
where $k:\R^d\times\R^d\to\R_{>0}$ is a positive kernel.  The attractive
term $V_p^+$ pulls generated samples toward nearby data; the repulsive term
$V_q^-$ pushes them away from one another to prevent mode collapse.

During training, noise samples $\epsilon \in \mathbb{R}^k$ are mapped to $x=f_\theta(\epsilon)$ and
drifted to target states $\tilde x=x+V_{p,q}(x)$. Note the target data distribution $p$ only enters the optimization through the drift operator $V_{p,q}$. The generator is trained
via a \emph{stop-gradient loss}:
\begin{equation}
\label{eq:loss}
  \mathcal{L}(\theta) = \E_\epsilon[\|f_\theta(\epsilon)-\sg[\tilde x]\|^2].
\end{equation}
\citet{deng2026drifting} established $q=p\Rightarrow V_{p,q}=0$ but left
the converse, \textit{identifiability}, open.  Moreover, as mentioned above, their reliance on a Laplacian kernel
with grid-searched bandwidths lacks theoretical justification, and the
necessity of $\sg[\cdot]$ remains a heuristic.  We resolve all
three in the sections that follow.

\section{Drifting is Score Matching}
\label{sec:score_connection}

\subsection{The Core Identity}
The following result is the foundation for all subsequent analysis.
\begin{theorem}[Gaussian drift as score difference]
\label{thm:score_matching}
Under the Gaussian kernel $\varphi_\sigma$, the drift operator admits the
closed form expression:
\begin{equation}
\label{eq:drift_score}
  V_{p,q}^{(\sigma)}(x) \;=\; \sigma^2\,\nabla_x \log \frac{p_\sigma(x)}{q_\sigma(x)},
\end{equation}
where $p_\sigma:=p*\varphi_\sigma$ and $q_\sigma:=q*\varphi_\sigma$.
\end{theorem}

\begin{proof}
See \ref{app:score}.
\end{proof}
This identity, which we derive by direct kernel substitution within the
drifting architecture, inserts drifting squarely into the score-matching
family under the Gaussian kernel.  It was independently established by
\citet{weber2023scoredifference} from the opposite direction, starting from a
KL-minimizing probability flow; the two derivations share the identity and
little else.  Ours is the starting point for the Fourier stability analysis,
the annealing schedule, and the gradient-flow formalism that follow.

\subsection{Connection to Denoising Score Matching}
Theorem~\ref{thm:score_matching} places drifting in direct correspondence with the score-matching framework \citep{vincent2011connection, song2021scorebased}. Classical Denoising Score Matching trains a network $\psi_\theta$ to \emph{approximate} the score $\nabla \log p_\sigma$ of a smoothed data distribution by regressing against the perturbation-kernel score $-(x-y)/\sigma^2$ \citep{vincent2011connection}; at inference, $\psi_\theta$ is then evaluated repeatedly to integrate an SDE or ODE. Drifting \citep{deng2026drifting} instead parameterizes a \emph{sampler} $f_\theta$ whose pushforward defines $q = (f_\theta)_\#\mathcal{N}(0,I)$, and reads off the score difference $\sigma^2\bigl(\nabla \log p_\sigma - \nabla \log q_\sigma\bigr)$ exactly and non-parametrically from Parzen estimates of $p_\sigma$ and $q_\sigma$ on the current minibatch. The neural network thus never represents a vector field: the score is computed from particles on the fly, used only at training time to push $f_\theta$ toward $p$, and \emph{vanishes from the pipeline at inference}, where sampling reduces to a single forward pass of $f_\theta$.

\subsection{Continuous-Time Limit and McKean--Vlasov Structure}
\label{sec:continuous}

The discrete training update $x_{n+1}=x_n+V_{p,q_n}(x_n)$ naturally admits a
continuous-time limit in which each sample evolves as
\begin{equation}
\label{eq:time-continuous}
  \frac{dx_t}{dt} = V_{p,\,q(t)}(x_t),
\end{equation}
with the generator distribution satisfying the continuity equation
\begin{equation}
\label{eq:continuity}
  \partial_t q(t,x) + \nabla_x\cdot\bigl(q(t,x)\,V_{p,\,q(t)}(x)\bigr) = 0.
\end{equation}
Derivation details are in \ref{app:continuous_time}.  Under the Gaussian
kernel, Eq. \eqref{eq:continuity} is the \emph{McKean--Vlasov equation},
well-studied in kinetic theory \citep{villani2002review}.  Contrary to Diffusion and Flow models, because the velocity field is prescribed (not learned), the dynamics are amenable to
the stability analysis in \S\ref{sec:analysis}.

\section{Theoretical Consequences of the Identity}
\label{sec:analysis}

Theorem~\ref{thm:score_matching} allow us to resolve all three foundational
questions.  We treat each in turn.

\subsection{Identifiability}

\begin{proposition}[Identifiability]
\label{prop:identifiability}
If $V_{p,q}^{(\sigma)}(x) = 0$ for all $x \in \R^d$ and $\sigma > 0$,
then $p = q$.
\end{proposition}

\begin{proof}
See \ref{app:identifiability}.
\end{proof}

\subsection{Why the Laplacian Kernel? Landau Damping in Generative Models}
\label{sec:fourier_stability}

\begin{figure*}[t]
  \centering
  \includegraphics[width=\textwidth]{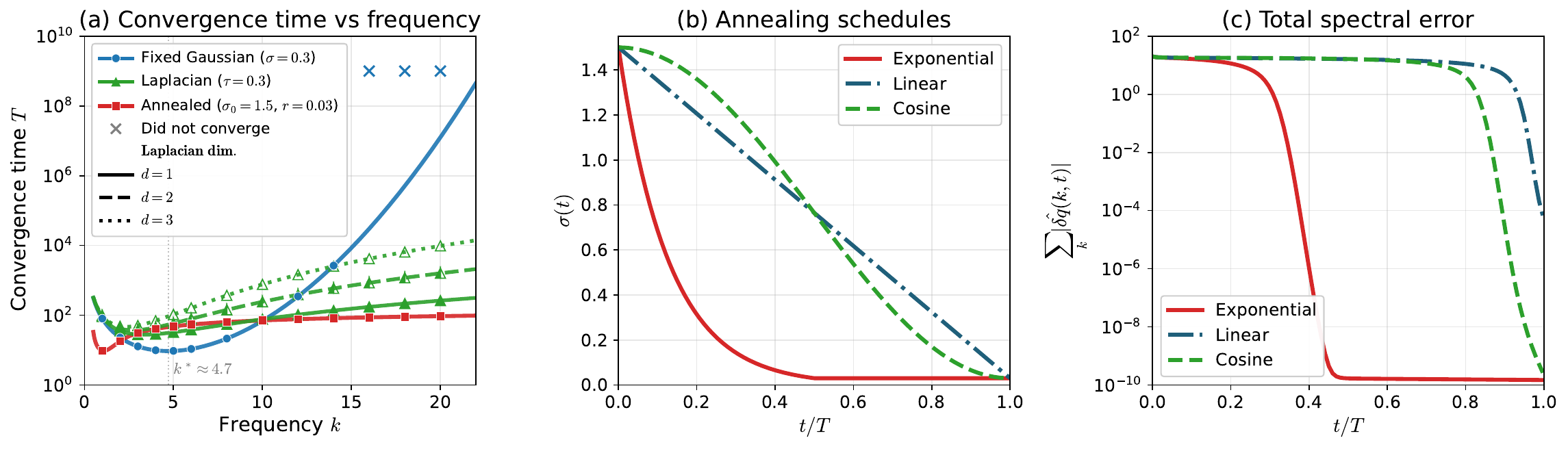}
  \caption{
    Spectral validation.
    \textbf{(a)}~Convergence time $T(k)$ vs.\ frequency: lines are analytical
    predictions from Theorem~\ref{thm:general_kernel_timescales}, markers are
    measured decay times.  The fixed-bandwidth Gaussian exhibits exponential
    slowdown (Landau damping); the Laplacian kernel yields polynomial scaling;
    the annealed Gaussian eliminates the bottleneck entirely.
    \textbf{(b)}~Annealing schedules $\sigma(t)$.
    \textbf{(c)}~Total spectral error under different schedules.
    Details in \ref{app:spectral_check}.
  }
  \label{fig:spectral_validation}
\end{figure*}

We can exploit the analytical expression of the drift operator in Equation \ref{eq:drift_score} and linearize  the PDE in Eq. \eqref{eq:continuity} around equilibrium, writing
$q(t,x)=p(x)+\delta(t,x)$ with $\delta$ small.  Under the local-homogeneity
approximation (\ref{app:linearization}), the linearized dynamics for the
smoothed perturbation $\delta_\sigma=\varphi_\sigma*\delta$ reduce to
\begin{equation}
\label{eq:diffusion}
  \partial_t\delta(x,t) = \sigma^2\Delta\bigl(\delta_\sigma(x,t)\bigr).
\end{equation}

\begin{theorem}[Mode-resolved convergence timescales]
\label{thm:general_kernel_timescales}
Let $\kappa\in L^1(\R^d)$ be an even kernel.  Under the linearized dynamics
$\partial_t\delta = c_\kappa\Delta(\kappa*\delta)$, each Fourier mode evolves as
\[
  \partial_t\hat\delta(\xi,t) = -\lambda_\kappa(\xi)\,\hat\delta(\xi,t),
  \qquad
  \lambda_\kappa(\xi) = c_\kappa|\xi|^2\hat\kappa(\xi).
\]
The convergence timescale of mode $\xi$ is
$\tau_\kappa(\xi)=1/\lambda_\kappa(\xi)=(c_\kappa|\xi|^2\hat\kappa(\xi))^{-1}$.
\end{theorem}
\begin{proof}
See \ref{app:linearization}
\end{proof}

For the Gaussian kernel, the effective linearized convolution kernel is the
Gaussian kernel itself. For the exponential kernel, the same linearization
yields an equation of the form in Theorem~\ref{thm:general_kernel_timescales}, but with
the companion kernel $h_\tau(r) \propto\tau(r+\tau)e^{-r/\tau}$.

\begin{remark}[Landau damping analogy]
In plasma physics, perturbations in a collisionless plasma decay via
\emph{Landau damping}: frequency modes are damped at rates controlled by the
medium's spectral properties \citep{villani2002review, mouhot2011landau}.  In
generative drifting, the \emph{kernel} plays the role of the medium.  The
dispersion relation $\lambda_\kappa(\xi)\propto|\xi|^2\hat\kappa(\xi)$ is the
precise analogue.  To our knowledge, this identification has not appeared in
prior work on generative models.
\end{remark}

Applying Theorem~\ref{thm:general_kernel_timescales} to the two kernels of
interest makes the bottleneck explicit.

\begin{corollary}[Gaussian vs.\ Laplacian convergence times]
\label{cor:convergence_time}
To reduce all Fourier modes up to $|k|=K_{\max}$ by a factor $1/\epsilon$:
\begin{align}
  T_{\mathrm{Gauss}}
  &= \frac{\log(1/\epsilon)}{\sigma^2 K_{\max}^2}
     \exp\!\Bigl(\frac{\sigma^2 K_{\max}^2}{2}\Bigr),
  &\text{(exponential in $K_{\max}^2$),}
  \label{eq:T_gauss} \\[4pt]
  T_{\mathrm{exp}}
  &\propto \log(1/\epsilon)\;\tau^{d+1}\,K_{\max}^{d+1},
  &\text{(polynomial in $K_{\max}$).}
  \label{eq:T_exp}
\end{align}
\end{corollary}
\begin{proof}
See \ref{app:kernel_time}
\end{proof}
The Gaussian kernel at bandwidth $\sigma$ acts as a low-pass filter with
cutoff $|k|\sim 1/\sigma$: modes above the cutoff are exponentially
suppressed.  This bottleneck is intrinsic to the \emph{kernelized particle
dynamics}, not the neural network; it persists even with a perfect function
approximator, and is mechanistically distinct from the spectral bias of
neural networks \citep{rahaman2019spectral, tancik2020fourier}.
Corollary~\ref{cor:convergence_time} gives a first principled explanation
of why \citet{deng2026drifting} may have favor the Laplacian kernel. It may also suggest their use of a multi-scale drift operator of the form $V_{p,q}^{\text{multi}}=\sum_{\sigma} V_{p,q}^{(\sigma)}$.  In Figure
\ref{fig:spectral_validation}(a) we validate the analytical predictions
numerically, details are presented in \ref{app:spectral_check}.

\subsection{Stop-Gradient as a Structural Necessity}
\label{sec:gradient_flow}
A distinctive feature of drifting is that its dynamics unfold during \emph{training} rather than inference: the sequence of generators $\{f_{\theta_i}\}$ induces distributions $q_i=(f_{\theta_i})_\#\nu$ that define a training-time probability flow. We now show this flow is the Jordan--Kinderlehrer--Otto (JKO) discretization of a Wasserstein gradient flow \citep{noauthororeditor, Ambrosio2005} of a smoothed KL functional, and that the stop-gradient operator is not a stabilization trick but the frozen-field discretization this variational scheme requires, earning the iterates $\{q_i\}$ monotone-descent guarantees toward $q=p$.

\paragraph{The smoothed KL energy.}
Fix $p\in\Pc(\Omega)$ and $\sigma>0$.  Define
\begin{equation}
\label{eq:F_sigma}
  F_\sigma[q] := \sigma^2\,\KL(q_\sigma\|p_\sigma)
  = \sigma^2\int_{\R^d} q_\sigma(x)\log\frac{q_\sigma(x)}{p_\sigma(x)}\,dx.
\end{equation}

\begin{proposition}[First variation and drift recovery]
\label{prop:first_variation_drift}
\leavevmode
\begin{enumerate}[label=\textup{(\roman*)}, leftmargin=*]
\item $F_\sigma$ is lower-semi-continuous.
\item $\delta F_\sigma/\delta q(x) =
  \sigma^2(\varphi_\sigma*\log(q_\sigma/p_\sigma))(x)$
  (up to an additive constant), and is $C^\infty$ on $\R^d$.
\item The Wasserstein gradient field
  $v_\sigma[q]:=-\nabla_x(\delta F_\sigma/\delta q)$ satisfies
  $v_\sigma[q](x)\approx -V_{p,q}^{(\sigma)}(x)$ when $\log(q_\sigma/p_\sigma)$
  varies slowly at scale $\sigma$; the error is
  $O(\sigma^2\|\nabla^2\log(q_\sigma/p_\sigma)\|_\infty)$
  (\ref{app:convolution_approximation}).
\end{enumerate}
\end{proposition}
\begin{proof}
See \ref{app:first_variation}.
\end{proof}

The drifting continuity equation \ref{eq:continuity} is thus (up to this quantified approximation)
the Wasserstein gradient flow of $F_\sigma$.

\paragraph{The JKO scheme and the frozen-field step.}
\begin{figure}[t]
  \centering
  \includegraphics[width=\linewidth]{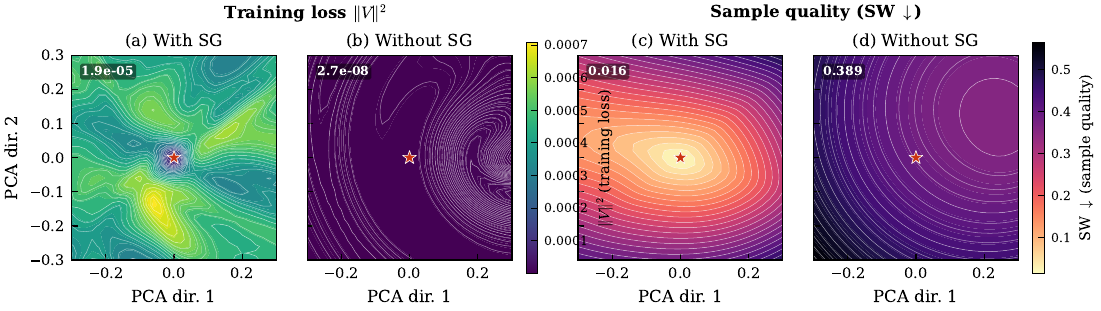}
  \caption{
    Loss landscapes with and without stop-gradient, projected onto the top
    two principal gradient-variation directions.
    \textbf{(a,b)}~Training loss $\|V\|^2$: without SG the minimum is
    ${\sim}100\times$ deeper.
    \textbf{(c,d)}~Sliced Wasserstein distance: with SG the loss minimum
    coincides with low distributional error; without SG the deep minimum
    corresponds to poor sample quality.  Red star: trained solution.
    Details in \ref{app:stop_grad_check}.
  }
  \label{fig:loss_landscape}
\end{figure}
The Jordan--Kinderlehrer--Otto scheme
\citep{jordan1998variational} discretizes the gradient flow as
\begin{equation}
\label{eq:JKO}
  q_{n+1}^\tau \in \argmin_{q\in\Pc(\Omega)}
  \Bigl\{F_\sigma[q]+\frac{1}{2\tau}W_2^2(q,q_n^\tau)\Bigr\}.
\end{equation}
We show that, given the functional \ref{eq:F_sigma}, this scheme is well-posed (unique minimizer, monotone energy descent;
\ref{app:JKO_wellposed}) and converges as $\tau\to0$
(\ref{app:convergence}).  The Euler--Lagrange condition reveals an implicit
velocity:
\begin{equation}
\label{eq:velocity_identification}
  \frac{T_n(x)-x}{\tau} = v_\sigma[\hat q](x)
  = -\sigma^2\nabla(\varphi_\sigma*\log(\hat q_\sigma/p_\sigma))(x),
  \quad \hat q\text{-a.e.},
\end{equation}
where $T_n$ is the optimal map from $\hat q:=q_{n+1}^\tau$ to $q_n^\tau$.
The velocity has the same form as the drift, \emph{but is evaluated at the
unknown minimizer $\hat q$}, making the JKO step implicit and intractable.

The practical drifting algorithm resolves this by replacing $q_{n+1}^\tau$
with $q_n^\tau$ in the velocity field, the \emph{frozen-field} (explicit
Euler) approximation:
\begin{equation}
\label{eq:frozen_map}
  S_n(x):=x+\tau v_\sigma[q_n^\tau](x), \qquad
  \tilde q_{n+1}^\tau:=(S_n)_\#q_n^\tau.
\end{equation}
Where $(S_n)_\#q_n^\tau$ represents the pushforward of the distribution $q_n^{\tau}$ by map $S_n$. We show in detail in \ref{app:consistency}, that given our assumptions, this approximation introduces an error of only $W_2(\tilde q_{n+1}^\tau,q_{n+1}^\tau)=O(\tau^{3/2})$ \label{eq:consistency}
.  The chain of approximations is:
\[
\underbrace{q_{n+1}^\tau=\argmin\{F_\sigma+\tfrac{1}{2\tau}W_2^2(\cdot,q_n^\tau)\}}_{%
\text{JKO (implicit, intractable)}}
\;\xrightarrow{\text{freeze velocity}}\;
\underbrace{\tilde q_{n+1}^\tau=(S_n)_\#q_n^\tau}_{%
\text{Explicit Euler}}
\;\xrightarrow{\text{parametric fit}}\;
\underbrace{\min_\theta\;\mathcal{L}(\theta)}_{%
\text{Stop-gradient loss}}
\]

\paragraph{Parametric implementation and necessity proof.}
Fitting the frozen-field target with a generator $G_\theta$ yields the loss
\begin{equation}
\label{eq:stopgrad_loss}
\mathcal{L}(\theta)=\E_{\varepsilon\sim\nu}\Bigl[
\bigl\|G_\theta(\varepsilon)-\sg\bigl[
G_{\theta_n}(\varepsilon)+\tau v_\sigma[q_{\theta_n}](G_{\theta_n}(\varepsilon))
\bigr]\bigr\|^2\Bigr],
\end{equation}
which is precisely the drifting loss of \citet{deng2026drifting} with
$\eta=\tau$.

\begin{theorem}[Stop-gradient preserves Wasserstein discretization]
\label{thm:stopgrad}
\leavevmode
\begin{enumerate}[label=\textup{(\roman*)}, leftmargin=*]
\item \label{item:structural}
\textup{(Structural correspondence.)}
A global minimizer of~\eqref{eq:stopgrad_loss} satisfies
$q_{\theta_{n+1}}=(S_n)_\#q_{\theta_n}$, the frozen-field explicit Euler
step of the Wasserstein gradient flow of $F_\sigma$.

\item \label{item:necessity}
\textup{(Stop-gradient necessity.)}
Removing $\sg[\cdot]$ yields
$\mathcal{L}_{\mathrm{coupled}}(\theta)=\tau^2\|v_\sigma[q_\theta]\|_{L^2(q_\theta)}^2$,
whose gradient includes distribution-feedback terms
$(D_q v_\sigma)\nabla_\theta q_\theta$.  Stationarity can now be achieved by
\emph{drift collapse}---reducing the velocity norm without transporting mass
toward $p$---rather than by distributional convergence.
\end{enumerate}
\end{theorem}

\begin{proof}
See \ref{app:stopgrad}.
\end{proof}

\begin{figure*}[t]
  \centering
  \includegraphics[width=\textwidth]{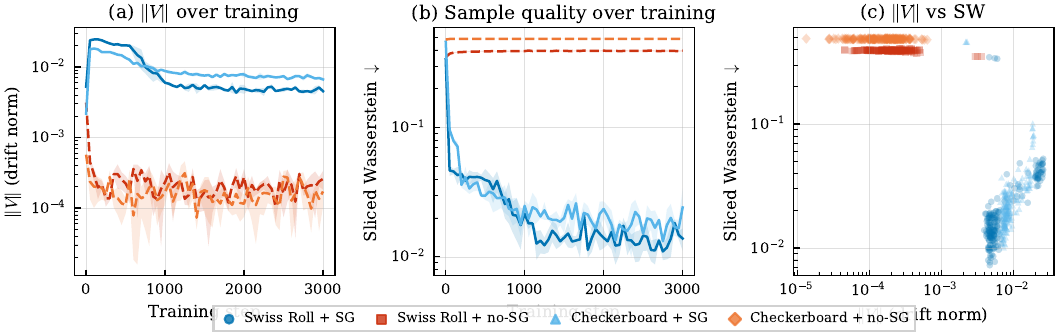}
  \caption{
    Drift norm vs.\ distributional distance during training on synthetic 2D
    targets.
    \textbf{(a)}~Mean drift norm $\|V\|$.
    \textbf{(b)}~Sliced Wasserstein distance.
    \textbf{(c)}~Log--log scatter across training steps and seeds.
    With stop-gradient (solid), the two quantities are strongly correlated
    ($r>0.95$) and jointly decay.  Without stop-gradient (dashed), the drift
    norm collapses to ${\sim}10^{-8}$ while the Wasserstein distance remains
    at $0.389$---a direct demonstration of drift collapse
    (Theorem~\ref{thm:stopgrad}\ref{item:necessity}).
    Details in \ref{app:stop_grad_check}.
  }
  \label{fig:drift_vs_sw}
\end{figure*}

Figures~\ref{fig:loss_landscape} and~\ref{fig:drift_vs_sw} provide direct
empirical confirmation: with stop-gradient the loss minimum aligns with low
distributional error; without it, drift collapse produces a spuriously deep
minimum with no distributional improvement.
\section{From Theory to Practice}
\label{sec:practical}

This extensive analysis is not only descriptive but also prescriptive, indeed the score-matching identity, spectral analysis, and gradient-flow formalism each yield a concrete practical contribution.  We treat them in turn.

\subsection{Kernel Comparison: Why the Laplacian Works}
\label{sec:kernel_comparison}

Applying Theorem~\ref{thm:general_kernel_timescales} to the Gaussian kernel
gives the mode-resolved timescale
\begin{equation}
\label{eq:timescale}
  \tau(k,\sigma) = \frac{1}{\sigma^2|k|^2}\exp\!\Bigl(\frac{\sigma^2|k|^2}{2}\Bigr).
\end{equation}
For $\sigma|k|\gg1$, this grows exponentially: modes above the cutoff
$|k|\sim1/\sigma$ are frozen out.  For the Laplacian kernel,
$\hat k_\tau(\xi)\propto(1+\tau^2|\xi|^2)^{-(d+1)/2}$, giving the
high-frequency rate $\lambda_{\mathrm{exp}}(k)\asymp\tau^{-(d+1)}|k|^{-(d-1)}$, only
polynomial slowdown.  Corollary~\ref{cor:convergence_time} quantifies the
difference, providing the first principled justification for the empirical
kernel preference in \citet{deng2026drifting}.


\subsection{Exponential Bandwidth Annealing}
\label{sec:annealing}

The spectral analysis reveals a potential issue: the Gaussian kernel provides
identifiability and a clean score-matching form, but suffers an exponential
high-frequency bottleneck; the Laplacian kernel avoids this bottleneck but
lacks these analytical properties.  We resolve this problem with the following bandwidth annealing schedule:
\begin{equation}
\label{eq:annealing_schedule}
  \sigma(t) = \sigma_0\,e^{-rt},
\end{equation}
held constant at $\sigma_{\min}$ once reached.  The exponential form ensures
that the optimal-rate window ($\sigma^2|k|^2=2$) sweeps continuously across
increasing frequencies, activating each mode at its maximal convergence rate.

\begin{theorem}[Convergence time under exponential annealing]
\label{thm:annealing_convergence}
Let $\sigma(t)$ follow~Eq.\eqref{eq:annealing_schedule} until $\sigma_{\min}$,
then remain constant.  To reduce all modes $|k|\le K_{\max}$ by factor
$1/\epsilon$,
\begin{equation}
\label{eq:T_annealing}
  T_{\epsilon,\mathrm{anneal}} \lesssim
  \frac{1}{r}\log\!\Bigl(\frac{\sigma_0}{\sigma_{\min}}\Bigr)
  + \frac{1}{\lambda_{\min}(K_{\max})}\log(1/\epsilon),
  \quad
  \lambda_{\min}=\sigma_{\min}^2 K_{\max}^2\exp\!\Bigl(-\tfrac12\sigma_{\min}^2 K_{\max}^2\Bigr).
\end{equation}
\end{theorem}
\begin{proof}
Full proof in \ref{app:annealing}.
\end{proof}

The total time depends only \emph{logarithmically} on $K_{\max}$, compared
to exponentially for the fixed-bandwidth Gaussian.
Figure~\ref{fig:spectral_validation}(b,c) confirms that the exponential
schedule achieves the fastest spectral error reduction across all modes.

\begin{figure}[t]
  \centering
  \includegraphics[width=\linewidth]{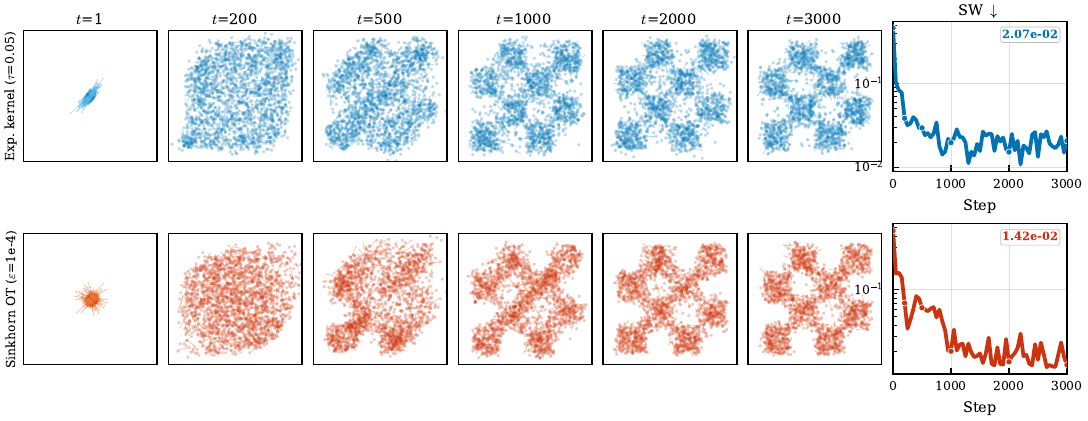}
  \caption{
    Sinkhorn-derived drift (bottom row) vs.\ Laplacian-kernel drift (top row)
    on the checkerboard distribution.  Training snapshots with drift vectors
    overlaid; rightmost panels show sliced Wasserstein distance over training.
    Both converge successfully (final SW $1.42\times10^{-2}$ and
    $2.07\times10^{-2}$ respectively), demonstrating that the gradient-flow
    template of \S\ref{sec:principled_drift} yields practical operators
    beyond the original kernel family.  Details are provided in \ref{app:sinkhorn_check}.
  }
  \label{fig:sinkhorn_vs_kernel}
\end{figure}

\subsection{Principled Drift Operator Construction}
\label{sec:principled_drift}

The original operator introduced in \cite{deng2026drifting} was motivated through anti-symmetry properties of the kernel mean-shift, and was presented as one specific instantiation of a drifting model. Our gradient-flow formalism reveals the kernel drift as a particular instance of a universal variational template: given any sufficiently regular discrepancy functional $F$ on $\Pc(\Omega)$, define
\begin{equation}
\label{eq:general_drift}
  V(x) = -\nabla_x\frac{\delta F}{\delta q}(x).
\end{equation}
The entire JKO--frozen-field--stop-gradient chain of \S\ref{sec:gradient_flow} applies verbatim to any $F$ satisfying: \textbf{(a)} lower semi-continuity in $W_2$; \textbf{(b)} existence and smoothness of $\delta F/\delta q$; and \textbf{(c)} $F[q]\ge 0$ with $F[q]=0\Leftrightarrow q=p$. This supersedes the anti-symmetry requirement of \citet{deng2026drifting}: condition~(c) alone is what guarantees that the gradient flow drives $q$ to $p$, and (a)--(b) ensure that the JKO scheme is well-posed and that the velocity field is well-defined.

\paragraph{The Sinkhorn divergence as an energy functional.}
A natural candidate that satisfies (a)--(c) is the \emph{Sinkhorn divergence} of \citet{feydy2018interpolatingoptimaltransportmmd}, built from entropy-regularized optimal transport. For two probability measures $\mu,\nu\in\Pc(\Omega)$ and a cost $c(x,y)=\|x-y\|^r$ ($r\in\{1,2\}$), the entropic OT problem with regularization $\varepsilon>0$ is
\begin{equation}
\label{eq:entropic_ot}
  \mathrm{OT}_\varepsilon(\mu,\nu)
  \;=\;
  \min_{\pi\in\Pi(\mu,\nu)}
  \int c(x,y)\,d\pi(x,y)
  \;+\;\varepsilon\,\KL\!\bigl(\pi\,\big\|\,\mu\otimes\nu\bigr),
\end{equation}
where $\Pi(\mu,\nu)$ denotes couplings with marginals $\mu,\nu$. The unique minimizer has the Gibbs form $\pi^\star(x,y) \propto \exp\!\bigl((f(x)+g(y)-c(x,y))/\varepsilon\bigr)$, with dual potentials $(f,g)$ obtained by Sinkhorn iterations \citep{peyre2020computationaloptimaltransport}; we use the standard log-domain implementation, fully detailed in \ref{sec:sinkhorn-implementation}. The Sinkhorn divergence is the \emph{debiased} version
\begin{equation}
\label{eq:sinkhorn_energy}
  S_\varepsilon(q,p) \;=\; \mathrm{OT}_\varepsilon(q,p)
  -\tfrac12\,\mathrm{OT}_\varepsilon(q,q)-\tfrac12\,\mathrm{OT}_\varepsilon(p,p),
\end{equation}
which interpolates between the MMD ($\varepsilon\to\infty$) and unregularized OT ($\varepsilon\to 0$), and removes the entropic bias of $\mathrm{OT}_\varepsilon$ alone. By Theorem~1 of \citet{feydy2018interpolatingoptimaltransportmmd}, $S_\varepsilon$ is symmetric, convex in each argument, positive-definite, and metrizes weak convergence; in particular $S_\varepsilon(q,p)=0\Leftrightarrow q=p$, so condition~(c) holds. Lower semi-continuity and smoothness of the first variation follow from Proposition~2 there, securing (a)--(b). Plugging $F[q]:=S_\varepsilon(q,p)$ into the template~\eqref{eq:general_drift} therefore yields a drift with full JKO--frozen-field--stop-gradient guarantees.

\paragraph{The Sinkhorn drift operator.}
We can compute the variational gradient explicitly at the particle level. For an empirical generator distribution $q=\frac{1}{N}\sum_i\delta_{x_i}$, let $\pi_{q\to p}^\star$ and $\pi_{q\to q}^\star$ be the entropic couplings between $(q,p)$ and $(q,q)$, and define the \emph{barycentric projections}
\begin{equation}
\label{eq:bary_projections}
  T_{q\to p}(x_i)
  \;=\; \frac{\sum_j [\pi_{q\to p}^\star]_{ij}\,y_j}{\sum_j [\pi_{q\to p}^\star]_{ij}},
  \qquad
  T_{q\to q}(x_i)
  \;=\; \frac{\sum_j [\pi_{q\to q}^\star]_{ij}\,x_j}{\sum_j [\pi_{q\to q}^\star]_{ij}}.
\end{equation}
The map $T_{q\to p}$ sends each generated particle to its entropic OT image in the data distribution, and $T_{q\to q}$ does the same against the generator's own particles.

\begin{proposition}[Sinkhorn drift]
\label{prop:sinkhorn_drift}
Let $F[q]=S_\varepsilon(q,p)$ with $S_\varepsilon$ as in~\eqref{eq:sinkhorn_energy}. Then the variational drift operator~\eqref{eq:general_drift} admits the closed form
\begin{equation}
\label{eq:sinkhorn_drift}
  V_{\mathrm{SK}}(x_i)
  \;\propto\;
  T_{q\to p}(x_i)\;-\;T_{q\to q}(x_i),
\end{equation}
where $T_{q\to p},T_{q\to q}$ are the barycentric projections of~\eqref{eq:bary_projections}. The operator $V_{\mathrm{SK}}$ inherits all JKO--frozen-field--stop-gradient guarantees of \S\ref{sec:gradient_flow}.
\end{proposition}

\begin{proof}
Full derivation in \ref{app:sinkhorn_check}.
\end{proof}

Equation~\eqref{eq:sinkhorn_drift} is a striking parallel to the kernel drift $V_{p,q}=V_p^+-V_q^-$: an \emph{attractive} term pulling particles toward $p$, minus a \emph{repulsive} term pushing them away from one another. But here the structure does not stem from a hand-chosen kernel and anti-symmetry argument; it arises mechanically from differentiating a divergence satisfying (a)--(c), with the attractive and repulsive halves coming from the two $\mathrm{OT}_\varepsilon$ terms in the debiasing. The geometry has changed from kernel mean-shift to entropic optimal transport, but the gradient-flow scaffolding is unchanged. We plug $V_{\mathrm{SK}}$ into the standard stop-gradient loss $\mathcal{L}(\theta)=\E_\epsilon[\|f_\theta(\epsilon)-\sg[\tilde x]\|^2]$ with $\tilde x=x+V_{\mathrm{SK}}(x)$; full algorithmic details, including log-domain Sinkhorn iterations and a diagonal-masking heuristic that prevents $T_{q\to q}$ from collapsing to the identity as $q\to p$, are deferred to \ref{sec:sinkhorn-implementation}.

Figure~\ref{fig:sinkhorn_vs_kernel} shows that $V_{\mathrm{SK}}$ converges comparably to the Laplacian kernel drift on the checkerboard target, demonstrating that the gradient-flow template yields practical operators well beyond the original kernel family.

\subsection{Image Generation Experiments}
\label{sec:image_exp}

We now put our theoretical insights to the test on the realistic ImageNet generation setting, placing ourselves in the same ablation experimental setup as \citet{deng2026drifting} and testing three variants suggested by the theory: the Gaussian-kernel drift (which has the clean score-matching identity of Theorem~\ref{thm:score_matching}), the exponentially annealed Gaussian (which targets the spectral bottleneck of \S\ref{sec:fourier_stability}), and the Sinkhorn-divergence drift of Proposition~\ref{prop:sinkhorn_drift} (which exercises the gradient-flow template). 

\paragraph{Experimental setup.}
We train class-conditional generators on ImageNet-$256$ ($1000$ classes) in the latent space of the Stable Diffusion VAE; FID is computed after VAE-decoding back to pixels. The generator is \texttt{DitGen-B}, a one-step DiT backbone. The drift loss is computed on multi-scale activations of a frozen MAE-pretrained ResNet-18, summed across streams. All variants share the same architecture, SSL representation, and AdamW optimizer and are trained for $30{,}000$ total steps on $16$ H100s, \emph{only the drift operator differs} across rows of Table~\ref{tab:imagenet-ablation}. We report one-step (NFE\,$=$\,1) FID and IS, selecting the best CFG scale per checkpoint. Full hyperparameters, kernel and Sinkhorn algorithms, qualitative samples, and per-step wall-clock comparisons are provided in \ref{app:image_exp}.

\begin{table}[t]
  \caption{Ablation on ImageNet $256{\times}256$ class-conditional generation. All variants share the same B/2 architecture, latent-MAE SSL representation, and training schedule of \citet{deng2026drifting}; only the drift operator differs. We report FID ($\downarrow$) and IS ($\uparrow$) under one-step (NFE\,$=$\,1) generation. Best in \textbf{bold}, second-best \underline{underlined}.}
  \label{tab:imagenet-ablation}
  \centering
  \small
  \setlength{\tabcolsep}{4pt}
  \begin{tabularx}{\linewidth}{@{}lXccc@{}}
    \toprule
    \textbf{Drift operator} & \textbf{Bandwidth} & \textbf{NFE} & \textbf{FID\,$\downarrow$} & \textbf{IS\,$\uparrow$} \\
    \midrule
    Laplacian \citep{deng2026drifting} & $\tau$ fixed                            & 1 & \underline{8.55}          & 148.0 \\
    Gaussian (ours)                    & $\sigma$ fixed                          & 1 & 8.69          & \textbf{157.0} \\
    Gaussian (ours)                    & $\sigma(t)=\sigma_0 e^{-rt}$            & 1 & \textbf{8.36} & \underline{154.2} \\
    Sinkhorn (ours)                    & $\varepsilon$ fixed                     & 1 & 8.81          & 135.48 \\
    \bottomrule
  \end{tabularx}
\end{table}

\paragraph{Discussion.}
Three observations stand out. \emph{First}, the fixed-bandwidth Gaussian kernel is competitive with the Laplacian baseline (FID 8.69 vs.\ 8.55, IS 157.0 vs.\ 148.0), validating that the kernel admitting a clean score-matching identity (Theorem~\ref{thm:score_matching}) and full identifiability guarantees (Proposition~\ref{prop:identifiability}) is also a strong empirical performer; the theoretical convenience of the Gaussian does not come at an empirical cost. \emph{Second}, the annealed schedule $\sigma(t)=\sigma_0 e^{-rt}$ improves FID over both the fixed Gaussian and the Laplacian baseline (8.36 vs.\ 8.69 and 8.55), confirming that the spectral analysis of \S\ref{sec:fourier_stability} ports from the linearized toy regime to a high-dimensional generative setting. We note that in this regime the high-frequency bottleneck is not the only relevant axis, since the polynomial rate of the Laplacian kernel itself depends on dimension (Corollary~\ref{cor:convergence_time}); that the annealing prescription nonetheless yields a measurable improvement suggests the spectral diagnosis remains informative beyond its strict linearization assumptions. \emph{Third}, the Sinkhorn-divergence drift converges to a regime comparable with the kernel-based operators (FID 8.81) despite arising from optimal-transport geometry rather than kernel mean-shift; this is direct evidence that the gradient-flow template $V=-\nabla(\delta\mathcal{F}/\delta q)$ of \S\ref{sec:gradient_flow} is constructive, and that the anti-symmetry property invoked in \citet{deng2026drifting} is not fundamental: descent only requires conditions (a)--(c) of \S\ref{sec:principled_drift}. Taken together, these results show that each of the three theoretical contributions (the score-matching identity, the spectral analysis, and the variational template) translates into a concrete experimental signal at ImageNet scale, and not merely on the linearized toy problems of \S\ref{app:experiments}.

\section{Related Work}
\label{sec:related}

\paragraph{Generative modeling and the three-objective tradeoff.}
Modern generative modeling balances sample quality, mode coverage, and sampling speed. GANs \citep{goodfellow2014generative, brock2019large} optimize quality and speed but are prone to mode collapse; diffusion and flow models \citep{pmlr-v37-sohl-dickstein15, ho2020denoising, song2021scorebased, lipman2023flow} attain state-of-the-art quality and coverage at the cost of tens to hundreds of network evaluations per sample, with distillation \citep{zhou2024sid, song2023consistency} compressing them via expensive two-stage pipelines. Drifting \citep{deng2026drifting} obtains one-step generation directly; the closest contemporary entries are Inductive Moment Matching \citep{zhou2025imm}, using MMD with Gaussian kernels, and \citet{li2026longshort}, who reinterpret drifting through long-horizon flow maps.

\paragraph{Unification under score matching.}
A parallel line of work has progressively unified continuous generative modeling under the score-matching umbrella: \citet{vincent2011connection} identified denoiser training with score estimation, \citet{song2021scorebased} reframed diffusion's noise prediction via Tweedie's formula, and \citet{gao2025diffusionmeetsflow} placed flow matching in the same family. This unification has been productive, percolating denoising-score-matching theory and convergence analyses across paradigms. Theorem~\ref{thm:score_matching} extends the trajectory to drifting and reduces identifiability to Fourier injectivity. The score-difference identity was independently obtained by \citet{weber2023scoredifference} and \citet{lai2026unifiedviewdriftingscorebased}; neither develops the Fourier stability analysis, Landau-damping diagnosis, annealing schedule, or gradient-flow formalism it enables.

\paragraph{Spectral methods.}
Frequency analyses have been productive across kernel learning and optimization, with a recurring finding that learning is biased toward low frequencies \citep{rahaman2019spectral, tancik2020fourier}. Our analysis is mechanistically distinct from this network-level bias: the timescales we identify are intrinsic to the \emph{kernelized particle dynamics} and persist with a perfect approximator, originating in the dispersion relation $\lambda_\kappa(\xi)=c_\kappa|\xi|^2\hat\kappa(\xi)$. The Landau-damping interpretation \citep{villani2002review, mouhot2011landau} makes this sharp and, to our knowledge, is new in generative modeling. \citet{carrillo2024stein} prove exponential KL convergence for continuous SVGD \citep{liu2016svgd, liu2017svgd_gradient_flow} via a Stein--log-Sobolev inequality but yield only a single global rate; we instead recover mode-resolved timescales.

\paragraph{Optimal transport.}
OT plays increasingly central roles in modern generative modeling: stabilizing GAN losses \citep{salimans2018improvinggansusingoptimal}, straightening flow-matching couplings \citep{tong2024improvinggeneralizingflowbasedgenerative}, and underpinning Schr\"odinger-bridge methods \citep{debortoli2023diffusionschrodingerbridgeapplications}. We use it in a different register as a tool of \emph{analysis}. The Jordan--Kinderlehrer--Otto framework \citep{jordan1998variational, noauthororeditor, Ambrosio2005} reveals drifting as the explicit-Euler discretization of a Wasserstein gradient flow of the smoothed KL energy, with stop-gradient as the required frozen-field structure. This grounds drifting in a well-understood variational theory and, as \S\ref{sec:gradient_flow} shows, opens a constructive route to new operators; our Sinkhorn-divergence drift, drawing on entropic OT \citep{peyre2020computationaloptimaltransport, feydy2018interpolatingoptimaltransportmmd}, is one such instance.
\section{Conclusion}
\label{sec:conclusion}

We began with a simple question, what does the drifting kernel actually
compute? We demonstrate that under a Gaussian kernel, it is a score
difference on smoothed distributions.  This single identity serves as a foundation that helps resolve all three open theoretical questions simultaneously, each through a distinct
analytical lens.

Identifiability follows from Fourier injectivity.  Landau damping offers an explanation for kernel selection: the Gaussian kernel exponentially suppresses
high-frequency modes, justifying the empirical preference for the Laplacian
and motivating the exponential annealing schedule $\sigma(t)=\sigma_0e^{-rt}$
that reduces convergence time from $\exp(O(K_{\max}^2))$ to $O(\log K_{\max})$.
The stop-gradient operator, far from being a heuristic, is the frozen-field
discretization mandated by the JKO scheme; removing it leads to drift collapse,
a spurious minimization that reduces the loss without transporting mass toward
the data distribution.

Beyond resolving these open problems, the gradient-flow formalism yields
a modular template $V=-\nabla(\delta F/\delta q)$ that generalizes drift
construction beyond the original kernel family, demonstrated here with a
Sinkhorn divergence drift.

\section{Limitations and Future Work}
\label{sec:limitations}

\paragraph{Scope of the spectral analysis.} The Landau-damping diagnosis of \S\ref{sec:fourier_stability} is local: it linearizes around equilibrium and treats the background density as homogeneous, so its predicted timescales hold strictly only in the small-perturbation regime. This linearization around fixed points yields necessary conditions and points toward concrete fixes, whose worth is then settled by experiment. The annealing schedule of \S\ref{sec:annealing} is such a fix, and its gains on ImageNet (Table~\ref{tab:imagenet-ablation}) suggest the diagnosis remains informative well beyond its strict assumptions. A fully nonlinear treatment, perhaps via the machinery of \citet{mouhot2011landau}, would extend the guarantees into the regime where most training occurs.

\paragraph{Toward optimal kernels.} Our analysis exposes a tension between analytical structure (the Gaussian kernel yields the score-matching identity and full identifiability) and high-frequency resolution (the Laplacian kernel's polynomial decay resolves fine modes faster). Annealing reconciles the two empirically, but the broader point is that kernel choice jointly controls the available theory and the convergence dynamics. A natural open direction is to design kernels with prescribed spectral profiles, tailored to the target distribution's frequency content, while retaining the Gaussian's analytical scaffolding.

\paragraph{Beyond the JKO discretization.} JKO is one discretization of the Wasserstein gradient flow, and the stop-gradient loss inherits its first-order explicit-Euler character. Other minimizing-movement schemes, higher-order discretizations, splitting methods, or Nesterov-style acceleration on Wasserstein space may yield faster or more stable training. Our score-matching identity and gradient-flow formalism make this principled rather than speculative. A systematic study of these discretizations is a natural next step.

\bibliography{references}
\bibliographystyle{tmlr}

\appendix

\newpage
\section*{Appendix Contents}

To ease navigation, here is the table of contents of the appendix.

\vspace{0.5em}
\noindent
\begin{tabular}{@{}p{0.5em}p{0.5em}l@{\hfill}r@{}}
\multicolumn{3}{@{}l}{\textbf{\hyperref[sec:score_connection]{A\quad Detailed Proofs for Section~\ref{sec:score_connection}}}} & \pageref{app:score} \\[2pt]
& \multicolumn{2}{@{}l}{\hyperref[app:score]{A.1\quad Proof of Proposition~\ref{eq:drift_score} (Gaussian Drift as Score Difference)}} & \pageref{app:score} \\
& \multicolumn{2}{@{}l}{\hyperref[app:identifiability]{A.2\quad Proof of Proposition~\ref{prop:identifiability} (Identifiability)}} & \pageref{app:identifiability} \\[6pt]

\multicolumn{3}{@{}l}{\textbf{\hyperref[sec:continuous]{B\quad Detailed Proofs for Section~\ref{sec:continuous}}}} & \pageref{app:continuous_time} \\[2pt]
& \multicolumn{2}{@{}l}{\hyperref[app:continuous_time]{B.1\quad Continuous-Time Limit}} & \pageref{app:continuous_time} \\
& \multicolumn{2}{@{}l}{\hyperref[app:linearization]{B.2\quad Linearization and Dispersion Relation}} & \pageref{app:linearization} \\
& \multicolumn{2}{@{}l}{\hyperref[thm:general_kernel_timescales]{B.3\quad Proof of Theorem~\ref{thm:general_kernel_timescales}}} & \\
& \multicolumn{2}{@{}l}{\hyperref[app:kernel_time]{B.4\quad Proof of Corollary~\ref{cor:convergence_time}}} & \pageref{app:kernel_time} \\
& \multicolumn{2}{@{}l}{\hyperref[app:annealing]{B.5\quad Proof of Theorem~\ref{thm:annealing_convergence} (Annealing Convergence Time)}} & \pageref{app:annealing} \\[6pt]

\multicolumn{3}{@{}l}{\textbf{\hyperref[app:hitch]{C\quad A Hitchhiker's Guide to Optimal Transport and Wasserstein Gradient Flows}}} & \pageref{app:hitch} \\[6pt]

\multicolumn{3}{@{}l}{\textbf{\hyperref[app:gradient_flow]{D\quad Complete Proofs for Section~\ref{sec:gradient_flow}}}} & \pageref{app:gradient_flow} \\[2pt]
& \multicolumn{2}{@{}l}{\hyperref[app:energy_properties]{D.1\quad Properties of the Smoothed KL Energy}} & \pageref{app:energy_properties} \\
& & \hyperref[app:lsc]{D.1.1\quad Lower Semicontinuity of $F_\sigma$} \pageref{app:lsc} \\
& & \hyperref[app:first_variation]{D.1.2\quad First Variation of $F_\sigma$} \pageref{app:first_variation} \\
& & \hyperref[app:convolution_approximation]{D.1.3\quad The Convolution Approximation Error} \pageref{app:convolution_approximation} \\
& & \hyperref[app:lipschitz_velocity]{D.1.4\quad Lipschitz Continuity of the Velocity Field} \pageref{app:lipschitz_velocity} \\
& \multicolumn{2}{@{}l}{\hyperref[app:JKO]{D.2\quad The JKO Scheme for $F_\sigma$}} & \pageref{app:JKO} \\
& & \hyperref[app:JKO_wellposed]{D.2.1\quad Well-Posedness} \pageref{app:JKO_wellposed} \\
& & \hyperref[app:EL]{D.2.2\quad Euler--Lagrange Condition} \pageref{app:EL} \\
& & \hyperref[app:apriori]{D.2.3\quad A Priori Estimates} \pageref{app:apriori} \\
& & \hyperref[app:convergence]{D.2.4\quad Convergence to Gradient Flow} \pageref{app:convergence} \\
& \multicolumn{2}{@{}l}{\hyperref[app:consistency]{D.3\quad Proof of Consistency (Implicit--Explicit)}} & \pageref{app:consistency} \\
& \multicolumn{2}{@{}l}{\hyperref[app:stopgrad]{D.4\quad Proof of Theorem~\ref{thm:stopgrad} (Stop-Gradient)}} & \pageref{app:stopgrad} \\[6pt]

\multicolumn{3}{@{}l}{\textbf{\hyperref[app:schedule_ablation]{E\quad Schedule Ablation}}} & \pageref{app:schedule_ablation} \\[6pt]

\multicolumn{3}{@{}l}{\textbf{\hyperref[app:experiments]{F\quad Toy Experiments}}} & \pageref{app:experiments} \\[2pt]
& \multicolumn{2}{@{}l}{\hyperref[app:score_check]{F.1\quad Score-matching verification}} & \pageref{app:score_check} \\
& \multicolumn{2}{@{}l}{\hyperref[app:spectral_check]{F.2\quad Spectral convergence times}} & \pageref{app:spectral_check} \\
& \multicolumn{2}{@{}l}{\hyperref[app:stop_grad_check]{F.3\quad Stop-gradient necessity}} & \pageref{app:stop_grad_check} \\
& \multicolumn{2}{@{}l}{\hyperref[app:sinkhorn_check]{F.4\quad Sinkhorn drift feasibility}} & \pageref{app:sinkhorn_check} \\[6pt]

\multicolumn{3}{@{}l}{\textbf{\hyperref[app:image_exp]{G\quad Image Generation Experimental Details}}} & \pageref{app:image_exp} \\[2pt]
& \multicolumn{2}{@{}l}{\hyperref[app:image_exp:setup]{G.1\quad Dataset, representation space, and architecture}} & \pageref{app:image_exp:setup} \\
& \multicolumn{2}{@{}l}{G.2\quad Training protocol} & \\
& \multicolumn{2}{@{}l}{\hyperref[sec:gaussian-drift-impl]{G.3\quad Gaussian kernel drift}} & \pageref{sec:gaussian-drift-impl} \\
& \multicolumn{2}{@{}l}{\hyperref[sec:sinkhorn-implementation]{G.4\quad Sinkhorn-Drift operator implementations}} & \pageref{sec:sinkhorn-implementation} \\
& \multicolumn{2}{@{}l}{\hyperref[app:image_exp:samples]{G.5\quad Qualitative samples}} & \pageref{app:image_exp:samples} \\
& \multicolumn{2}{@{}l}{\hyperref[app:image_exp:compute]{G.6\quad Compute, memory, and implementation notes}} & \pageref{app:image_exp:compute} \\
\end{tabular}

\clearpage

\section{Detailed Proofs for Section~\ref{sec:score_connection}}

\subsection{Proof of Proposition~\ref{eq:drift_score} (Gaussian Drift as Score Difference)}
\label{app:score}

\begin{proof}
Recall the Gaussian kernel
\[
\varphi_\sigma(z) := (2\pi\sigma^2)^{-d/2}\exp\!\Big(-\frac{\|z\|^2}{2\sigma^2}\Big),
\qquad
p_\sigma := p * \varphi_\sigma.
\]
We start from the attractive drift definition
\[
V_p^{+}(x)=
\frac{\int \varphi_\sigma(x-y)(y-x)p(y)\,dy}
{\int \varphi_\sigma(x-y)p(y)\,dy}.
\]
A direct computation gives the gradient identity (differentiate w.r.t.\ $x$)
\[
\nabla_x \varphi_\sigma(x-y)
= -\frac{x-y}{\sigma^2}\,\varphi_\sigma(x-y),
\]
hence
\[
(y-x)\,\varphi_\sigma(x-y) = \sigma^2 \nabla_x \varphi_\sigma(x-y).
\]
Substituting into the numerator,
\begin{align*}
\int \varphi_\sigma(x-y)(y-x)p(y)\,dy
&= \sigma^2 \int \nabla_x \varphi_\sigma(x-y)\,p(y)\,dy \\
&= \sigma^2 \nabla_x \int \varphi_\sigma(x-y)\,p(y)\,dy \\
&= \sigma^2 \nabla_x p_\sigma(x).
\end{align*}
The interchange of $\nabla_x$ and $\int dy$ is justified by dominated convergence:
both $\varphi_\sigma(\cdot)$ and $\nabla\varphi_\sigma(\cdot)$ are bounded by
Gaussian envelopes integrable against $p(y)\,dy$.

The denominator is exactly $p_\sigma(x)$, so
\[
V_p^{+}(x) = \sigma^2 \frac{\nabla p_\sigma(x)}{p_\sigma(x)}
= \sigma^2 \nabla \log p_\sigma(x).
\]
The same calculation applied to the repulsive drift $V_q^{-}$ (with $q$ in place of $p$)
yields
\[
V_q^{-}(x)=\sigma^2 \nabla \log q_\sigma(x),
\qquad q_\sigma := q * \varphi_\sigma.
\]
Therefore,
\[
V^{(\sigma)}_{p,q}(x)=V_p^{+}(x)-V_q^{-}(x)
=\sigma^2 \nabla \log\frac{p_\sigma(x)}{q_\sigma(x)}.
\]
\end{proof}

\subsection{Proof of Proposition~\ref{prop:identifiability} (Identifiability)}
\label{app:identifiability}

\begin{proof}
Assume that for some fixed $\sigma>0$,
\[
V^{(\sigma)}_{p,q}(x) = 0 \quad \text{for all }x.
\]
By \cref{eq:drift_score},
\[
\sigma^2 \nabla \log\frac{p_\sigma(x)}{q_\sigma(x)} = 0
\quad\Rightarrow\quad
\nabla \log\frac{p_\sigma(x)}{q_\sigma(x)} = 0,
\]
so $\log(p_\sigma/q_\sigma)$ is constant on $\R^d$. Hence there exists $C>0$ such that
\[
p_\sigma(x) = C\,q_\sigma(x) \qquad \text{for all }x.
\]
Integrating both sides and using that $p_\sigma,q_\sigma$ are probability densities
(indeed $\int p_\sigma = \int p = 1$ and likewise for $q$),
we obtain $C=1$, hence $p_\sigma=q_\sigma$.

Now take Fourier transforms. Convolution with a Gaussian multiplies Fourier transforms:
\[
\widehat{p_\sigma}(\xi) = \hat p(\xi)\,\widehat{\varphi_\sigma}(\xi),
\qquad
\widehat{\varphi_\sigma}(\xi)=e^{-\sigma^2\|\xi\|^2/2},
\]
and similarly for $q$. Since $p_\sigma=q_\sigma$,
\[
\hat p(\xi)\,e^{-\sigma^2\|\xi\|^2/2}
=
\hat q(\xi)\,e^{-\sigma^2\|\xi\|^2/2}.
\]
Because $e^{-\sigma^2\|\xi\|^2/2}>0$ for all $\xi$, we conclude $\hat p(\xi)=\hat q(\xi)$
for all $\xi$. The Fourier transform uniquely determines a probability measure
(equivalently, characteristic functions are injective), hence $p=q$.
\end{proof}

\section{Detailed Proofs for Section~\ref{sec:continuous}}

\subsection{Continuous-Time Limit}
\label{app:continuous_time}

We sketch the (standard) passage from the discrete particle update to the McKean--Vlasov
continuity equation, emphasizing what is formal and what can be made rigorous.

Consider the particle-level update
\[
x_{n+1} = x_n + \varepsilon\, V_{p,q_n}(x_n),
\qquad \varepsilon>0,
\]
and define the piecewise-constant interpolation
\[
x^\varepsilon(t) := x_n, \qquad t \in [n\varepsilon,(n+1)\varepsilon).
\]
Formally,
\[
\frac{x^\varepsilon(t+\varepsilon)-x^\varepsilon(t)}{\varepsilon}
= V_{p,q_n}(x_n).
\]
If, as $\varepsilon\to 0$:
(i) $q_n \to q(t)$ in $W_2$ at $t=n\varepsilon$,
(ii) for each fixed $q$ the map $x\mapsto V_{p,q}(x)$ is locally Lipschitz with at most
linear growth, and
(iii) the dependence $q\mapsto V_{p,q}$ is continuous in $W_2$ in a suitable sense,
then standard stability results for ODE schemes imply that $x^\varepsilon(\cdot)$
converges (along subsequences) to a limit $x(\cdot)$ satisfying the McKean--Vlasov ODE
\[
\dot x(t) = V_{p,q(t)}(x(t)).
\]

Let $q(t)$ denote the law of $x(t)$. Assume $q(t)$ admits a density and
$v(t,x):=V_{p,q(t)}(x)$ is sufficiently regular so that the chain rule holds. For any
test function $\phi\in C_c^\infty(\R^d)$,
\[
\frac{d}{dt}\phi(x(t)) = \nabla \phi(x(t))\cdot v(t,x(t)).
\]
Taking expectation gives
\[
\frac{d}{dt}\int \phi(x)\,dq(t)(x)
= \int \nabla\phi(x)\cdot v(t,x)\,dq(t)(x).
\]
Integrating by parts (justified since $\phi$ is compactly supported) yields the weak form
\[
\frac{d}{dt}\int \phi\,dq(t)
= -\int \phi(x)\,\nabla\cdot(q(t,x)v(t,x))\,dx,
\]
which is exactly the continuity equation \eqref{eq:continuity} in distributional form:
\[
\partial_t q + \nabla\cdot(q\,V_{p,q}) = 0.
\]

\subsection{Linearization and Dispersion Relation}
\label{app:linearization}

We derive the dispersion relation in a way that separates (a) exact identities from
(b) simplifying assumptions used to obtain a closed-form Fourier rate.

\begin{proof}
Write \(q(t,x)=p(x)+\delta(t,x)\) with \(|\delta|\ll p\). Denote Gaussian smoothing by
\(f_\sigma=\varphi_\sigma * f\), so \(q_\sigma=p_\sigma+\delta_\sigma\).
For the Gaussian kernel the drift admits the score form
\[
V_{p,q}^{(\sigma)}(x)=\sigma^2\nabla\log\frac{p_\sigma(x)}{q_\sigma(x)} .
\]
Substituting \(q=p+\delta\) gives
\[
V_{p,p+\delta}^{(\sigma)}
=\sigma^2\nabla\log\frac{p_\sigma}{p_\sigma+\delta_\sigma}
=-\sigma^2\nabla\log\!\left(1+\frac{\delta_\sigma}{p_\sigma}\right).
\]
Linearizing for small \(\delta\) yields
\[
V_{p,p+\delta}^{(\sigma)}
= -\sigma^2\nabla\!\left(\frac{\delta_\sigma}{p_\sigma}\right)
+O(\delta^2).
\]

The continuity equation
\[
\partial_t q+\nabla\!\cdot(qV_{p,q})=0
\]
then implies, keeping only first-order terms,
\[
\partial_t\delta
=-\nabla\!\cdot(pV_{p,p+\delta}^{(\sigma)})
=\sigma^2\nabla\!\cdot\!\left(p\,\nabla\!\left(\frac{\delta_\sigma}{p_\sigma}\right)\right)
+O(\delta^2).
\]
Under the local-homogeneity approximation \(p_\sigma\approx \mathrm{const}\),
this reduces to
\[
\partial_t\delta(x,t)=\sigma^2\Delta(\delta_\sigma(x,t)).
\]
\end{proof}

\subsection{Proof of Theorem \ref{thm:general_kernel_timescales}}
\emph{Proof} Take the Fourier transform in space of the linearized PDE
\[
\partial_t \delta = c_\kappa \Delta(\kappa * \delta).
\]
Using that convolution becomes multiplication and that
\[
\widehat{\Delta f}(\xi)=-|\xi|^2\hat f(\xi),
\]
we obtain
\[
\partial_t \hat\delta(\xi,t)
= c_\kappa \widehat{\Delta(\kappa * \delta)}(\xi,t)
= -c_\kappa |\xi|^2 \widehat{\kappa * \delta}(\xi,t)
= -c_\kappa |\xi|^2 \hat\kappa(\xi)\hat\delta(\xi,t).
\]
Hence each Fourier mode evolves independently according to
\[
\partial_t\hat\delta(\xi,t)=-\lambda_\kappa(\xi)\hat\delta(\xi,t),
\qquad
\lambda_\kappa(\xi)=c_\kappa|\xi|^2\hat\kappa(\xi).
\]
This is a scalar ODE, with solution
\[
\hat\delta(\xi,t)=e^{-\lambda_\kappa(\xi)t}\hat\delta(\xi,0).
\]
Therefore the mode \(\xi\) decays exponentially at rate \(\lambda_\kappa(\xi)\), so its characteristic convergence timescale is
\[
\tau_\kappa(\xi)=\lambda_\kappa(\xi)^{-1}
=\frac{1}{c_\kappa|\xi|^2\hat\kappa(\xi)}.
\]
This proves the claim. \(\square\)

\paragraph{Exponential kernel.}
We now show that the exponential-kernel drift also linearizes to the form
required by Theorem~\ref{thm:general_kernel_timescales}, but with an effective
convolution kernel different from the original exponential kernel.

Let
\[
k_\tau(r)=e^{-r/\tau}.
\]
Define the companion kernel
\[
h_\tau(r):=\tau(r+\tau)e^{-r/\tau}.
\]
Since
\[
h_\tau'(r)=-r e^{-r/\tau},
\]
we have, for \(r=\|x-y\|\),
\[
\nabla_x h_\tau(\|x-y\|)
=
h_\tau'(r)\frac{x-y}{r}
=
(y-x)e^{-\|x-y\|/\tau}.
\]
Therefore
\[
V_p^+(x)
=
\frac{\int k_\tau(\|x-y\|)(y-x)p(y)\,dy}
{\int k_\tau(\|x-y\|)p(y)\,dy}
=
\frac{\nabla(h_\tau*p)(x)}{(k_\tau*p)(x)}.
\]
Similarly,
\[
V_q^-(x)
=
\frac{\nabla(h_\tau*q)(x)}{(k_\tau*q)(x)}.
\]
Hence
\[
V^{\exp}_{p,q}(x)
=
\frac{\nabla(h_\tau*p)(x)}{(k_\tau*p)(x)}
-
\frac{\nabla(h_\tau*q)(x)}{(k_\tau*q)(x)}.
\]

Now write \(q=p+\delta\), with \(|\delta|\ll p\). Under the same
local-homogeneity approximation used above,
\[
p(x)\approx \rho_0,
\qquad
(k_\tau*p)(x)\approx \rho_0 Z_\tau,
\qquad
\nabla(h_\tau*p)(x)\approx 0,
\]
where
\[
Z_\tau:=\int_{\mathbb R^d}k_\tau(z)\,dz.
\]
Keeping only first-order terms gives
\[
V^{\exp}_{p,p+\delta}(x)
\approx
-\frac{1}{\rho_0 Z_\tau}\nabla(h_\tau*\delta)(x).
\]
Substituting into the continuity equation,
\[
\partial_t q+\nabla\cdot(qV_{p,q})=0,
\]
and retaining only first-order terms yields
\[
\partial_t\delta
=
-\nabla\cdot(pV^{\exp}_{p,p+\delta})
\approx
\frac{1}{Z_\tau}\Delta(h_\tau*\delta).
\]
Equivalently,
\[
\partial_t\delta
=
\Delta(\kappa_\tau*\delta),
\qquad
\kappa_\tau:=\frac{h_\tau}{Z_\tau}.
\]
Thus the exponential-kernel drift falls under
Theorem~\ref{thm:general_kernel_timescales} with \(c_\kappa=1\) and
\(\kappa=\kappa_\tau\).

\subsection{Proof of Corollary \ref{cor:convergence_time}}
\label{app:kernel_time}
\begin{proof}
From Theorem~\ref{thm:general_kernel_timescales}, each Fourier mode evolves as
\[
\hat\delta(k,t)=e^{-\lambda_\kappa(k)t}\hat\delta(k,0),
\qquad
\lambda_\kappa(k)=c_\kappa |k|^2\hat\kappa(k).
\]
To reduce the amplitude of mode \(k\) by a factor \(1/\epsilon\), we require
\[
e^{-\lambda_\kappa(k)T}\le \epsilon
\quad\Longrightarrow\quad
T \ge \frac{\log(1/\epsilon)}{\lambda_\kappa(k)} .
\]
To damp all modes with \(|k|\le K_{\max}\), the required time is governed by the slowest-decaying mode in that range:
\[
T = \max_{|k|\le K_{\max}}\frac{\log(1/\epsilon)}{\lambda_\kappa(k)} .
\]

\paragraph{Gaussian kernel.}
For the Gaussian kernel \(\kappa=\varphi_\sigma\),
\[
\hat\kappa(k)=e^{-\sigma^2|k|^2/2}, 
\qquad c_\kappa=\sigma^2,
\]
so
\[
\lambda_{\mathrm{Gauss}}(k)
=\sigma^2 |k|^2 e^{-\sigma^2|k|^2/2}.
\]
For large \(K_{\max}\), the smallest rate on \([0,K_{\max}]\) occurs at \(k=K_{\max}\), giving
\[
T_{\mathrm{Gauss}}
= \frac{\log(1/\epsilon)}{\sigma^2 K_{\max}^2}
\exp\!\Bigl(\frac{\sigma^2 K_{\max}^2}{2}\Bigr).
\]

\paragraph{Exponential (Laplacian) kernel.}
For the exponential kernel, the effective kernel in
Theorem~\ref{thm:general_kernel_timescales} is
\[
\kappa_\tau=\frac{h_\tau}{Z_\tau},
\qquad
h_\tau(r)=\tau(r+\tau)e^{-r/\tau},
\qquad
Z_\tau=\int_{\mathbb R^d}e^{-\|z\|/\tau}\,dz.
\]
Using \(Z_\tau\asymp \tau^d\) and
\(\widehat h_\tau(k)\asymp \tau^{d+2}(1+\tau^2|k|^2)^{-(d+3)/2}\),
\[
\widehat{\kappa_\tau}(k)
\asymp
\tau^2(1+\tau^2|k|^2)^{-(d+3)/2}.
\]
Hence
\[
\lambda_{\mathrm{exp}}(k)
\asymp
|k|^2\tau^2(1+\tau^2|k|^2)^{-(d+3)/2}.
\]
For large \(|k|\),
\[
\lambda_{\mathrm{exp}}(k)
\asymp
\tau^{-(d+1)}|k|^{-(d+1)}.
\]
Evaluating at the slowest mode \(k=K_{\max}\) gives
\[
T_{\mathrm{exp}}
\asymp
\log(1/\epsilon)\,\tau^{d+1}K_{\max}^{d+1}.
\]
\end{proof}
\subsection{Proof of Theorem~\ref{thm:annealing_convergence} (Annealing Convergence Time)}
\label{app:annealing}

\begin{proof}
Fix any mode $|k| \le K_{\max}$ and define the cumulative decay
\[
\Lambda(k,t) := \int_0^t \sigma(s)^2 |k|^2 e^{-\sigma(s)^2 |k|^2/2} \,ds,
\]
so that $\hat{\delta}(k,t) = \hat{\delta}(k,0) e^{-\Lambda(k,t)}$. Reducing
$|\hat{\delta}(k,t)|$ by a factor $1/\epsilon$ is equivalent to
$\Lambda(k,t) \ge \log(1/\epsilon)$. We analyse the two phases separately.

\paragraph{Annealing phase ($t \in [0, T_{\mathrm{ann}}]$).}
Substituting $u = \sigma(s)^2|k|^2$, so $du = -2ru\,ds$, gives
\[
\Lambda(k, T_{\mathrm{ann}})
= \frac{1}{2r}\int_{\sigma_{\min}^2|k|^2}^{\sigma_0^2|k|^2} e^{-u/2}\,du
= \frac{1}{r}\Bigl[e^{-\sigma_{\min}^2|k|^2/2} - e^{-\sigma_0^2|k|^2/2}\Bigr]
\ge 0,
\]
since the integrand is strictly positive. Modes with $\sigma_{\min}^2|k|^2 < 2$
pass through their peak rate $f(2) = 2/e$ during this phase and accumulate
decay of order $1/r$; they are not the bottleneck.

\paragraph{Constant phase ($t > T_{\mathrm{ann}}$).}
Once $\sigma$ freezes at $\sigma_{\min}$, mode $k$ decays at the fixed rate
$\lambda(k,\sigma_{\min}) = \sigma_{\min}^2|k|^2 e^{-\sigma_{\min}^2|k|^2/2}$.
Under the assumption $\sigma_{\min}^2 K_{\max}^2 \ge 2$, the function
$f(u) = ue^{-u/2}$ is strictly decreasing for $u \ge 2$, so the rate is
minimised at $k = K_{\max}$:
\[
\lambda(k,\sigma_{\min}) \ge \lambda_{\min}
:= \sigma_{\min}^2 K_{\max}^2\exp\!\Bigl(-\tfrac{1}{2}\sigma_{\min}^2 K_{\max}^2\Bigr) > 0,
\qquad \forall\,|k|\le K_{\max}.
\]

\paragraph{Conclusion.}
Combining both phases and using $\Lambda(k, T_{\mathrm{ann}}) \ge 0$,
\[
\Lambda(k,\, T_{\mathrm{ann}} + \Delta t) \ge \lambda_{\min}\,\Delta t.
\]
Setting $\Delta t = \frac{1}{\lambda_{\min}}\log(1/\epsilon)$ ensures
$\Lambda \ge \log(1/\epsilon)$ for all $|k| \le K_{\max}$, yielding
\[
T_{\epsilon,\mathrm{anneal}}
\le \frac{1}{r}\log\!\Bigl(\frac{\sigma_0}{\sigma_{\min}}\Bigr)
+ \frac{1}{\lambda_{\min}(K_{\max})}\log(1/\epsilon).
\]
Choosing $\sigma_{\min} = \sqrt{2}/K_{\max}$ gives $\lambda_{\min} = 2e^{-1}$
independently of $K_{\max}$, so the second term is $\frac{e}{2}\log(1/\epsilon)$
and the first is $\frac{1}{r}\log(\sigma_0 K_{\max}/\sqrt{2})$, establishing
the $O\!\bigl(\frac{1}{r}\log K_{\max}\bigr)$ headline bound.
\end{proof}

\section{A Hitchhiker’s Guide to Optimal Transport and Wasserstein Gradient Flows}
\label{app:hitch}
This appendix summarizes the minimal elements of optimal transport and Wasserstein gradient flows needed in Sections~4--5. We state definitions and structural results without full proofs; detailed treatments can be found in \citet{noauthororeditor}, \citet{villani2009optimal} and \citet{Ambrosio2005}.

Throughout, $\Omega \subset \mathbb{R}^d$ is compact and convex, and $\mathcal{P}_2(\Omega)$ denotes probability measures with finite second moment.

\subsection{Optimal Transport and the $2$-Wasserstein Distance}

\paragraph{Monge formulation.}
Given $\mu,\nu \in \mathcal{P}_2(\Omega)$, the quadratic Monge problem seeks
\begin{equation}
\inf_{T_\# \mu = \nu}
\int_\Omega \|x - T(x)\|^2 \, d\mu(x),
\end{equation}
where $T$ pushes $\mu$ onto $\nu$. This formulation may fail to admit a solution because maps are restrictive.

\paragraph{Kantorovich relaxation.}
The relaxed problem considers transport plans:
\begin{equation}
W_2^2(\mu,\nu)
=
\inf_{\pi \in \Pi(\mu,\nu)}
\int_{\Omega\times\Omega} \|x-y\|^2 \, d\pi(x,y),
\end{equation}
where $\Pi(\mu,\nu)$ is the set of couplings with marginals $\mu$ and $\nu$. This problem always admits a minimizer.

The induced distance $W_2$ satisfies non-negativity, symmetry, and the triangle inequality; thus $(\mathcal{P}_2(\Omega), W_2)$ is a metric space.

\paragraph{Existence of transport maps.}
For quadratic cost, if $\mu$ is absolutely continuous, then the optimal plan is induced by a map
\[
T = \nabla \varphi
\]
for a convex potential $\varphi$ (Brenier’s theorem). In the smooth density setting of Section~5, Wasserstein updates are therefore realized by transport maps.

\subsection{Functionals and First Variations}

A \emph{functional} is a map
\[
F : \mathcal{P}_2(\Omega) \to \mathbb{R}.
\]
Examples used in this paper include:
\begin{itemize}
    \item the smoothed KL divergence $F_\sigma[q] = \sigma^2 \mathrm{KL}(q_\sigma \| p_\sigma)$,
    \item entropic optimal transport divergences,
    \item entropy $\int q \log q$.
\end{itemize}

\paragraph{First variation.}
Let $q \in \mathcal{P}_2(\Omega)$ with density and consider perturbations
\[
q_s = (Id + s\xi)_\# q,
\]
for smooth compactly supported vector fields $\xi$. The first variation of $F$ at $q$ is defined via
\begin{equation}
\frac{d}{ds} F[q_s]\Big|_{s=0}
=
\int_\Omega
\nabla \frac{\delta F}{\delta q}(x)
\cdot \xi(x)\,
q(x)\,dx.
\end{equation}
The function $\frac{\delta F}{\delta q}$ is the \emph{functional derivative} of $F$.

For example,
\[
\frac{\delta}{\delta q} \mathrm{KL}(q\|p)
=
\log\frac{q}{p} + 1.
\]

\subsection{Continuity Equation and Velocity Fields}

If particles evolve according to
\[
\dot x(t) = v(t,x),
\]
then their density satisfies the continuity equation
\begin{equation}
\partial_t q + \nabla\cdot(q v) = 0.
\end{equation}
Thus probability evolution is fully determined by a velocity field.

\subsection{Wasserstein Gradient Flows}

The space $\mathcal{P}_2(\Omega)$ admits a formal Riemannian structure in which:
\begin{itemize}
    \item tangent vectors at $q$ are velocity fields $v$,
    \item the metric is
    \[
    \|v\|_{T_q}^2
    =
    \int_\Omega \|v(x)\|^2 q(x)\,dx.
    \]
\end{itemize}

Under this geometry, the steepest descent flow of a functional $F$ satisfies
\begin{equation}
\partial_t q
=
\nabla \cdot
\left(
q \nabla \frac{\delta F}{\delta q}
\right).
\end{equation}
Equivalently,
\[
\partial_t q + \nabla\cdot(q v)=0,
\qquad
v = -\nabla \frac{\delta F}{\delta q}.
\]
Thus the velocity field is the negative spatial gradient of the functional derivative. This principle underlies the drift structure derived in Section~5.

\subsection{The JKO Scheme (Minimizing Movements)}

The Wasserstein gradient flow of $F$ can be discretized via the Jordan--Kinderlehrer--Otto (JKO) scheme:
\begin{equation}
q_{n+1}
=
\arg\min_{q\in\mathcal{P}_2(\Omega)}
\left\{
F[q]
+
\frac{1}{2\tau}
W_2^2(q,q_n)
\right\}.
\end{equation}
This is the implicit Euler scheme in Wasserstein space.

\paragraph{First-order optimality.}
Let $T_{n+1}$ be the optimal map from $q_{n+1}$ to $q_n$. Then the Euler--Lagrange condition yields
\begin{equation}
\frac{Id - T_{n+1}}{\tau}
=
- \nabla
\frac{\delta F}{\delta q}(q_{n+1}).
\end{equation}
Thus
\[
T_{n+1}(x)
=
x
+
\tau v_{n+1}(x),
\qquad
v_{n+1}
=
-\nabla \frac{\delta F}{\delta q}(q_{n+1}).
\]
As $\tau \to 0$, the JKO interpolants converge to the solution of the gradient-flow PDE.

\subsection{Frozen-Field Approximation}

The JKO update is implicit because the velocity depends on $q_{n+1}$. A tractable explicit approximation freezes the velocity at $q_n$:
\[
q_{n+1}
\approx
(Id + \tau v[q_n])_\# q_n.
\]
This is the explicit Euler discretization of the Wasserstein gradient flow. Section~5 shows that the stop-gradient training objective implements precisely this frozen-field step.

\paragraph{Connection to this work.}
In Section~5:
\begin{itemize}
    \item The smoothed KL energy defines a functional $F_\sigma$.
    \item Its functional derivative yields the drift velocity.
    \item The drifting PDE is its Wasserstein gradient flow.
    \item The stop-gradient loss corresponds to the frozen-field explicit Euler approximation of the JKO scheme.
\end{itemize}
This geometric viewpoint applies to any sufficiently regular divergence functional.

\section{Complete Proofs for Section~\ref{sec:gradient_flow}}
\label{app:gradient_flow}

\paragraph{Standing assumptions (densities).}
Throughout this appendix we work on a compact convex domain
$\Omega \subset \R^d$ with nonempty interior. We assume that all measures
$q \in \Pc(\Omega)$ considered in the JKO scheme admit densities (still denoted $q$)
with respect to Lebesgue measure on $\Omega$, and we extend them by $0$ outside $\Omega$.
Moreover, for each fixed $\sigma>0$ we assume the iterates satisfy
\begin{equation}
\label{eq:standing_bounds}
0 < m_\sigma \;\le\; q_\sigma(x) \;\le\; M_\sigma < \infty
\qquad \text{for all } x \in \Omega,
\end{equation}
for constants $m_\sigma,M_\sigma$ that may depend on $\sigma$ and $\Omega$ but not on
the iterate index. (On a compact $\Omega$, a uniform lower bound holds automatically for
$q_\sigma$ when $q$ is a probability measure, since $\varphi_\sigma$ is strictly positive and
$\Omega$ is bounded; a uniform upper bound is mild and holds, e.g., if $q\in L^\infty(\Omega)$.)

We fix $\sigma>0$ and a data distribution $p\in\Pc(\Omega)$ with density $p$.
We define
\[
q_\sigma := q * \varphi_\sigma,\qquad
p_\sigma := p * \varphi_\sigma,
\]
where convolution is taken in $\R^d$ after extending $q,p$ by $0$ outside $\Omega$.
The energy functional is
\[
F_\sigma[q] := \sigma^2\,\KL(q_\sigma \| p_\sigma)
= \sigma^2 \int_{\R^d} q_\sigma(x)\log\frac{q_\sigma(x)}{p_\sigma(x)}\,dx.
\]
Because $\varphi_\sigma$ is smooth and strictly positive and $\Omega$ is bounded,
$q_\sigma,p_\sigma \in C^\infty(\R^d)$ and $q_\sigma,p_\sigma>0$ everywhere.

\subsection{Properties of the Smoothed KL Energy (Proposition~\ref{prop:first_variation_drift})}
\label{app:energy_properties}

We collect the analytic properties of $F_\sigma$ that underpin the variational theory:
lower semicontinuity (needed for existence in the JKO scheme),
the first variation (which recovers the drift velocity), and the
convolution approximation error (which quantifies the gap between
$v_\sigma[q]$ and $-\sigma^2\nabla\log(q_\sigma/p_\sigma)$).

\subsubsection{Lower Semicontinuity of $F_\sigma$
(Proposition~\ref{prop:first_variation_drift}\textup{(i)})}
\label{app:lsc}

\begin{proof}
Let $q^{(j)}\to q$ narrowly in $\Pc(\Omega)$.
Since $\varphi_\sigma$ is bounded and continuous,
\[
q^{(j)}_\sigma(x) = \int_\Omega \varphi_\sigma(x-y)\,q^{(j)}(dy)
\;\longrightarrow\;
\int_\Omega \varphi_\sigma(x-y)\,q(dy) = q_\sigma(x)
\]
pointwise for every $x\in\R^d$.
Moreover, on a compact $\Omega$ the smoothed densities are uniformly bounded:
$q^{(j)}_\sigma, q_\sigma \in [m_\sigma, M_\sigma]$ for all $j$.

The integrand $u\mapsto u\log(u/p_\sigma(x))$ is continuous and convex on $(0,\infty)$.
On the bounded interval $[m_\sigma,M_\sigma]$ it is uniformly bounded, so
dominated convergence gives
\[
\int_{\R^d} q^{(j)}_\sigma(x)\log\frac{q^{(j)}_\sigma(x)}{p_\sigma(x)}\,dx
\;\longrightarrow\;
\int_{\R^d} q_\sigma(x)\log\frac{q_\sigma(x)}{p_\sigma(x)}\,dx.
\]
In particular, $F_\sigma[q^{(j)}]\to F_\sigma[q]$, which is stronger than lower semicontinuity.
(In fact, under the standing bounds~\eqref{eq:standing_bounds},
$F_\sigma$ is continuous with respect to narrow convergence on $\Pc(\Omega)$.)
\end{proof}

\subsubsection{First Variation of $F_\sigma$}
\label{app:first_variation}

\begin{proof}[Proof of Proposition~\ref{prop:first_variation_drift}\textup{(ii,\,iii)}]
We compute the directional derivative of $F_\sigma$ along transport perturbations.

Let $q \in \Pc(\Omega)$ (with density) and let $\xi \in C_c^\infty(\Omega;\R^d)$.
For $|s|$ small, define $T_s := \mathrm{Id} + s\xi$ and $q_s := (T_s)_\# q$.
Then $q_s$ satisfies the continuity equation in the sense of distributions:
\[
\partial_s q_s\big|_{s=0} = -\nabla\cdot(q\,\xi).
\]
Convolving with $\varphi_\sigma$ and using commutation of convolution with derivatives,
\[
\partial_s q_{s,\sigma}\big|_{s=0}
= ( \partial_s q_s|_{s=0}) * \varphi_\sigma
= -(\nabla\cdot(q\xi))*\varphi_\sigma
= -\nabla\cdot\bigl((q\xi)*\varphi_\sigma\bigr).
\]

Differentiate $F_\sigma[q_s]$ at $s=0$:
\begin{align*}
\frac{d}{ds}F_\sigma[q_s]\Big|_{s=0}
&= \sigma^2 \int_{\R^d} \Bigl(1+\log\frac{q_\sigma}{p_\sigma}\Bigr)(x)\,
\partial_s q_{s,\sigma}\big|_{s=0}(x)\,dx \\
&= -\sigma^2 \int_{\R^d} \Bigl(1+\log\frac{q_\sigma}{p_\sigma}\Bigr)(x)\,
\nabla\cdot\bigl((q\xi)*\varphi_\sigma\bigr)(x)\,dx.
\end{align*}
We now integrate by parts on $\R^d$. This is justified because $(q\xi)*\varphi_\sigma$
decays at least Gaussianly as $\|x\|\to\infty$ (since $q\xi$ is compactly supported in $\Omega$),
while $\nabla\log(q_\sigma/p_\sigma)$ has at most polynomial growth; hence boundary terms vanish.
Thus,
\begin{align*}
\frac{d}{ds}F_\sigma[q_s]\Big|_{s=0}
&= \sigma^2 \int_{\R^d} \nabla\log\frac{q_\sigma}{p_\sigma}(x)\cdot
\bigl((q\xi)*\varphi_\sigma\bigr)(x)\,dx.
\end{align*}

Write $(q\xi)*\varphi_\sigma(x)=\int_\Omega \varphi_\sigma(x-y)\,\xi(y)\,q(y)\,dy$ and apply Fubini:
\begin{align*}
\frac{d}{ds}F_\sigma[q_s]\Big|_{s=0}
&= \sigma^2 \int_\Omega \left(\int_{\R^d}
\nabla\log\frac{q_\sigma}{p_\sigma}(x)\,\varphi_\sigma(x-y)\,dx\right)\cdot \xi(y)\,q(y)\,dy \\
&= \int_\Omega \sigma^2\bigl(\varphi_\sigma * \nabla\log(q_\sigma/p_\sigma)\bigr)(y)\cdot\xi(y)\,q(y)\,dy.
\end{align*}

By the definition of first variation in Wasserstein calculus, this identifies
\[
\nabla \frac{\delta F_\sigma}{\delta q}(y)
= \sigma^2\bigl(\varphi_\sigma * \nabla\log(q_\sigma/p_\sigma)\bigr)(y),
\]
and hence (up to an additive constant)
\[
\frac{\delta F_\sigma}{\delta q}(y)
= \sigma^2\bigl(\varphi_\sigma * \log(q_\sigma/p_\sigma)\bigr)(y) + C.
\]
Regularity follows since $\log(q_\sigma/p_\sigma)\in C^\infty(\R^d)$ and convolution with
$\varphi_\sigma$ preserves smoothness.
\end{proof}

\subsubsection{The Convolution Approximation Error}
\label{app:convolution_approximation}

The velocity field $v_\sigma[q] = -\nabla(\delta F_\sigma/\delta q)$
differs from $-\sigma^2\nabla\log(q_\sigma/p_\sigma)$ by a convolution smoothing.
The following quantifies this gap.

\begin{proposition}
\label{prop:convolution_error}
Let $g \in C^2(\R^d)$. Then for all $x \in \R^d$:
\[
\bigl|(\varphi_\sigma * g)(x) - g(x)\bigr|
\;\le\;
\frac{\sigma^2 d}{2}\,\|\nabla^2 g\|_\infty.
\]
In particular, the velocity field satisfies
\[
\bigl\|
v_\sigma[q]
+ \sigma^2 \nabla\log(q_\sigma/p_\sigma)
\bigr\|_\infty
\;\le\;
\frac{\sigma^4 d}{2}\,
\bigl\|\nabla^2 \log(q_\sigma/p_\sigma)\bigr\|_\infty.
\]
\end{proposition}

\begin{proof}
Taylor-expand $g$ around $x$:
\[
g(x + z) = g(x) + \nabla g(x) \cdot z
+ \frac{1}{2} z^\top \nabla^2 g(\xi) z
\]
for some $\xi$ on the segment between $x$ and $x+z$.
Integrate against $\varphi_\sigma(z)\,dz$. The linear term vanishes by symmetry.
The remainder is bounded by
$\frac12\|\nabla^2 g\|_\infty\,\E[\|Z\|^2] = \frac{\sigma^2 d}{2}\|\nabla^2 g\|_\infty$
for $Z\sim\mathcal N(0,\sigma^2 I)$.
Apply this with $g=\log(q_\sigma/p_\sigma)$ and take a gradient, producing the extra $\sigma^2$ factor.
\end{proof}

\subsubsection{Lipschitz Continuity of the Velocity Field}
\label{app:lipschitz_velocity}

The following Lipschitz estimate is used in the convergence proof
(Theorem~\ref{thm:convergence}) and the consistency bound $W_2(\tilde q_{n+1}^\tau,q_{n+1}^\tau)=O(\tau^{3/2})$.

\begin{lemma}
\label{lem:lipschitz_velocity}
For fixed $\sigma > 0$, under the standing bounds \eqref{eq:standing_bounds} the map
$\mu \mapsto v_\sigma[\mu]$ is Lipschitz from $(\Pc(\Omega), W_2)$ to
$(C(\Omega;\R^d), \|\cdot\|_\infty)$.
\end{lemma}

\begin{proof}
Let $\mu,\nu\in\Pc(\Omega)$. By definition,
\[
v_\sigma[\mu]-v_\sigma[\nu]
= -\sigma^2 \nabla\Big(\varphi_\sigma * \bigl[\log(\mu_\sigma/p_\sigma)-\log(\nu_\sigma/p_\sigma)\bigr]\Big)
= -\sigma^2 \nabla\Big(\varphi_\sigma * \log(\mu_\sigma/\nu_\sigma)\Big).
\]
Thus,
\[
\|v_\sigma[\mu]-v_\sigma[\nu]\|_\infty
\le \sigma^2 \|\nabla\varphi_\sigma\|_{L^1(\R^d)}\,
\|\log(\mu_\sigma/\nu_\sigma)\|_{L^\infty(\Omega)}.
\]
Using the lower bound $m_\sigma$ in \eqref{eq:standing_bounds} and the mean value theorem for $\log$,
\[
\|\log(\mu_\sigma/\nu_\sigma)\|_\infty
\le \frac{1}{m_\sigma}\|\mu_\sigma-\nu_\sigma\|_\infty.
\]
Finally, for any $\mu,\nu\in\Pc(\Omega)$ and $x\in\Omega$,
the map $y\mapsto\varphi_\sigma(x-y)$ is Lipschitz with constant $\|\nabla\varphi_\sigma\|_\infty$, hence
\begin{equation}
\label{eq:uniform_smoothing_conv}
\|\mu_\sigma-\nu_\sigma\|_{L^\infty(\Omega)}
\le \|\nabla\varphi_\sigma\|_\infty\,W_1(\mu,\nu)
\le \|\nabla\varphi_\sigma\|_\infty\,W_2(\mu,\nu),
\end{equation}
where we used $W_1\le W_2$ on a bounded metric space.
Combining the inequalities yields
\[
\|v_\sigma[\mu]-v_\sigma[\nu]\|_\infty \le L_\sigma\,W_2(\mu,\nu),
\]
for a constant $L_\sigma$ depending on $\sigma$, $\Omega$, and $m_\sigma$.
\end{proof}

\subsection{The JKO Scheme for $F_\sigma$}
\label{app:JKO}

We now establish that the JKO minimizing-movement scheme applied to
$F_\sigma$ is well-posed, enjoys monotone energy descent, and converges
as $\tau\to 0$ to a solution of the gradient-flow PDE
$\partial_t q + \nabla\cdot(q\,v_\sigma[q]) = 0$. Our treatment follows closely the exposition in Chapter 8 in \citep{noauthororeditor}.

\subsubsection{Well-Posedness: Existence, Uniqueness, and Energy Descent}
\label{app:JKO_wellposed}

\begin{proposition}[JKO well-posedness]
\label{prop:JKO_wellposed}
For each $n\ge 0$ and $\tau>0$, the JKO functional
$\mathcal{J}(q) := F_\sigma[q] + \frac{1}{2\tau}W_2^2(q,q_n^\tau)$
admits a unique minimizer $q_{n+1}^\tau \in \Pc(\Omega)$, and the energy decreases monotonically:
$F_\sigma[q_{n+1}^\tau] \le F_\sigma[q_n^\tau]$.
\end{proposition}

\begin{proof}
Define the JKO functional
\[
\mathcal{J}(q) := F_\sigma[q] + \frac{1}{2\tau}W_2^2(q,q_n^\tau),
\qquad q\in\Pc(\Omega).
\]

\textbf{(i) Existence.}
Let $(q^{(j)})_{j\ge1}$ be a minimizing sequence.
Since $\Omega$ is compact, $\Pc(\Omega)$ is compact for narrow convergence (equivalently for $W_2$).
Hence, up to a subsequence, $q^{(j)}\to \hat q$ narrowly.
By Proposition~\ref{prop:first_variation_drift}(i), $F_\sigma$ is narrowly lower semicontinuous.
On a compact domain, $W_2(\cdot,q_n^\tau)$ is continuous under narrow convergence.
Therefore $\mathcal{J}(\hat q)\le\liminf_j \mathcal{J}(q^{(j)})=\inf\mathcal{J}$ and $\hat q$
is a minimizer.

\textbf{(ii) Uniqueness.}
We show strict convexity of $\mathcal{J}$ along \emph{mixtures}.
The map $q\mapsto q_\sigma$ is linear, and the functional
$\rho\mapsto \KL(\rho\|p_\sigma)=\int \rho\log(\rho/p_\sigma)$ is strictly convex in $\rho$
on strictly positive densities.
Gaussian convolution is injective on probability measures, so $q_1\neq q_2$ implies
$q_{1,\sigma}\neq q_{2,\sigma}$. Hence for $\lambda\in(0,1)$,
\[
F_\sigma[\lambda q_1+(1-\lambda)q_2]
= \sigma^2\KL(\lambda q_{1,\sigma}+(1-\lambda)q_{2,\sigma}\|p_\sigma)
< \lambda F_\sigma[q_1] + (1-\lambda)F_\sigma[q_2].
\]
Moreover, $q\mapsto W_2^2(q,q_n^\tau)$ is convex along mixtures: if $\pi_i$ is optimal (or $\varepsilon$-optimal)
between $q_i$ and $q_n^\tau$, then $\lambda\pi_1+(1-\lambda)\pi_2$ is a coupling between
$\lambda q_1+(1-\lambda)q_2$ and $q_n^\tau$, giving
\[
W_2^2(\lambda q_1+(1-\lambda)q_2,q_n^\tau)\le \lambda W_2^2(q_1,q_n^\tau)+(1-\lambda)W_2^2(q_2,q_n^\tau).
\]
Thus $\mathcal{J}$ is strictly convex along mixtures, hence its minimizer is unique.

\textbf{(iii) Energy descent.}
By optimality of $q_{n+1}^\tau$ and choosing the competitor $q=q_n^\tau$,
\[
F_\sigma[q_{n+1}^\tau] + \frac{1}{2\tau}W_2^2(q_{n+1}^\tau,q_n^\tau)
\le F_\sigma[q_n^\tau],
\]
which implies $F_\sigma[q_{n+1}^\tau]\le F_\sigma[q_n^\tau]$.
\end{proof}

\subsubsection{Euler--Lagrange Condition
(Equation~\ref{eq:velocity_identification})}
\label{app:EL}

\begin{proof}
Let $\hat q:=q_{n+1}^\tau$ denote the unique minimizer of the JKO step
\[
\hat q\in\argmin_{q\in\Pc(\Omega)}\left\{F_\sigma[q]+\frac{1}{2\tau}W_2^2(q,q_n^\tau)\right\}.
\]
Under the standing density assumptions, $\hat q$ is absolutely continuous on $\Omega$.
Hence Brenier's theorem applies: there exists an optimal transport map
$T_n:\Omega\to\Omega$ pushing $\hat q$ to $q_n^\tau$, and a Kantorovich potential
$\bar\varphi$ (unique up to additive constants) such that
\[
T_n(x)=x-\nabla\bar\varphi(x)\qquad \hat q\text{-a.e.}
\]

\textbf{Step 1: First-order optimality of $\hat q$.}
Consider perturbations $q_\varepsilon=(\mathrm{Id}+\varepsilon\xi)_\#\hat q$ with $\xi\in C_c^\infty(\Omega;\R^d)$.
Optimality of $\hat q$ implies
\[
\frac{d}{d\varepsilon}\left(F_\sigma[q_\varepsilon]+\frac{1}{2\tau}W_2^2(q_\varepsilon,q_n^\tau)\right)\Big|_{\varepsilon=0}=0.
\]

For the $F_\sigma$ term, Proposition~\ref{prop:first_variation_drift}(ii) yields
\[
\frac{d}{d\varepsilon}F_\sigma[q_\varepsilon]\Big|_{\varepsilon=0}
= \int_\Omega \nabla\left(\frac{\delta F_\sigma}{\delta q}\right)(x)\cdot \xi(x)\,\hat q(x)\,dx,
\qquad
\frac{\delta F_\sigma}{\delta q}(x)=\sigma^2(\varphi_\sigma * \log(\hat q_\sigma/p_\sigma))(x)+C.
\]

For the Wasserstein term, standard OT first-variation calculus (e.g., \citep{noauthororeditor},
Prop.~7.17) gives
\[
\frac{d}{d\varepsilon}\frac{1}{2\tau}W_2^2(q_\varepsilon,q_n^\tau)\Big|_{\varepsilon=0}
= -\int_\Omega \frac{1}{\tau}\nabla\bar\varphi(x)\cdot\xi(x)\,\hat q(x)\,dx.
\]

Summing and using arbitrariness of $\xi$ gives, $\hat q$-a.e.,
\[
\nabla\left(\sigma^2(\varphi_\sigma * \log(\hat q_\sigma/p_\sigma))\right)(x)
-\frac{1}{\tau}\nabla\bar\varphi(x)=0,
\]
hence (integrating in $x$) there exists a constant $C$ such that
\[
\sigma^2(\varphi_\sigma * \log(\hat q_\sigma/p_\sigma))(x)+\frac{1}{\tau}\bar\varphi(x)=C
\qquad \hat q\text{-a.e.},
\]
which is \eqref{eq:velocity_identification}.

\textbf{Step 2: Velocity identification.}
Differentiating the previous equation:
\[
\frac{1}{\tau}\nabla\bar\varphi(x)
= -\nabla\Big(\sigma^2(\varphi_\sigma * \log(\hat q_\sigma/p_\sigma))\Big)(x).
\]
Using $T_n(x)=x-\nabla\bar\varphi(x)$ yields
\[
\frac{T_n(x)-x}{\tau}
= -\frac{1}{\tau}\nabla\bar\varphi(x)
= -\nabla\Big(\sigma^2(\varphi_\sigma * \log(\hat q_\sigma/p_\sigma))\Big)(x)
= v_\sigma[\hat q](x),
\qquad \hat q\text{-a.e.},
\]
which is \eqref{eq:velocity_identification}.
\end{proof}

\subsubsection{A Priori Estimates}
\label{app:apriori}

\begin{proposition}[A priori estimates]
\label{prop:apriori}
Let $(q_k^\tau)_{k\ge 0}$ be the JKO iterates and let $\tilde q^\tau$ denote the
$W_2$-geodesic interpolation between successive iterates.
\begin{enumerate}[label=\textup{(\roman*)}, leftmargin=*]
\item \textup{(Energy summability.)}
$\displaystyle\sum_{k=0}^{N-1}\frac{W_2^2(q_{k+1}^\tau,q_k^\tau)}{\tau}
\le 2(F_\sigma[q_0]-\inf F_\sigma)=:C_0$.
\item \textup{(H\"older regularity.)}
$W_2(\tilde q^\tau(t),\tilde q^\tau(s)) \le C_0^{1/2}\,|t-s|^{1/2}$
for all $0\le s < t \le T$.
\end{enumerate}
\end{proposition}

\begin{proof}
\textbf{(i) Energy summability.}
By optimality of $q_{k+1}^\tau$ and competitor $q_k^\tau$,
\begin{equation}
\label{eq:energy_step_app}
F_\sigma[q_{k+1}^\tau]+\frac{1}{2\tau}W_2^2(q_{k+1}^\tau,q_k^\tau)\le F_\sigma[q_k^\tau].
\end{equation}
Rearrange and sum over $k=0,\dots,N-1$:
\[
\sum_{k=0}^{N-1}\frac{W_2^2(q_{k+1}^\tau,q_k^\tau)}{\tau}
\le 2\sum_{k=0}^{N-1}(F_\sigma[q_k^\tau]-F_\sigma[q_{k+1}^\tau])
=2(F_\sigma[q_0]-F_\sigma[q_N^\tau])
\le 2(F_\sigma[q_0]-\inf F_\sigma)=:C_0.
\]

\textbf{(ii) H\"older regularity.}
Let $\tilde q^\tau$ be the $W_2$-geodesic interpolation between successive iterates.
Its metric derivative satisfies for $t\in(k\tau,(k+1)\tau)$:
\[
|(\tilde q^\tau)'|(t)=\frac{W_2(q_{k+1}^\tau,q_k^\tau)}{\tau}.
\]
Thus for $0\le s<t\le T$, by Cauchy--Schwarz,
\begin{align*}
W_2(\tilde q^\tau(t),\tilde q^\tau(s))
&\le \int_s^t |(\tilde q^\tau)'|(r)\,dr
\le |t-s|^{1/2}\left(\int_s^t |(\tilde q^\tau)'|^2(r)\,dr\right)^{1/2} \\
&\le |t-s|^{1/2}\left(\int_0^T |(\tilde q^\tau)'|^2(r)\,dr\right)^{1/2}
= |t-s|^{1/2}\left(\sum_k \frac{W_2^2(q_{k+1}^\tau,q_k^\tau)}{\tau}\right)^{1/2}
\le C_0^{1/2}|t-s|^{1/2}.
\end{align*}
\end{proof}

\subsubsection{Convergence to Gradient Flow}
\label{app:convergence}

\begin{theorem}[Convergence to gradient flow]
\label{thm:convergence}
As $\tau\to 0$, the JKO interpolants converge (up to subsequences) uniformly in $W_2$
to a limit $q \in C^{0,1/2}([0,T];\Pc(\Omega))$ satisfying
$\partial_t q + \nabla\cdot(q\,v_\sigma[q]) = 0$
in distributions, with $F_\sigma[q(t)] \le F_\sigma[q(s)]$ for $0\le s<t\le T$
and $q(0)=q_0$.
\end{theorem}

\begin{proof}
\textbf{Step 1: Compactness and extraction.}
By Proposition~\ref{prop:apriori}(ii), $\{\tilde q^\tau\}_{\tau>0}$ is equicontinuous in
$C([0,T];\Pc(\Omega))$ with the $W_2$ metric. Since $\Pc(\Omega)$ is compact (compact $\Omega$),
Arzel\`a--Ascoli gives $\tau_j\to0$ and $q\in C^{0,1/2}([0,T];\Pc(\Omega))$ such that
$\tilde q^{\tau_j}\to q$ uniformly in $W_2$ on $[0,T]$.
The piecewise-constant interpolation $q^{\tau}(t)$ converges to the same limit.

\textbf{Step 2: Continuity equation for the geodesic interpolant and momentum bounds.}
Fix $k\ge 0$ and let $T_{k+1}:\Omega\to\Omega$ be the optimal transport map pushing
$q_{k+1}^\tau$ to $q_k^\tau$. By Euler-Lagrange condition ~\ref{eq:velocity_identification},
\begin{equation}
\label{eq:EL_map_app_convergence}
T_{k+1}(x) = x + \tau\,v_\sigma[q_{k+1}^\tau](x)
\qquad q_{k+1}^\tau\text{-a.e.}
\end{equation}
Define the displacement interpolation on $(k\tau,(k+1)\tau]$ by
\[
s(t):=\frac{t-k\tau}{\tau}\in(0,1],\qquad
X_{k+1,s}:=(1-s)\,\Id + s\,T_{k+1},
\qquad
\tilde q^\tau(t):=(X_{k+1,s(t)})_\# q_{k+1}^\tau.
\]
(Thus $\tilde q^\tau(k\tau^+)=q_{k+1}^\tau$ and $\tilde q^\tau((k+1)\tau)=q_k^\tau$.)

On this interval the curve $\tilde q^\tau$ solves the continuity equation
\[
\partial_t \tilde q^\tau + \nabla\cdot(\tilde q^\tau\,w^\tau)=0
\quad\text{in }\mathcal D'((0,T)\times\Omega),
\]
with the (a.e.-defined) velocity field
\[
w^\tau(t,\cdot):=\frac{T_{k+1}-\Id}{\tau}\circ (X_{k+1,s(t)})^{-1}.
\]
Moreover, the Benamou--Brenier action along each geodesic segment yields
\begin{equation}
\label{eq:action_bound_app}
\int_{k\tau}^{(k+1)\tau}\!\!\int_\Omega \|w^\tau(t,x)\|^2\,d\tilde q^\tau(t)(x)\,dt
= \frac{1}{\tau}\,W_2^2(q_{k+1}^\tau,q_k^\tau).
\end{equation}
Summing over $k$ and using Proposition~\ref{prop:apriori}(i) gives the uniform bound
\[
\int_0^T \|w^\tau(t)\|_{L^2(\tilde q^\tau(t))}^2\,dt
= \sum_k \frac{W_2^2(q_{k+1}^\tau,q_k^\tau)}{\tau}
\le C_0.
\]
Define the flux measures $E^\tau := \tilde q^\tau\,w^\tau$ on $[0,T]\times\Omega$.
The above bound implies $\{E^\tau\}$ is uniformly bounded as vector-valued Radon measures,
so along the subsequence $\tau_j$ we have $E^{\tau_j}\stackrel{*}{\rightharpoonup}E$ and,
passing to the limit in the continuity equation,
\[
\partial_t q + \nabla\cdot E = 0
\quad\text{in }\mathcal D'((0,T)\times\Omega).
\]

\textbf{Step 3: Identification $E=q\,v_\sigma[q]$.}
We first relate $w^\tau$ to the JKO velocity at the right endpoint.
By \eqref{eq:EL_map_app_convergence} we have
\[
\frac{T_{k+1}-\Id}{\tau} = v_\sigma[q_{k+1}^\tau]
\qquad q_{k+1}^\tau\text{-a.e.}
\]
Hence for $t\in(k\tau,(k+1)\tau]$,
\[
w^\tau(t,\cdot) = v_\sigma[q_{k+1}^\tau]\circ (X_{k+1,s(t)})^{-1}
\qquad \tilde q^\tau(t)\text{-a.e.}
\]
Next, observe that for $t\in(k\tau,(k+1)\tau]$,
\[
W_2\bigl(\tilde q^\tau(t),\,q_{k+1}^\tau\bigr)
\le W_2(q_{k+1}^\tau,q_k^\tau),
\]
since displacement interpolants are constant-speed geodesics.
Therefore, using Lemma~\ref{lem:lipschitz_velocity} and Proposition~\ref{prop:apriori}(i),
\[
\int_0^T \bigl\|v_\sigma[\tilde q^\tau(t)]-v_\sigma[q_{k(t)+1}^\tau]\bigr\|_\infty\,dt
\;\le\; L_\sigma \int_0^T W_2(\tilde q^\tau(t),q_{k(t)+1}^\tau)\,dt
\;\le\; L_\sigma \sum_k \tau\,W_2(q_{k+1}^\tau,q_k^\tau)
\;\xrightarrow[\tau\to0]{}\;0,
\]
where $k(t)$ denotes the unique index with $t\in(k\tau,(k+1)\tau]$ and we used
Cauchy--Schwarz together with $\sum_k \frac{W_2^2(q_{k+1}^\tau,q_k^\tau)}{\tau}\le C_0$.

Now test against any $f\in C([0,T]\times\Omega;\R^d)$.
Using the representation of $E^\tau$ and the previous estimate,
\begin{align*}
\int_0^T\!\!\int_\Omega f\cdot dE^\tau
&= \int_0^T\!\!\int_\Omega f(t,x)\cdot w^\tau(t,x)\,d\tilde q^\tau(t)(x)\,dt \\
&= \int_0^T\!\!\int_\Omega f(t,x)\cdot v_\sigma[\tilde q^\tau(t)](x)\,d\tilde q^\tau(t)(x)\,dt
\;+\; o(1).
\end{align*}
Passing to the limit $\tau_j\to0$ using $\tilde q^{\tau_j}\to q$ uniformly in $W_2$ and
the Lipschitz continuity of $v_\sigma[\cdot]$ (Lemma~\ref{lem:lipschitz_velocity}) yields
\[
\int_0^T\!\!\int_\Omega f\cdot dE
= \int_0^T\!\!\int_\Omega f(t,x)\cdot v_\sigma[q(t)](x)\,dq(t)(x)\,dt.
\]
Hence $E=q\,v_\sigma[q]$, and therefore
\[
\partial_t q + \nabla\cdot\bigl(q\,v_\sigma[q]\bigr)=0
\]
in distributions on $(0,T)\times\Omega$. 

\textbf{Step 4: Energy monotonicity.}
From \eqref{eq:energy_step_app}, for each $\tau$ the map $t\mapsto F_\sigma[q^\tau(t)]$ is nonincreasing.
Fix $0\le s<t\le T$. Then $F_\sigma[q^\tau(t)]\le F_\sigma[q^\tau(s)]$.
Along $\tau_j\to0$, we have $q^{\tau_j}(r)\to q(r)$ in $W_2$, hence by \eqref{eq:uniform_smoothing_conv}
$q^{\tau_j}_\sigma(r)\to q_\sigma(r)$ uniformly on $\Omega$. Under the standing bounds
\eqref{eq:standing_bounds}, the integrand $u\mapsto u\log(u/p_\sigma)$ is continuous and dominated on
$[m_\sigma,M_\sigma]$, so dominated convergence yields
$F_\sigma[q^{\tau_j}(r)]\to F_\sigma[q(r)]$ for $r=s,t$.
Passing to the limit gives $F_\sigma[q(t)]\le F_\sigma[q(s)]$.

\textbf{Step 5: Initial datum.}
By construction $q^\tau(0)=q_0$ and the $1/2$-H\"older bound gives $q(0)=q_0$.
\end{proof}

\subsection{Proof of Consistency: ie, Equation ~\ref{eq:consistency}
(Implicit--Explicit Consistency)}
\label{app:consistency}

\begin{proof}
Let $q^*:=q_{n+1}^\tau$ be the JKO minimizer and let
\[
\tilde q:=(S_n)_\#q_n^\tau,
\qquad
S_n(x)=x+\tau v_\sigma[q_n^\tau](x),
\]
denote the frozen-field (explicit Euler) update.

Let $T_n$ be the optimal transport map pushing $q^*$ to $q_n^\tau$. By
Euler-Langrange Optimality condition~\ref{eq:velocity_identification},
\[
T_n(x)=x+\tau v_\sigma[q^*](x),
\qquad
q_n^\tau=(T_n)_\# q^*.
\]

Define
\[
R(x):=x+\tau v_\sigma[q^*](x).
\]
Then $R=T_n$ and $q_n^\tau=R_\#q^*$.

\medskip

\textbf{Step 1: Reduce to pushforwards of the same base measure.}

Define
\[
\bar q := (S_n)_\# q_n^\tau.
\]
Using $q_n^\tau = R_\# q^*$,
\[
\bar q = (S_n)_\#(R_\# q^*) = (S_n\circ R)_\# q^*.
\]
Thus $\tilde q = \bar q$, and the explicit update corresponds to
transporting $q^*$ with the map $S_n\circ R$.

To compare with $q^*$, it is convenient to rewrite both measures
as pushforwards of $q_n^\tau$.

Since $q_n^\tau=R_\#q^*$, we have
\[
q^* = (R^{-1})_\# q_n^\tau.
\]

Hence
\[
\tilde q = (S_n)_\# q_n^\tau,
\qquad
q^* = (R^{-1})_\# q_n^\tau.
\]

\medskip

\textbf{Step 2: Wasserstein stability of pushforwards.}

For any maps $A,B:\Omega\to\Omega$ and probability measure $\mu$,
the coupling $(A,B)_\#\mu$ yields
\[
W_2(A_\#\mu,B_\#\mu)
\le
\|A-B\|_{L^2(\mu)}.
\]

Applying this with $\mu=q_n^\tau$, $A=S_n$, and $B=R^{-1}$ gives
\[
W_2(\tilde q,q^*)
=
W_2((S_n)_\#q_n^\tau,(R^{-1})_\#q_n^\tau)
\le
\|S_n-R^{-1}\|_{L^2(q_n^\tau)}.
\]

\medskip

\textbf{Step 3: Approximation of the inverse map.}

Since
\[
R(x)=x+\tau v_\sigma[q^*](x),
\]
a first-order Taylor expansion shows that
\[
R^{-1}(x)
=
x-\tau v_\sigma[q^*](x)+O(\tau^2),
\]
uniformly on $\Omega$, provided $v_\sigma[q^*]$ is Lipschitz in $x$.

Therefore
\[
S_n(x)-R^{-1}(x)
=
\tau\bigl(v_\sigma[q_n^\tau](x)-v_\sigma[q^*](x)\bigr)
+
O(\tau^2).
\]

Taking the $L^2(q_n^\tau)$ norm yields
\[
\|S_n-R^{-1}\|_{L^2(q_n^\tau)}
\le
\tau \|v_\sigma[q_n^\tau]-v_\sigma[q^*]\|_\infty
+ O(\tau^2).
\]

\medskip

\textbf{Step 4: Controlling the velocity difference.}

By Lemma~\ref{lem:lipschitz_velocity},
\[
\|v_\sigma[q_n^\tau]-v_\sigma[q^*]\|_\infty
\le
L_\sigma W_2(q_n^\tau,q^*).
\]

From the JKO optimality inequality \eqref{eq:energy_step_app},
\[
W_2(q_n^\tau,q^*)^2
\le
2\tau\bigl(F_\sigma[q_n^\tau]-F_\sigma[q^*]\bigr)
\le
2\tau F_\sigma[q_n^\tau].
\]

Since $F_\sigma[q_n^\tau]$ is bounded along the scheme,
\[
W_2(q_n^\tau,q^*) = O(\sqrt{\tau}).
\]

Therefore
\[
\|v_\sigma[q_n^\tau]-v_\sigma[q^*]\|_\infty
=
O(\sqrt{\tau}).
\]

\medskip

\textbf{Step 5: Final estimate.}

Combining the previous bounds,
\[
W_2(\tilde q,q^*)
\le
\tau\,O(\sqrt{\tau}) + O(\tau^2)
=
O(\tau^{3/2}).
\]

This proves the consistency bound \eqref{eq:consistency}.
\end{proof}

\subsection{Proof of Theorem~\ref{thm:stopgrad}
(Stop-Gradient Preserves Wasserstein Discretization)}
\label{app:stopgrad}

\begin{proof}
\textbf{(i) Structural correspondence.}
Write the stop-gradient loss as
\[
\mathcal{L}(\theta)=\E_{\varepsilon\sim\nu}\bigl[\|G_\theta(\varepsilon)-t(\varepsilon)\|^2\bigr],
\quad
t(\varepsilon):=G_{\theta_n}(\varepsilon)+\tau v_\sigma[q_{\theta_n}]\bigl(G_{\theta_n}(\varepsilon)\bigr),
\]
where $t$ is treated as a fixed target (no gradient flows through it).
If the model class is realizable for $t$ (i.e., there exists $\theta_{n+1}$ with
$G_{\theta_{n+1}}(\varepsilon)=t(\varepsilon)$ $\nu$-a.e.), then any global minimizer achieves
$\mathcal{L}(\theta_{n+1})=0$ and satisfies $G_{\theta_{n+1}}=t$ $\nu$-a.e., i.e.
\[
G_{\theta_{n+1}} = S_n\circ G_{\theta_n}\qquad \nu\text{-a.e.},
\quad S_n(x):=x+\tau v_\sigma[q_{\theta_n}](x).
\]
Pushing forward $\nu$ yields
$q_{\theta_{n+1}}=(S_n)_\# q_{\theta_n}$, which is exactly the frozen-field explicit Euler step.

\textbf{(ii) Necessity (effect of removing stop-gradient).}
Without stop-gradient, the coupled loss becomes
\[
\mathcal{L}_{\mathrm{coupled}}(\theta)
=\E_\varepsilon\Bigl[\bigl\|G_\theta(\varepsilon)-\bigl(G_\theta(\varepsilon)+\tau v_\sigma[q_\theta](G_\theta(\varepsilon))\bigr)\bigr\|^2\Bigr]
=\tau^2\,\E_\varepsilon\bigl[\|v_\sigma[q_\theta](G_\theta(\varepsilon))\|^2\bigr]
=\tau^2\|v_\sigma[q_\theta]\|_{L^2(q_\theta)}^2.
\]
Its gradient includes both the pathwise term and the distribution-feedback term:
\[
\nabla_\theta \mathcal{L}_{\mathrm{coupled}}
=2\tau^2\,\E_\varepsilon\Bigl[
v\cdot \bigl((D_x v)\nabla_\theta G_\theta + (D_q v)\nabla_\theta q_\theta\bigr)
\Bigr],
\quad v:=v_\sigma[q_\theta](G_\theta(\varepsilon)).
\]
Unlike the stop-gradient objective, this is \emph{not} the regression of an explicit Euler target for a fixed velocity field:
as $\theta$ changes, the target velocity itself changes via $q_\theta$ (the $(D_q v)\nabla_\theta q_\theta$ term).
Consequently, minimizing $\mathcal{L}_{\mathrm{coupled}}$ does not implement the frozen-field discretization of the
Wasserstein gradient flow; it can admit spurious stationary points/minima where the loss decreases by reducing the
velocity norm on the current support of $q_\theta$ (``drift collapse'') without transporting mass toward $p$.
\end{proof}

\section{Schedule Ablation}
\label{app:schedule_ablation}

We compare exponential ($\sigma(t) = \sigma_0 e^{-rt}$), linear ($\sigma(t) = \sigma_0(1 - t/T)$), and cosine ($\sigma(t) = \sigma_0 \cos(\pi t / 2T)$) schedules, all sweeping from $\sigma_0 = 1.5$ to $\sigma_{\min} = 0.03$. The exponential schedule converges fastest for all $k$.  This is consistent with the activation-time analysis: the exponential schedule reaches $\sigma^2|k|^2 = 2$ earliest for each $k$.

\section{Toy Experiments}

\label{app:experiments}

Our experiments validate the theoretical predictions of
\S\ref{sec:drifting}--\S\ref{sec:gradient_flow} on synthetic benchmarks.
All models use 3-layer MLPs ($d_{\mathrm{hidden}} = 256$, ReLU, Adam, lr $10^{-3}$,
batch size 2048). 

\subsection{Score-matching verification (\ref{fig:score_matching_verification}).}
\label{app:score_check}
Let $p$ be a 4-mode Gaussian mixture and $q = \mathcal{N}(0, \sigma_q^2 I)$.
For each bandwidth $\sigma$, we evaluate the empirical kernel mean-shift drift
$\hat{V}^{(\sigma)}_{p,q}(x)
  = \frac{\sum_i \varphi_\sigma(x - x_i^p)\,(x_i^p - x)}
         {\sum_i \varphi_\sigma(x - x_i^p)}
  - \frac{\sum_j \varphi_\sigma(x - x_j^q)\,(x_j^q - x)}
         {\sum_j \varphi_\sigma(x - x_j^q)}$
at $N = 50$k samples and compare it against the analytical form
$V^{(\sigma)}(x) = \sigma^2 \bigl(\nabla \log p_\sigma(x) - \nabla \log q_\sigma(x)\bigr)$,
where the smoothed densities $p_\sigma = p * \mathcal{N}(0, \sigma^2 I)$ and
$q_\sigma = q * \mathcal{N}(0, \sigma^2 I)$ admit closed-form scores.
Writing $s^2 = \sigma_p^2 + \sigma^2$ for the inflated component variance and
$\phi_k(x) = \mathcal{N}(x;\,\mu_k,\, s^2 I)$:
\begin{equation}
\label{eq:gmm_score}
  \nabla\!\log p_\sigma(x)
  \;=\;
  \frac{\sum_k w_k\,\phi_k(x)\,(\mu_k - x)}
       {s^2 \sum_k w_k\,\phi_k(x)},
  \qquad
  \nabla\!\log q_\sigma(x)
  \;=\;
  -\,\frac{x - \mu_q}{\sigma_q^2 + \sigma^2}\,.
\end{equation}
The pointwise $\ell_2$ error
$e(x) = \|\hat{V}^{(\sigma)}_{p,q}(x) - V^{(\sigma)}(x)\|_2$ has mean
$4.9 \times 10^{-3}$ at $\sigma = 0.3$ confirming Equation
\ref{eq:drift_score}.

\subsection{Spectral convergence times: Figure (\ref{fig:spectral_validation}).}
\label{app:spectral_check}
We initialize a small perturbation
$\delta q(x, 0) = A \sum_{k \in \mathcal{K}} \cos(kx)$
with amplitude $A = 10^{-6}$ and tracked modes
$\mathcal{K} = \{1, 2, \ldots, 20\}$,
then evolve the linearized Fourier dynamics
$\hat{\delta q}(k, t{+}\mathrm{d}t) = \hat{\delta q}(k, t)\,
  \exp\!\bigl(-\lambda(k, \sigma(t))\,\mathrm{d}t\bigr)$
with per-mode decay rates
$\lambda_{\mathrm{G}}(k, \sigma) = \sigma^2 k^2 e^{-\sigma^2 k^2 / 2}$
(Gaussian kernel) and
$\lambda_{\mathrm{E}}(k, \tau) = 2\tau^3 k^2 / (1 + \tau^2 k^2)$
(exponential kernel),
and define the convergence time as the first $T$ at which
$|\hat{\delta q}(k, T)| < \varepsilon \,|\hat{\delta q}(k, 0)|$
with $\varepsilon = 10^{-3}$.
For fixed kernels, $T(k) = \log(1/\varepsilon)\,/\,\lambda(k)$;
for the annealed Gaussian schedule $\sigma(t) = \sigma_0 e^{-rt}$,
$T(k)$ solves the integral equation
$\int_0^{T} \lambda_{\mathrm{G}}(k, \sigma(t))\,\mathrm{d}t = \log(1/\varepsilon)$.
Panel~(a) shows that numerical threshold-crossing times (markers) match the analytical
curves (lines) across modes $k = 1, \ldots, 20$:
the fixed Gaussian ($\sigma = 0.3$) exhibits exponential slowdown past
$k^* = \sqrt{2}/\sigma \approx 4.7$,
the exponential kernel ($\tau = 0.3$) yields polynomial scaling,
and the annealed schedule eliminates the bottleneck entirely.
Panel~(c) plots the total spectral error
$E(t) = \sum_k |\hat{\delta q}(k, t)|$ under three annealing schedules
(exponential, linear, cosine);
the exponential schedule achieves the fastest decay, with $k{=}20$ converging
${\sim}6\times$ faster than linear.

\subsection{Stop-gradient necessity: Figures \ref{fig:loss_landscape} and \ref{fig:drift_vs_sw} }
\label{app:stop_grad_check}
We train generators $G_\theta$ on 2D targets (Swiss roll, checkerboard) and track two
quantities over training: the mean drift norm
$\bar{V} = \frac{1}{M} \sum_{i=1}^{M}
  \|V_{\mathrm{sg}}(G_\theta(z_i),\, p)\|_2$
(where $V_{\mathrm{sg}}$ denotes the exponential-kernel drift with $\tau{=}0.05$
evaluated at $M{=}2048$ generated points with the target detached),
and the sliced Wasserstein distance
$\mathrm{SW}(G_\theta, p)$ between 5k generated and 5k target samples (200 projections).
With stop-gradient, $\bar{V}$ and SW are strongly correlated (log-log $r > 0.95$,
3 seeds) and jointly decay to near zero (final SW $= 0.016$).
Without stop-gradient,where the loss $\|G_\theta(z) - (G_\theta(z) + V(G_\theta(z), p))\|^2$
backpropagates through the drift,$\bar{V}$ drops to ${\sim}10^{-8}$
while SW remains at $0.389$, demonstrating drift collapse
(\ref{thm:stopgrad}\ref{item:necessity}).
The loss landscape (\ref{fig:loss_landscape}) provides geometric intuition:
we project the parameter space onto the top-2 PCA directions of training gradients,
$\theta(\alpha,\beta) = \theta^* + \alpha\,d_1 + \beta\,d_2$,
and evaluate both the training loss $\|V\|^2$ and the sample quality $\mathrm{SW}$
on a $31 \times 31$ grid.
With stop-gradient, the loss minimum aligns with a basin of low SW;
without it, the coupled objective admits a ${\sim}100\times$ deeper minimum that
corresponds to poor sample quality.

\subsection{Sinkhorn drift feasibility: Figure (\ref{fig:sinkhorn_vs_kernel}).}
\label{app:sinkhorn_check}

\paragraph{Energy functional and validity conditions.}
We use the Sinkhorn divergence of \citet{feydy2018interpolatingoptimaltransportmmd}
as energy functional $F[q] := S_\varepsilon(q, p)$, defined in
\eqref{eq:sinkhorn_energy}.
By Theorem~1 of that reference,
$S_\varepsilon$ is symmetric positive definite, convex in each argument,
and metrizes the convergence in law; in particular $S_\varepsilon(q,p)=0
\Leftrightarrow q=p$. Lower semicontinuity and smoothness of
$\delta F/\delta q$ follow from the same reference (Proposition~2).
All three conditions (a)-(c) of \S\ref{sec:principled_drift} are
therefore satisfied, and the template $V = -\nabla_x(\delta F/\delta q)$
yields a valid drift with full JKO-frozen-field-stop-gradient guarantees.

\paragraph{Particle-level gradient.}
For empirical $q = \frac{1}{N}\sum_i \delta_{x_i}$, the gradient of
$S_\varepsilon(q,p)$ with respect to particle $x_i$ follows from
Proposition~2 of \citet{feydy2018interpolatingoptimaltransportmmd}
(equations 26-27):
\begin{equation}
\label{eq:sinkhorn_particle_grad}
\nabla_{x_i} S_\varepsilon(q,p)
\;=\; \nabla_{x_i}\,\mathrm{OT}_\varepsilon(q,p)
\;-\; \nabla_{x_i}\,\mathrm{OT}_\varepsilon(q,q),
\end{equation}
Concretely, denoting by $\pi_{qp}$ and
$\pi_{qq}$ the optimal entropic couplings, the barycentric projections
\[
T_{q\to p}(x_i) = \frac{\sum_j [\pi_{qp}]_{ij}\,y_j}{\sum_j [\pi_{qp}]_{ij}},
\qquad
T_{q\to q}(x_i) = \frac{\sum_j [\pi_{qq}]_{ij}\,x_j}{\sum_j [\pi_{qq}]_{ij}},
\]
give the drift
\begin{equation}
\label{eq:sinkhorn_drift_particles}
V_{\mathrm{SK}}(x_i)
\;\propto\;
T_{q\to p}(x_i) - T_{q\to q}(x_i),
\end{equation}

We plug this drift-operator in the stop-gradient training $\mathcal{L}(\theta) = \mathbb{E}_{\epsilon}[||f_{\theta}(\epsilon) - sg[\tilde{x}]||^2]$ with $\tilde x=x+V_{\mathrm{SK}}(x)$.  

\paragraph{Training}
We train on the checkerboard distribution using the exponential-kernel drift
($\tau{=}0.05$) and the Sinkhorn-derived drift from the debiased entropic cost
$S_\varepsilon(q,p)$ with $\varepsilon{=}10^{-4}$.
Both drifts successfully transport the generator to the target, achieving final
SW of $2.07 \times 10^{-2}$ and $1.42 \times 10^{-2}$ respectively.
The Sinkhorn drift converges to the prescribed distribution, 
confirming that the gradient-flow formulation of \S\ref{sec:principled_drift}
yields practical drift operators beyond the original kernel family.

\section{Image Generation Experimental Details}
\label{app:image_exp}

\subsection{Dataset, representation space, and architecture}
\label{app:image_exp:setup}

We train class-conditional generators on ImageNet-$256$ ($1000$ classes) in the
latent space of the Stable Diffusion VAE, which maps
a $256{\times}256{\times}3$ image to a $32{\times}32{\times}4$ latent. FID is computed after VAE-decoding back to pixels.

The generator is \texttt{DitGen-B}, a one-step DiT backbone (hidden size
$768$, depth $12$, $12$ heads, patch size $2$) with RoPE, RMSNorm, QK-norm,
SwiGLU MLPs, and AdaLN class+CFG conditioning. 

The feature extractor is a frozen MAE-pretrained ResNet-18,
from which we extract multi-scale activations at \texttt{conv1}, the four
stages \texttt{layer1}--\texttt{layer4}, and every second residual block
within each stage. The drift loss is computed independently on each stream and summed.

\subsection{Training protocol}

All training details are reported in Table \ref{tab:image_exp:hparams}
\begin{table}[h!]
\centering
\small
\caption{Training hyperparameters for the image-generation ablations
(\texttt{latent\_ablation} configuration).}
\label{tab:image_exp:hparams}
\begin{tabular}{ll}
\toprule
\textbf{Optimization} & \\
\midrule
Optimizer & AdamW ($\beta_1{=}0.9$, $\beta_2{=}0.95$, $\epsilon{=}10^{-8}$) \\
Weight decay & $0.01$ \\
Peak learning rate & $2 \times 10^{-4}$ \\
LR schedule & linear warmup ($5{,}000$ steps), then constant \\
Gradient clipping & $\|g\|_2 \le 2.0$ \\
Total steps & $30{,}000$ \\
\midrule
\textbf{Batching} & \\
\midrule
Labels per step & $1024$ \\
Generated samples per label & $64$ \\
GPUs & $16$ H100 \\
\midrule
\textbf{Memory bank} & \\
\midrule
Positive bank size (per class) & $128$ \\
Negative bank size (shared) & $1000$ \\
Pushes per step & $64$ \\
Positives / negatives drawn per label & $64$ / $16$ \\
\midrule
\textbf{CFG \& EMA} & \\
\midrule
CFG sampling range & $[1.0, 4.0]$ \\
CFG power-law exponent & $3$ \\
Unconditional-only fraction & $0$ \\
EMA decay & $0.999$ \\
\bottomrule
\end{tabular}
\end{table}
\subsection{Drift-operator implementations}

\subsubsection{Gaussian kernel drift}
\label{sec:gaussian-drift-impl}

The Gaussian-kernel drift is a drop-in replacement for the Laplacian baseline. Following the convention of \citet{deng2026drifting} we compute the kernel as a doubly-stochastic softmax (rather than the raw $\exp(-\|x-y\|^2/2\sigma^2)$ form), with quadratic logits in place of the linear Laplacian logits. Concretely, with bandwidth $\sigma$,
\begin{equation}
\mathrm{logits}(x, y) = -\frac{\|x - y\|^2}{2 \sigma^2}, \qquad \tilde k(x, y) = \sqrt{\mathrm{softmax}_y(\mathrm{logits}) \cdot \mathrm{softmax}_x(\mathrm{logits})}.
\end{equation}
The doubly-stochastic normalization preserves the antisymmetry $V_{p,q} = -V_{q,p}$ and is required for the empirical-mean Monte-Carlo interpretation of \citet{deng2026drifting} (Eq.~17 there). All other steps: masking self-interactions, splitting positive/negative streams, computing the attractive and repulsive coefficients---are unchanged from the Laplacian baseline; only the logit transformation differs. Algorithm~\ref{alg:gaussian-drift} summarizes the procedure with the bandwidth $\sigma(t)$ written generically: setting $\sigma(t) \equiv \sigma$ recovers the fixed-bandwidth Gaussian drift, while the exponential schedule $\sigma(t) = \max(\sigma_0 e^{-rt}, \sigma_{\min})$ implements the annealed variant of §6.2.

\begin{algorithm}[h!]
\caption{Gaussian-kernel drift step (with optional bandwidth schedule)}
\label{alg:gaussian-drift}
\begin{algorithmic}[1]
\REQUIRE generated batch $X = \{x_i\}_{i=1}^{N}$, data batch $Y = \{y_j\}_{j=1}^{N_{\mathrm{pos}}}$, negatives $Z = \{z_k\}_{k=1}^{N_{\mathrm{neg}}}$ (typically $Z = X$), training step $t$, bandwidth schedule $\sigma(\cdot)$ (constant $\sigma$ or $\sigma(t) = \max(\sigma_0 e^{-rt}, \sigma_{\min})$)
\STATE \textcolor{gray}{\textit{// 1. Bandwidth selection (fixed or annealed)}}
\STATE $\sigma \leftarrow \sigma(t)$
\STATE \textcolor{gray}{\textit{// 2. Pairwise distances and self-masking}}
\STATE $d^{\mathrm{pos}}_{ij} \leftarrow \|x_i - y_j\|$, \quad $d^{\mathrm{neg}}_{ik} \leftarrow \|x_i - z_k\|$
\STATE $d^{\mathrm{neg}}_{ii} \leftarrow +\infty$ \quad (mask self-interactions when $Z = X$)
\STATE \textcolor{gray}{\textit{// 3. Quadratic Gaussian logits}}
\STATE $\ell^{\mathrm{pos}}_{ij} \leftarrow -\,(d^{\mathrm{pos}}_{ij})^2 / (2\sigma^2)$, \quad $\ell^{\mathrm{neg}}_{ik} \leftarrow -\,(d^{\mathrm{neg}}_{ik})^2 / (2\sigma^2)$
\STATE \textcolor{gray}{\textit{// 4. Doubly-stochastic softmax (preserves antisymmetry)}}
\STATE $\ell \leftarrow [\ell^{\mathrm{pos}} \,\|\, \ell^{\mathrm{neg}}]$ \hfill {\small (concatenate along the column axis)}
\STATE $A^{\mathrm{row}} \leftarrow \mathrm{softmax}_{\mathrm{row}}(\ell)$, \quad $A^{\mathrm{col}} \leftarrow \mathrm{softmax}_{\mathrm{col}}(\ell)$
\STATE $A \leftarrow \sqrt{A^{\mathrm{row}} \odot A^{\mathrm{col}}}$ \hfill {\small (entrywise; doubly-stochastic kernel)}
\STATE $A^{\mathrm{pos}}, A^{\mathrm{neg}} \leftarrow$ split $A$ along columns into $[N_{\mathrm{pos}}, N_{\mathrm{neg}}]$
\STATE \textcolor{gray}{\textit{// 5. Drift via attractive / repulsive coefficients}}
\STATE $W^{\mathrm{pos}}_{ij} \leftarrow A^{\mathrm{pos}}_{ij} \cdot \sum_k A^{\mathrm{neg}}_{ik}$
\STATE $W^{\mathrm{neg}}_{ik} \leftarrow A^{\mathrm{neg}}_{ik} \cdot \sum_j A^{\mathrm{pos}}_{ij}$
\STATE $V^+(x_i) \leftarrow \sum_j W^{\mathrm{pos}}_{ij}\, y_j$, \quad $V^-(x_i) \leftarrow \sum_k W^{\mathrm{neg}}_{ik}\, z_k$
\STATE $V(x_i) \leftarrow V^+(x_i) - V^-(x_i)$
\STATE \textcolor{gray}{\textit{// 6. Stop-gradient drift loss}}
\STATE $\mathcal{L} \leftarrow \tfrac{1}{N}\sum_i \bigl\|x_i - \mathrm{sg}[x_i + V(x_i)]\bigr\|^2$
\RETURN $\mathcal{L}$
\end{algorithmic}
\end{algorithm}

\subsection{Sinkhorn-Drift operator implementations}
\label{sec:sinkhorn-implementation}

We now describe the practical implementation of the Sinkhorn-divergence drift introduced in §6.3. The key ingredients are: entropic regularization of the optimal transport problem, a numerically stable log-space Sinkhorn solver, and a diagonal-masking heuristic that is necessary in the regime $q \approx p$ encountered late in training.

\paragraph{Entropic optimal transport.}
Given two empirical distributions $\hat q = \frac{1}{N}\sum_{i=1}^{N}\delta_{x_i}$ and $\hat p = \frac{1}{M}\sum_{j=1}^{M}\delta_{y_j}$ on $\mathbb{R}^d$, with cost matrix $C_{ij} = \|x_i - y_j\|^r$ ($r \in \{1, 2\}$), the entropic optimal transport problem reads
\begin{equation}
\mathrm{OT}_\varepsilon(\hat q, \hat p) \;=\; \min_{\pi \in \Pi(\hat q, \hat p)} \;\sum_{i,j} \pi_{ij} C_{ij} \;+\; \varepsilon \sum_{i,j} \pi_{ij} \bigl(\log \pi_{ij} - 1\bigr),
\label{eq:eot-primal}
\end{equation}
where $\Pi(\hat q, \hat p)$ is the set of couplings with marginals $\frac{1}{N}\mathbf{1}_N$ and $\frac{1}{M}\mathbf{1}_M$. The unique minimizer admits the Gibbs form $\pi_{ij}^\star = \exp\bigl((f_i + g_j - C_{ij})/\varepsilon\bigr)$, where the dual potentials $(f, g)$ are obtained by alternately enforcing the two marginal constraints, the Sinkhorn iterations \citep{peyre2020computationaloptimaltransport}. To stabilize entropic regularization across data scales we set $\varepsilon = \alpha\,\mathbb{E}\|x - y\|^r$ with $\alpha$ a dimensionless temperature; this keeps the kernel exponent $-C_{ij}/\varepsilon$ at unit mean across batches and across cost powers.

\paragraph{Log-space Sinkhorn iterations.}
At small $\varepsilon$, the kernel $K_{ij} = \exp(-C_{ij}/\varepsilon)$ underflows to zero, and the standard matrix-scaling form of Sinkhorn becomes numerically unusable. We instead iterate on the dual potentials $(f, g)$ directly, using the log-sum-exp trick \citep{peyre2020computationaloptimaltransport, feydy2018interpolatingoptimaltransportmmd}. With uniform marginals $\log a_i = -\log N$ and $\log b_j = -\log M$, the updates read
\begin{align}
g_j &\leftarrow \log b_j - \mathrm{logsumexp}_i\bigl(f_i - C_{ij}/\varepsilon\bigr), \label{eq:sinkhorn-g}\\
f_i &\leftarrow \log a_i - \mathrm{logsumexp}_j\bigl(g_j - C_{ij}/\varepsilon\bigr). \label{eq:sinkhorn-f}
\end{align}
We iterate \eqref{eq:sinkhorn-g}--\eqref{eq:sinkhorn-f} for at most $T$ steps, with adaptive stopping when $\|f^{(t+1)} - f^{(t)}\|_\infty < \mathrm{tol}\cdot\varepsilon$. The coupling is reconstructed as $\log \pi_{ij} = f_i + g_j - C_{ij}/\varepsilon$ and exponentiated in a single pass.

\paragraph{Diagonal masking when $q \approx p$.}
The Sinkhorn drift is the difference of two barycentric maps, $V(x_i) \propto T_{q\to p}(x_i) - T_{q\to q}(x_i)$, and its validity rests on a non-trivial self-coupling $\pi_{qq}$. As training progresses and $q \to p$, however, the cross cost $C_{qp}$ and the self cost $C_{qq}$ become comparable in magnitude, but the latter has a free zero on its diagonal: at any finite $\varepsilon$ the entropic minimizer puts increasing mass on $\pi_{qq}^{ii}$, and in the limit $\pi_{qq} \to \frac{1}{N}\mathrm{diag}(\mathbf{1}_N)$, so $T_{q\to q}(x_i) \to x_i$ and the drift collapses to the \emph{biased} form $V \approx T_{q\to p}(x) - x$, which is no longer a Wasserstein gradient of the Sinkhorn divergence and reintroduces the bias that the debiasing in \eqref{eq:sinkhorn_energy} was designed to remove. We employ a simple and effective remedy: mask the diagonal of $C_{qq}$ by setting $C_{qq}^{ii} = +\infty$ before the Sinkhorn solve, which forces $\pi_{qq}^{ii} = 0$ and turns $T_{q\to q}(x_i)$ into a leave-one-out entropic centroid of the remaining particles. The drift then retains genuine self-interaction signal at all $\varepsilon$. We note that this heuristic introduces a small shift in the fixed point of the flow (since $\pi_{qp}$ is unmasked while $\pi_{qq}$ is masked, $V$ no longer vanishes \emph{exactly} at $q = p$); empirically we find this trade-off to be strongly favourable, and study it further in §H.4.

\paragraph{Drift operator and loss.}
With the entropic couplings $\pi_{qp}, \pi_{qq}$ in hand, the barycentric maps are computed by a single batched matrix product,
\begin{equation}
T_{q \to p}(x_i) \;=\; N\sum_{j} [\pi_{qp}]_{ij}\, y_j, \qquad T_{q \to q}(x_i) \;=\; N\sum_{j} [\pi_{qq}]_{ij}\, x_j,
\end{equation}
(the prefactor $N$ corrects for the uniform source marginal $1/N$), and the Sinkhorn drift is
\begin{equation}
V_{\mathrm{SK}}(x_i) \;=\; \eta \, \bigl(T_{q\to p}(x_i) - T_{q\to q}(x_i)\bigr),
\label{eq:sinkhorn-drift-final}
\end{equation}
where $\eta$ is a step-size hyperparameter (analogous to $\tau$ in the JKO derivation). The loss is the standard stop-gradient drift loss of Eq.~\eqref{eq:loss},
\begin{equation}
\mathcal{L}_{\mathrm{SK}}(\theta) \;=\; \mathbb{E}_{\epsilon}\bigl[\,\bigl\|f_\theta(\epsilon) - \mathrm{sg}\bigl[f_\theta(\epsilon) + V_{\mathrm{SK}}(f_\theta(\epsilon))\bigr]\bigr\|^2\,\bigr].
\end{equation}
The full procedure for a single training step is summarized in Algorithm~\ref{alg:sinkhorn-drift}.

\begin{algorithm}[t]
\caption{Sinkhorn-divergence drift step}
\label{alg:sinkhorn-drift}
\begin{algorithmic}[1]
\REQUIRE generated batch $X = \{x_i\}_{i=1}^{N}$, data batch $Y = \{y_j\}_{j=1}^{M}$, temperature $\alpha$, max iterations $T$, tolerance $\mathrm{tol}$, cost power $r \in \{1,2\}$, drift scale $\eta$
\STATE \textcolor{gray}{\textit{// 1. Cost matrices}}
\STATE $C_{qp} \leftarrow \|x_i - y_j\|^r$,\quad $C_{qq} \leftarrow \tfrac{1}{2}\bigl(\|x_i - x_j\|^r + \|x_j - x_i\|^r\bigr)$ \hfill {\small (symmetrize for fp robustness)}
\STATE $\varepsilon \leftarrow \alpha \cdot \mathrm{mean}(C_{qp})$
\STATE \textcolor{gray}{\textit{// 2. Diagonal masking on the self-cost}}
\STATE $C_{qq}^{ii} \leftarrow +\infty$ \quad for $i = 1, \dots, N$
\STATE \textcolor{gray}{\textit{// 3. Log-space Sinkhorn for both couplings}}
\STATE $\pi_{qp} \leftarrow \mathrm{LogSinkhorn}(C_{qp}, \varepsilon, T, \mathrm{tol})$ \hfill {\small (Eqs.~\eqref{eq:sinkhorn-g}--\eqref{eq:sinkhorn-f})}
\STATE $\pi_{qq} \leftarrow \mathrm{LogSinkhorn}(C_{qq}, \varepsilon, T, \mathrm{tol})$
\STATE \textcolor{gray}{\textit{// 4. Drift via barycentric maps}}
\STATE $T_{q \to p}(x_i) \leftarrow N\sum_j [\pi_{qp}]_{ij}\, y_j$
\STATE $T_{q \to q}(x_i) \leftarrow N\sum_j [\pi_{qq}]_{ij}\, x_j$
\STATE $V_{\mathrm{SK}}(x_i) \leftarrow \eta\,\bigl(T_{q \to p}(x_i) - T_{q \to q}(x_i)\bigr)$
\STATE \textcolor{gray}{\textit{// 5. Stop-gradient drift loss}}
\STATE $\mathcal{L} \leftarrow \tfrac{1}{N}\sum_i \bigl\|x_i - \mathrm{sg}[x_i + V_{\mathrm{SK}}(x_i)]\bigr\|^2$
\RETURN $\mathcal{L}$
\end{algorithmic}
\end{algorithm}

All operations are batched over the leading dimension and run end-to-end on GPU; the per-step overhead relative to the kernel baseline is dominated by the $T$ Sinkhorn iterations on two $N \times M$ and $N \times N$ cost matrices, as quantified in Table~\ref{tab:drift-compute}.

\subsection{Qualitative samples}
\label{app:image_exp:samples}

Figures~\ref{fig:baseline_panel}, \ref{fig:gaussian_panel}, \ref{fig:annealed_panel}, and \ref{fig:sinkhorn_panel} show $48$ uncurated $256{\times}256$ samples from each of our four models, exponential baseline \citep{deng2026drifting}, Gaussian-drift, Gaussian-drift with annealing, and Sinkhorn-drift after $30{,}000$ training steps. Each grid mixes samples from $10$ ImageNet classes (macaw, hummingbird, flamingo, golden retriever, tiger, monarch butterfly, African elephant, sports car, cheeseburger, volcano).

\begin{figure}[p]
  \centering
  \def\cs{2.3cm}
  \newcommand{\sscell}[1]{\includegraphics[width=\cs,height=\cs]{figures/baseline/#1.png}}
  \begin{tikzpicture}[
    every node/.style={inner sep=0pt, outer sep=0pt, anchor=north west},
    x=\cs, y=-\cs
  ]
    \node at (0,0) {\sscell{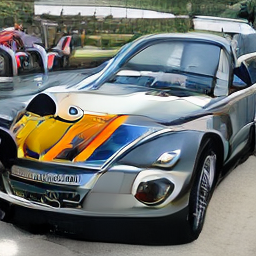}};
    \node at (1,0) {\sscell{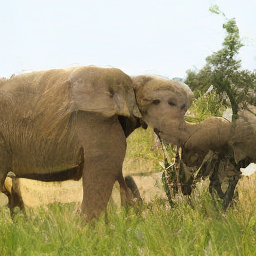}};
    \node at (2,0) {\sscell{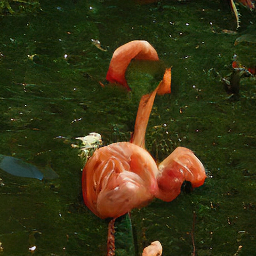}};
    \node at (3,0) {\sscell{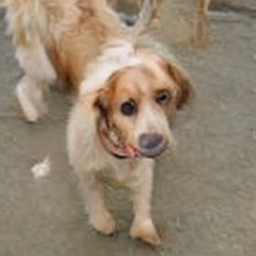}};
    \node at (4,0) {\sscell{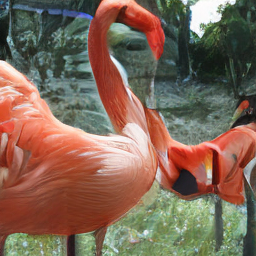}};
    \node at (5,0) {\sscell{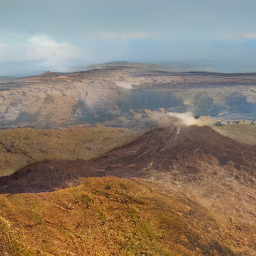}};
    \node at (0,1) {\sscell{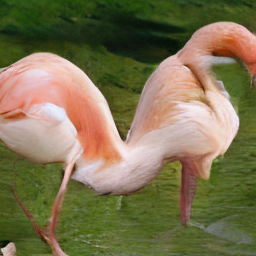}};
    \node at (1,1) {\sscell{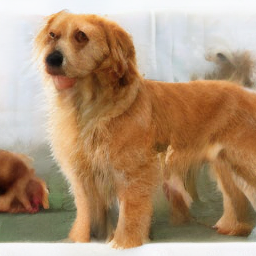}};
    \node at (2,1) {\sscell{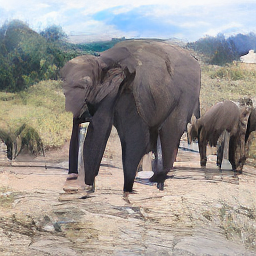}};
    \node at (3,1) {\sscell{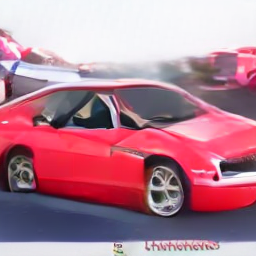}};
    \node at (4,1) {\sscell{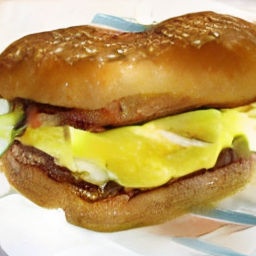}};
    \node at (5,1) {\sscell{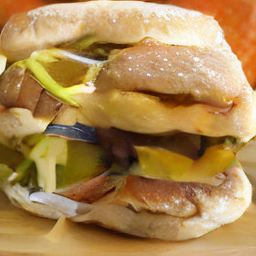}};
    \node at (0,2) {\sscell{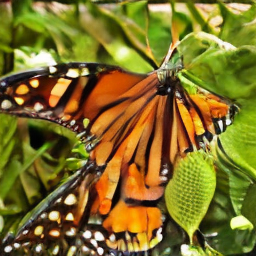}};
    \node at (1,2) {\sscell{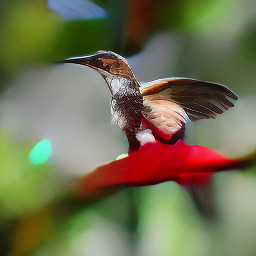}};
    \node at (2,2) {\sscell{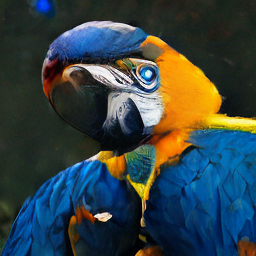}};
    \node at (3,2) {\sscell{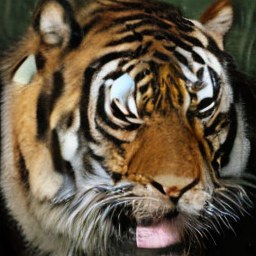}};
    \node at (4,2) {\sscell{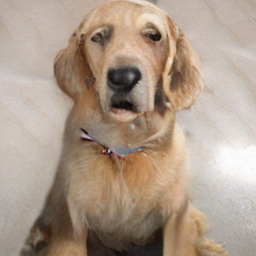}};
    \node at (5,2) {\sscell{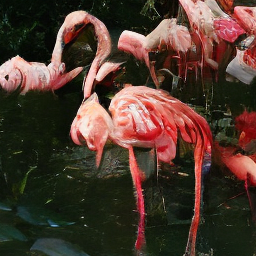}};
    \node at (0,3) {\sscell{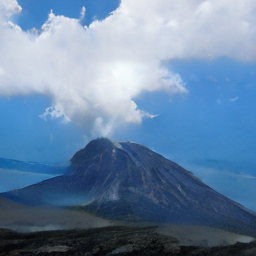}};
    \node at (1,3) {\sscell{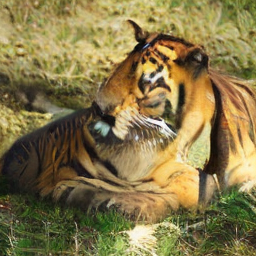}};
    \node at (2,3) {\sscell{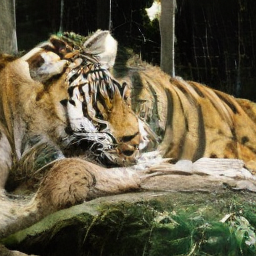}};
    \node at (3,3) {\sscell{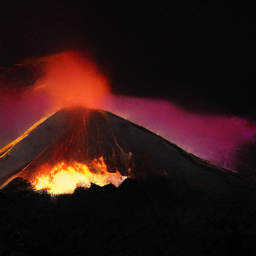}};
    \node at (4,3) {\sscell{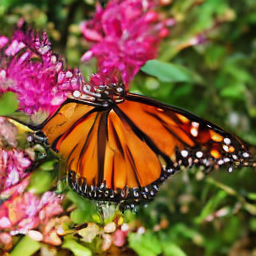}};
    \node at (5,3) {\sscell{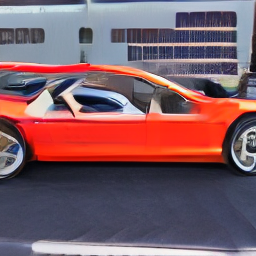}};
    \node at (0,4) {\sscell{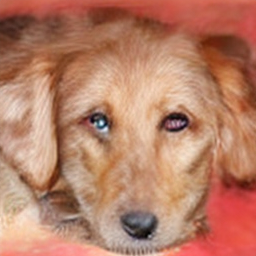}};
    \node at (1,4) {\sscell{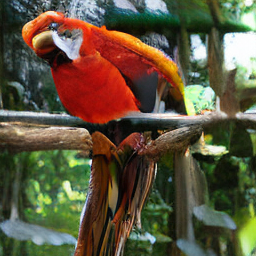}};
    \node at (2,4) {\sscell{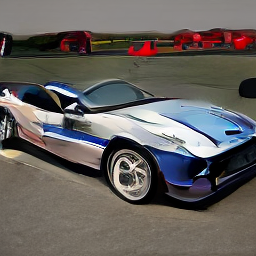}};
    \node at (3,4) {\sscell{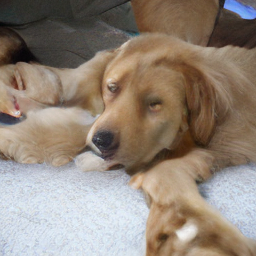}};
    \node at (4,4) {\sscell{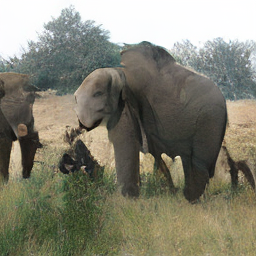}};
    \node at (5,4) {\sscell{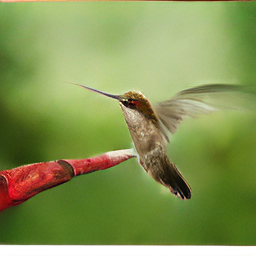}};
    \node at (0,5) {\sscell{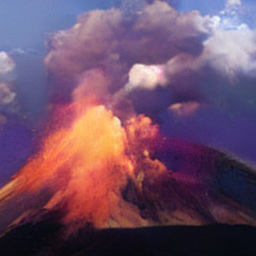}};
    \node at (1,5) {\sscell{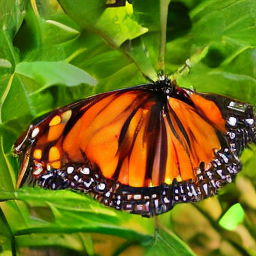}};
    \node at (2,5) {\sscell{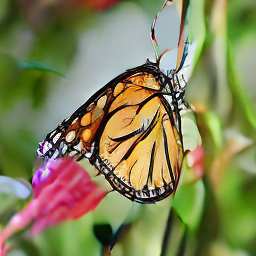}};
    \node at (3,5) {\sscell{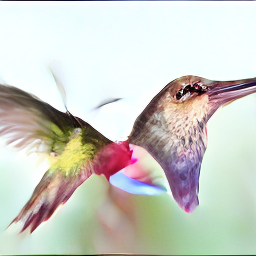}};
    \node at (4,5) {\sscell{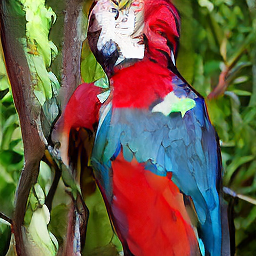}};
    \node at (5,5) {\sscell{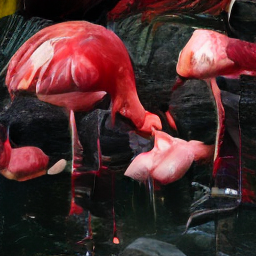}};
    \node at (0,6) {\sscell{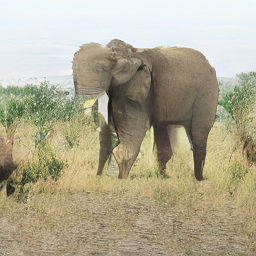}};
    \node at (1,6) {\sscell{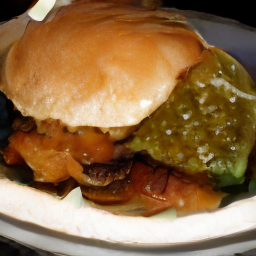}};
    \node at (2,6) {\sscell{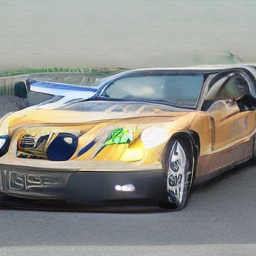}};
    \node at (3,6) {\sscell{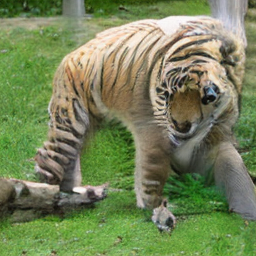}};
    \node at (4,6) {\sscell{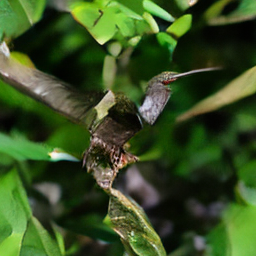}};
    \node at (5,6) {\sscell{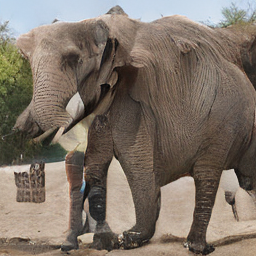}};
    \node at (0,7) {\sscell{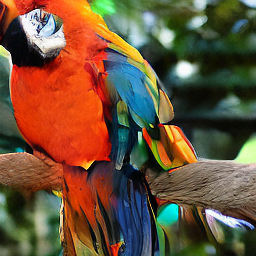}};
    \node at (1,7) {\sscell{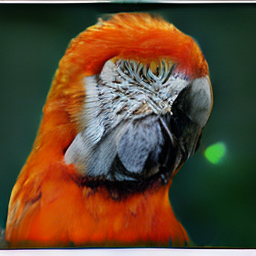}};
    \node at (2,7) {\sscell{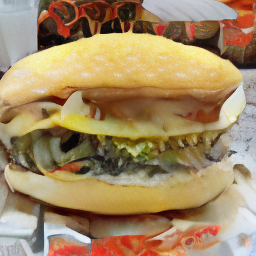}};
    \node at (3,7) {\sscell{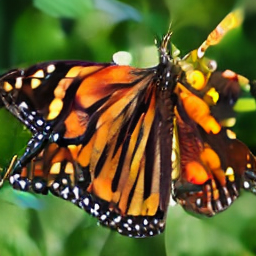}};
    \node at (4,7) {\sscell{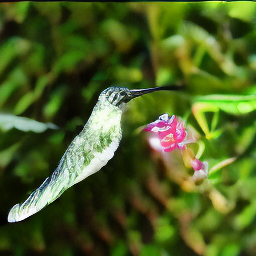}};
    \node at (5,7) {\sscell{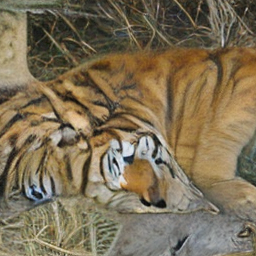}};
  \end{tikzpicture}
  \caption{Uncurated samples from the exponential baseline \citep{deng2026drifting} at $30{,}000$ steps, CFG $w{=}2.2$. Each cell is an independently generated $256{\times}256$ image; samples from $10$ ImageNet classes (macaw, hummingbird, flamingo, golden retriever, tiger, monarch butterfly, African elephant, sports car, cheeseburger, volcano) are shuffled across the grid.}
  \label{fig:baseline_panel}
\end{figure}

\begin{figure}[p]
  \centering
  \def\cs{2.3cm}
  \newcommand{\sscell}[1]{\includegraphics[width=\cs,height=\cs]{figures/gaussian/#1.png}}
  \begin{tikzpicture}[
    every node/.style={inner sep=0pt, outer sep=0pt, anchor=north west},
    x=\cs, y=-\cs
  ]
    \node at (0,0) {\sscell{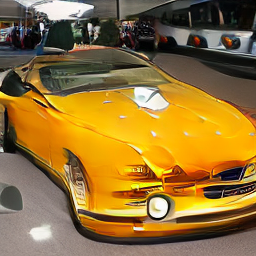}};
    \node at (1,0) {\sscell{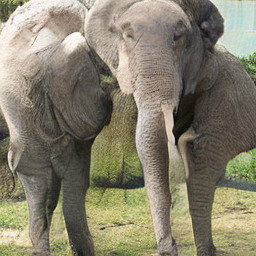}};
    \node at (2,0) {\sscell{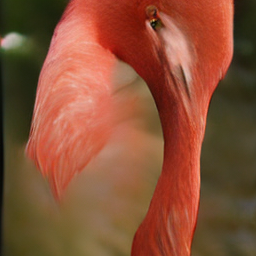}};
    \node at (3,0) {\sscell{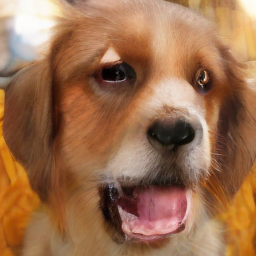}};
    \node at (4,0) {\sscell{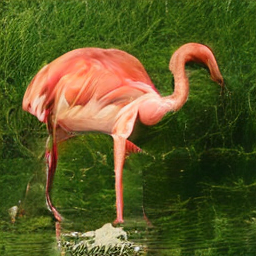}};
    \node at (5,0) {\sscell{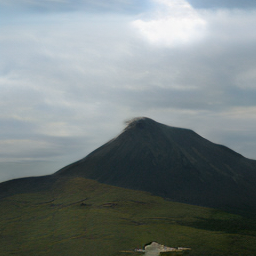}};
    \node at (0,1) {\sscell{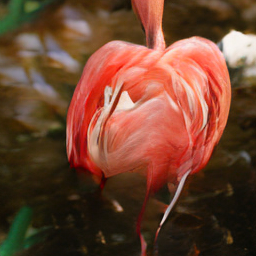}};
    \node at (1,1) {\sscell{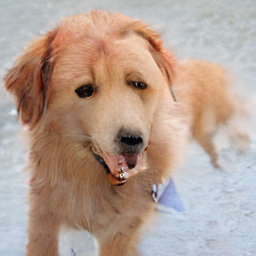}};
    \node at (2,1) {\sscell{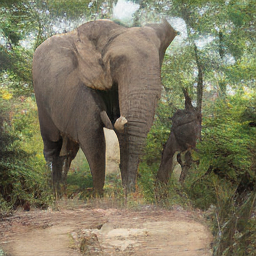}};
    \node at (3,1) {\sscell{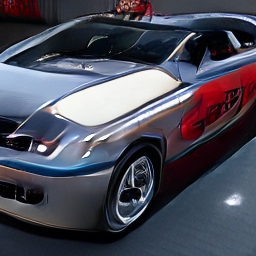}};
    \node at (4,1) {\sscell{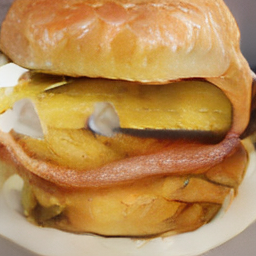}};
    \node at (5,1) {\sscell{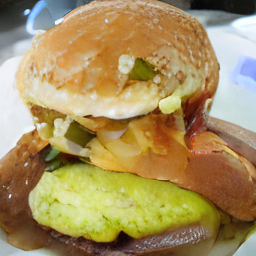}};
    \node at (0,2) {\sscell{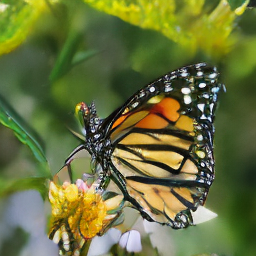}};
    \node at (1,2) {\sscell{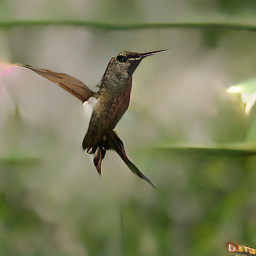}};
    \node at (2,2) {\sscell{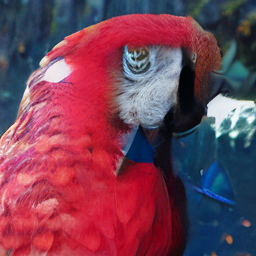}};
    \node at (3,2) {\sscell{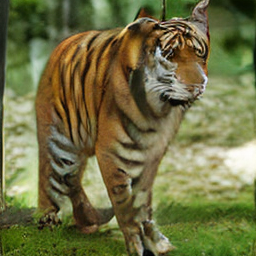}};
    \node at (4,2) {\sscell{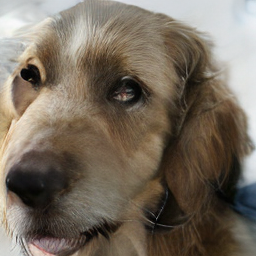}};
    \node at (5,2) {\sscell{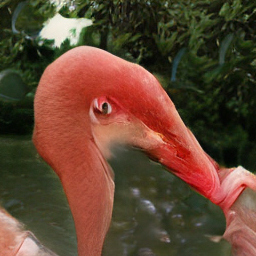}};
    \node at (0,3) {\sscell{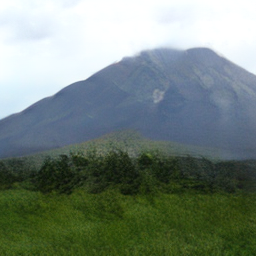}};
    \node at (1,3) {\sscell{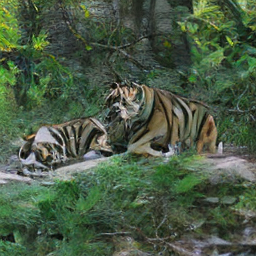}};
    \node at (2,3) {\sscell{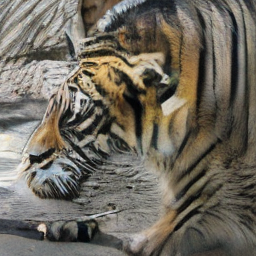}};
    \node at (3,3) {\sscell{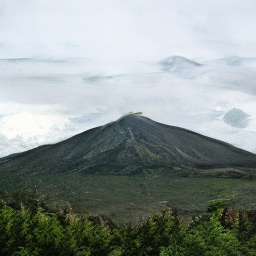}};
    \node at (4,3) {\sscell{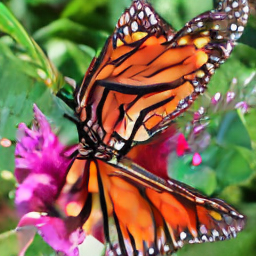}};
    \node at (5,3) {\sscell{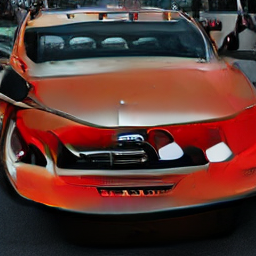}};
    \node at (0,4) {\sscell{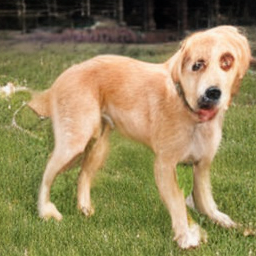}};
    \node at (1,4) {\sscell{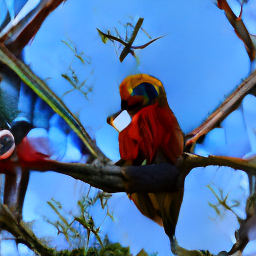}};
    \node at (2,4) {\sscell{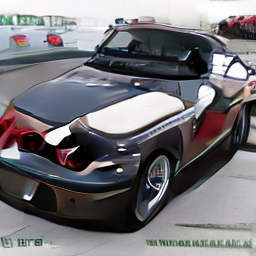}};
    \node at (3,4) {\sscell{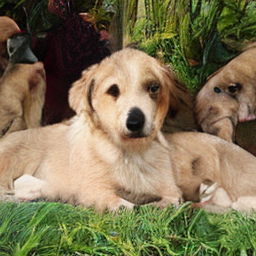}};
    \node at (4,4) {\sscell{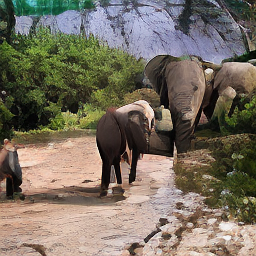}};
    \node at (5,4) {\sscell{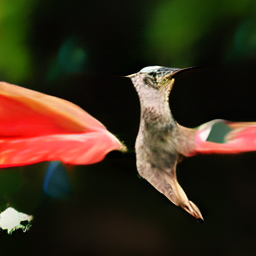}};
    \node at (0,5) {\sscell{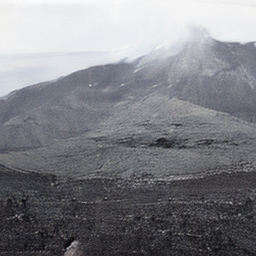}};
    \node at (1,5) {\sscell{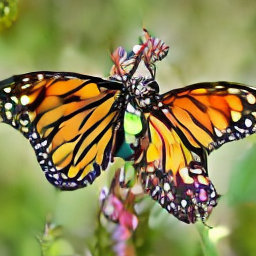}};
    \node at (2,5) {\sscell{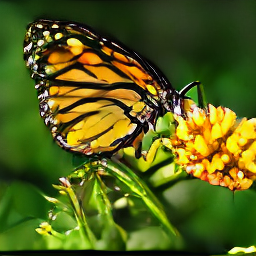}};
    \node at (3,5) {\sscell{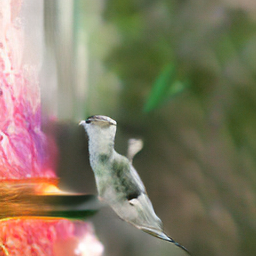}};
    \node at (4,5) {\sscell{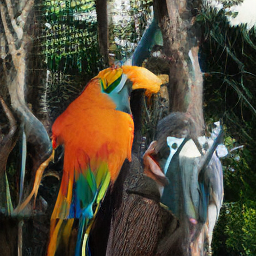}};
    \node at (5,5) {\sscell{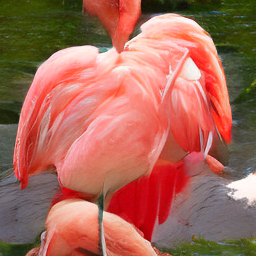}};
    \node at (0,6) {\sscell{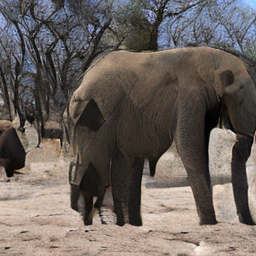}};
    \node at (1,6) {\sscell{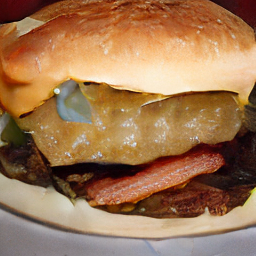}};
    \node at (2,6) {\sscell{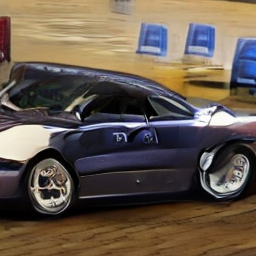}};
    \node at (3,6) {\sscell{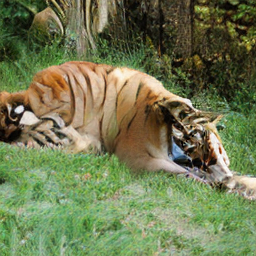}};
    \node at (4,6) {\sscell{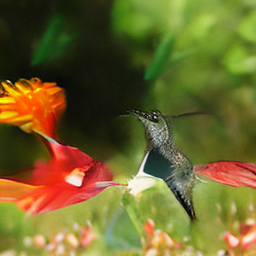}};
    \node at (5,6) {\sscell{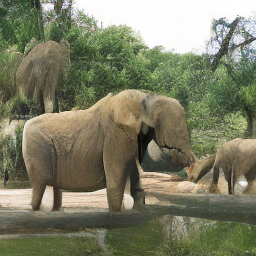}};
    \node at (0,7) {\sscell{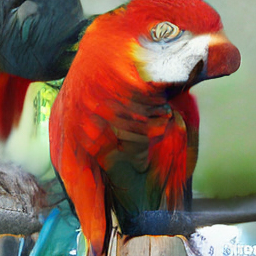}};
    \node at (1,7) {\sscell{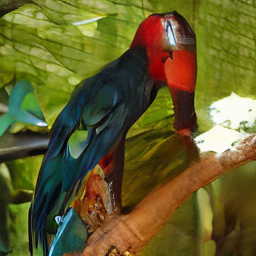}};
    \node at (2,7) {\sscell{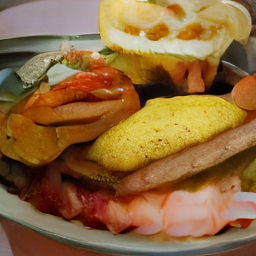}};
    \node at (3,7) {\sscell{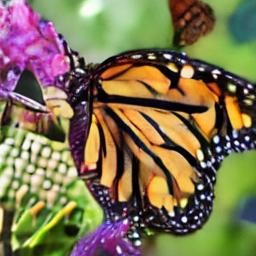}};
    \node at (4,7) {\sscell{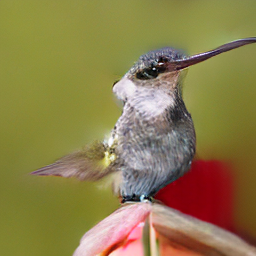}};
    \node at (5,7) {\sscell{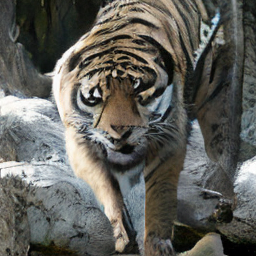}};
  \end{tikzpicture}
  \caption{Uncurated samples from the Gaussian-drift model at $30{,}000$ steps, CFG $w{=}1.833$. Each cell is an independently generated $256{\times}256$ image; samples from $10$ ImageNet classes (macaw, hummingbird, flamingo, golden retriever, tiger, monarch butterfly, African elephant, sports car, cheeseburger, volcano) are shuffled across the grid.}
  \label{fig:gaussian_panel}
\end{figure}

\begin{figure}[p]
  \centering
  \def\cs{2.3cm}
  \newcommand{\sscell}[1]{\includegraphics[width=\cs,height=\cs]{figures/annealed/#1.png}}
  \begin{tikzpicture}[
    every node/.style={inner sep=0pt, outer sep=0pt, anchor=north west},
    x=\cs, y=-\cs
  ]
    \node at (0,0) {\sscell{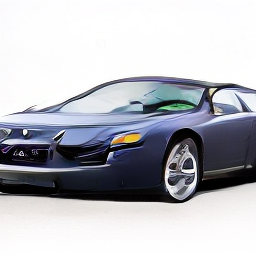}};
    \node at (1,0) {\sscell{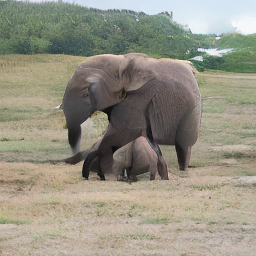}};
    \node at (2,0) {\sscell{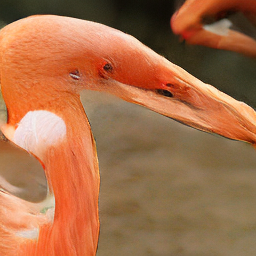}};
    \node at (3,0) {\sscell{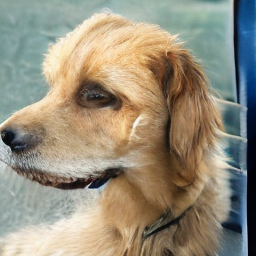}};
    \node at (4,0) {\sscell{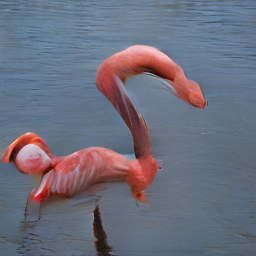}};
    \node at (5,0) {\sscell{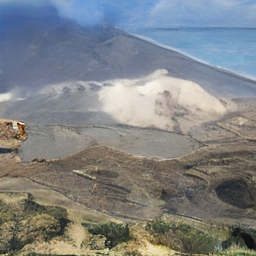}};
    \node at (0,1) {\sscell{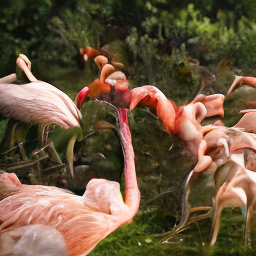}};
    \node at (1,1) {\sscell{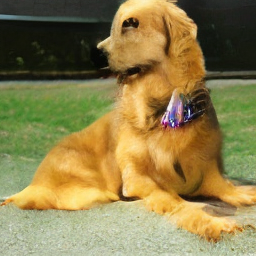}};
    \node at (2,1) {\sscell{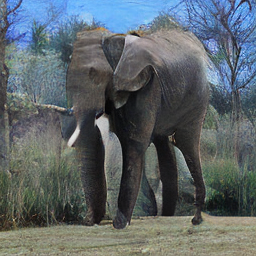}};
    \node at (3,1) {\sscell{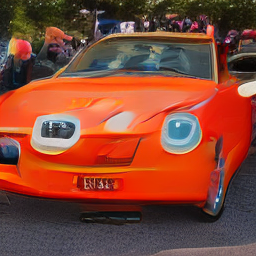}};
    \node at (4,1) {\sscell{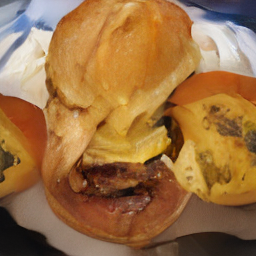}};
    \node at (5,1) {\sscell{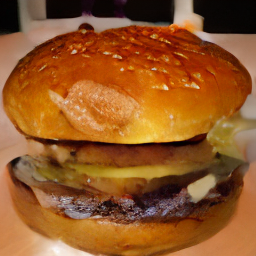}};
    \node at (0,2) {\sscell{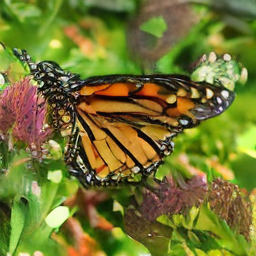}};
    \node at (1,2) {\sscell{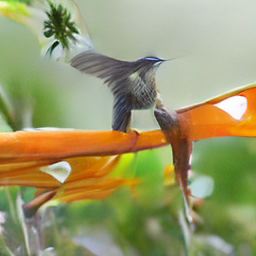}};
    \node at (2,2) {\sscell{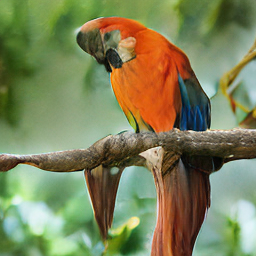}};
    \node at (3,2) {\sscell{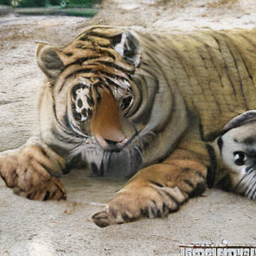}};
    \node at (4,2) {\sscell{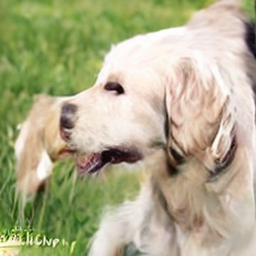}};
    \node at (5,2) {\sscell{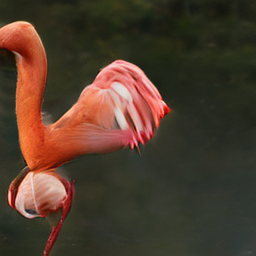}};
    \node at (0,3) {\sscell{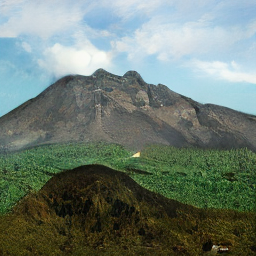}};
    \node at (1,3) {\sscell{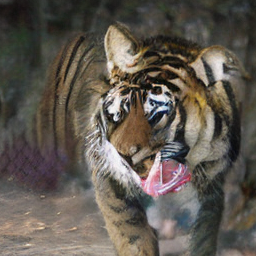}};
    \node at (2,3) {\sscell{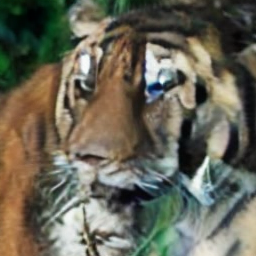}};
    \node at (3,3) {\sscell{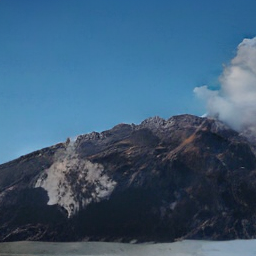}};
    \node at (4,3) {\sscell{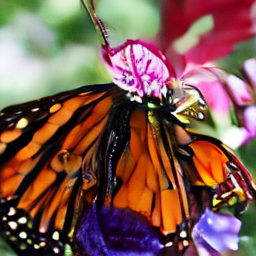}};
    \node at (5,3) {\sscell{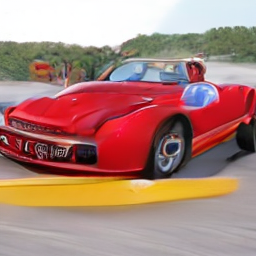}};
    \node at (0,4) {\sscell{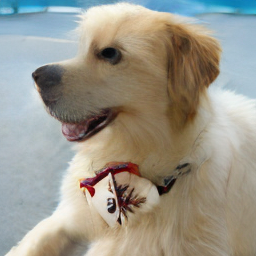}};
    \node at (1,4) {\sscell{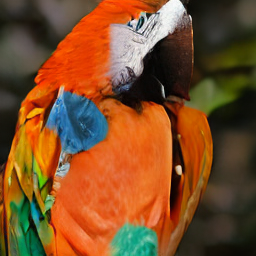}};
    \node at (2,4) {\sscell{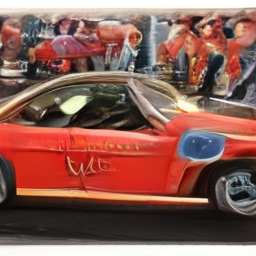}};
    \node at (3,4) {\sscell{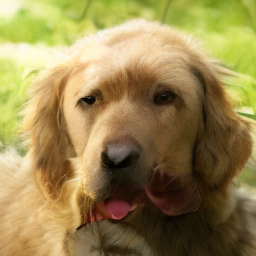}};
    \node at (4,4) {\sscell{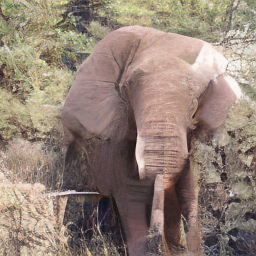}};
    \node at (5,4) {\sscell{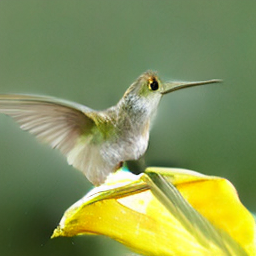}};
    \node at (0,5) {\sscell{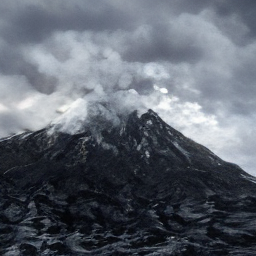}};
    \node at (1,5) {\sscell{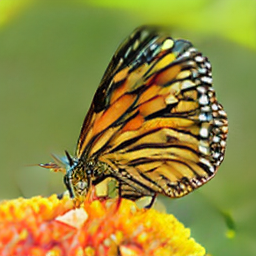}};
    \node at (2,5) {\sscell{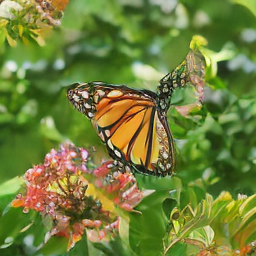}};
    \node at (3,5) {\sscell{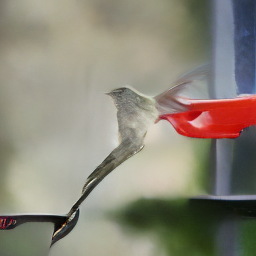}};
    \node at (4,5) {\sscell{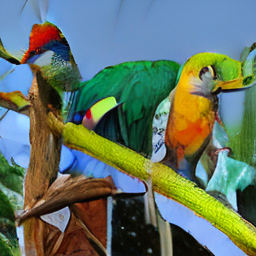}};
    \node at (5,5) {\sscell{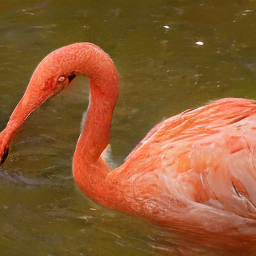}};
    \node at (0,6) {\sscell{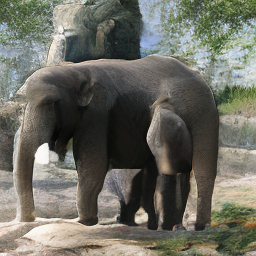}};
    \node at (1,6) {\sscell{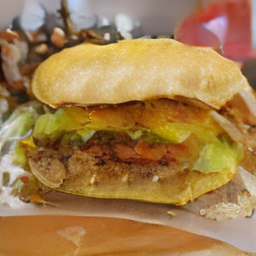}};
    \node at (2,6) {\sscell{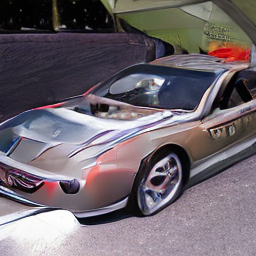}};
    \node at (3,6) {\sscell{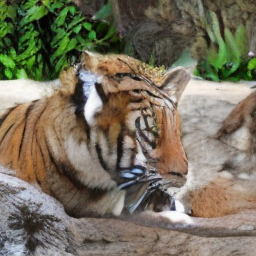}};
    \node at (4,6) {\sscell{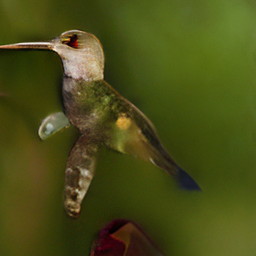}};
    \node at (5,6) {\sscell{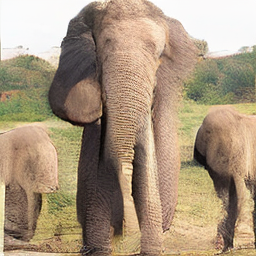}};
    \node at (0,7) {\sscell{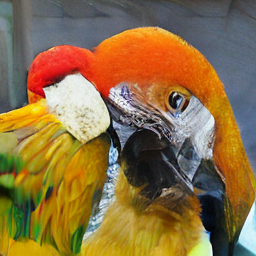}};
    \node at (1,7) {\sscell{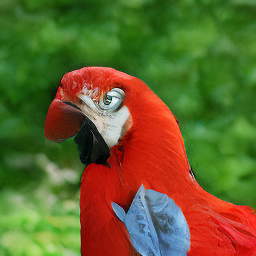}};
    \node at (2,7) {\sscell{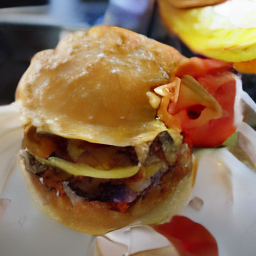}};
    \node at (3,7) {\sscell{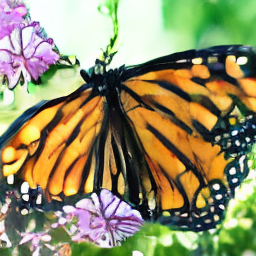}};
    \node at (4,7) {\sscell{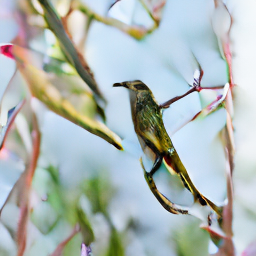}};
    \node at (5,7) {\sscell{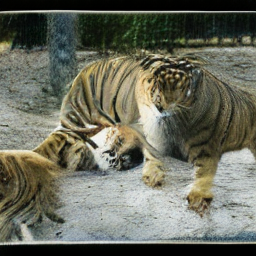}};
  \end{tikzpicture}
  \caption{Uncurated samples from the Gaussian-drift model with annealing at $30{,}000$ steps, CFG $w{=}1.944$. Each cell is an independently generated $256{\times}256$ image; samples from $10$ ImageNet classes (macaw, hummingbird, flamingo, golden retriever, tiger, monarch butterfly, African elephant, sports car, cheeseburger, volcano) are shuffled across the grid.}
  \label{fig:annealed_panel}
\end{figure}

\begin{figure}[p]
  \centering
  \def\cs{2.3cm}
  \newcommand{\sscell}[1]{\includegraphics[width=\cs,height=\cs]{figures/sinkhorn/#1.png}}
  \begin{tikzpicture}[
    every node/.style={inner sep=0pt, outer sep=0pt, anchor=north west},
    x=\cs, y=-\cs
  ]
    \node at (0,0) {\sscell{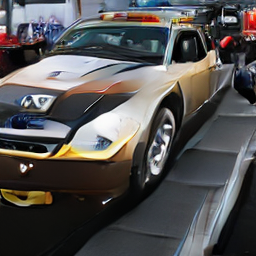}};
    \node at (1,0) {\sscell{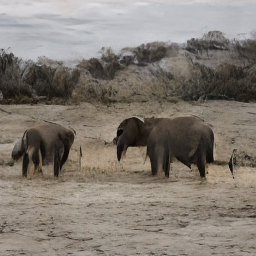}};
    \node at (2,0) {\sscell{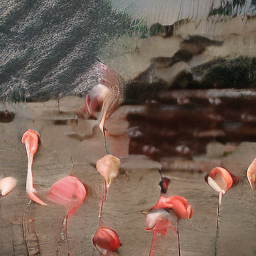}};
    \node at (3,0) {\sscell{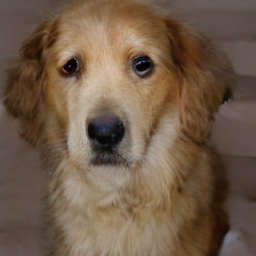}};
    \node at (4,0) {\sscell{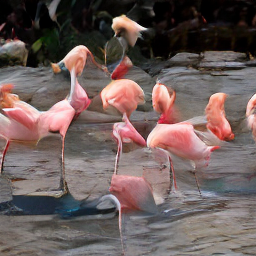}};
    \node at (5,0) {\sscell{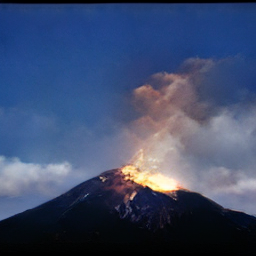}};
    \node at (0,1) {\sscell{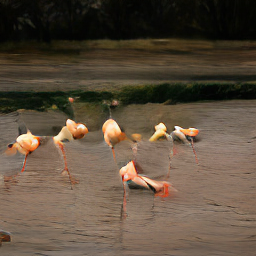}};
    \node at (1,1) {\sscell{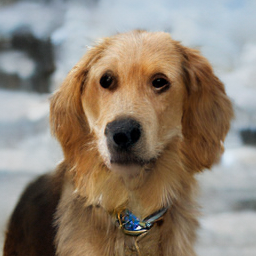}};
    \node at (2,1) {\sscell{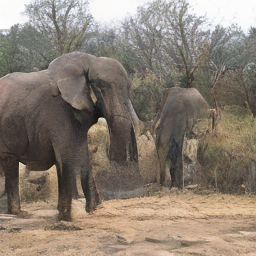}};
    \node at (3,1) {\sscell{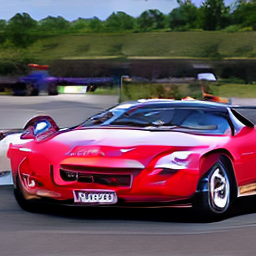}};
    \node at (4,1) {\sscell{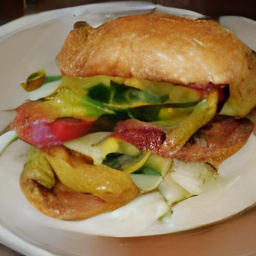}};
    \node at (5,1) {\sscell{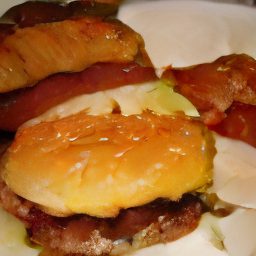}};
    \node at (0,2) {\sscell{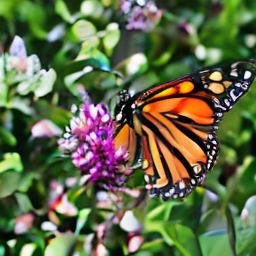}};
    \node at (1,2) {\sscell{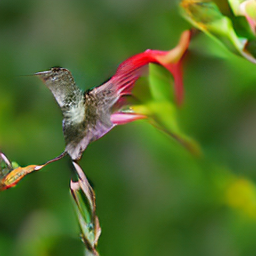}};
    \node at (2,2) {\sscell{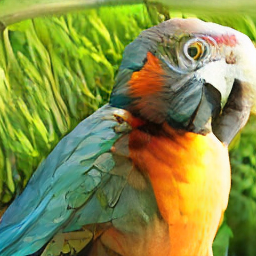}};
    \node at (3,2) {\sscell{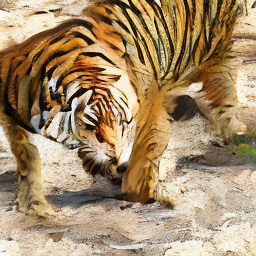}};
    \node at (4,2) {\sscell{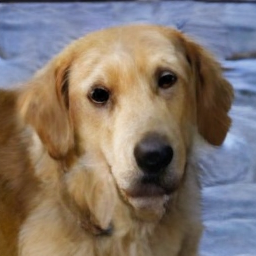}};
    \node at (5,2) {\sscell{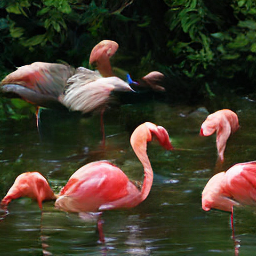}};
    \node at (0,3) {\sscell{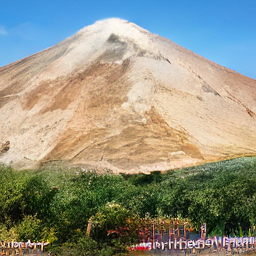}};
    \node at (1,3) {\sscell{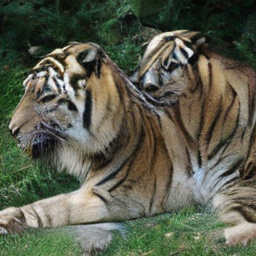}};
    \node at (2,3) {\sscell{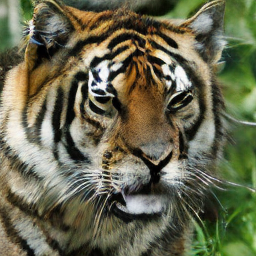}};
    \node at (3,3) {\sscell{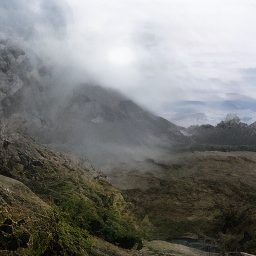}};
    \node at (4,3) {\sscell{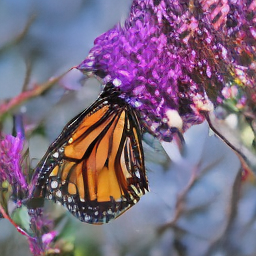}};
    \node at (5,3) {\sscell{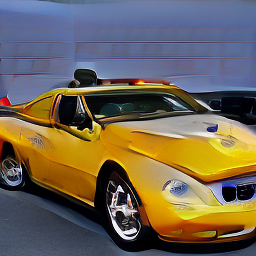}};
    \node at (0,4) {\sscell{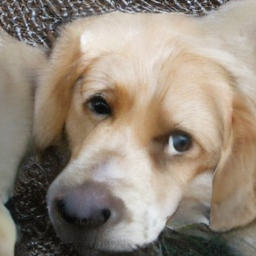}};
    \node at (1,4) {\sscell{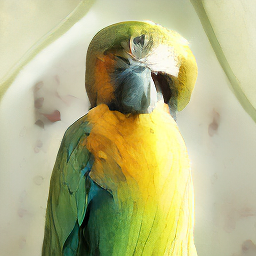}};
    \node at (2,4) {\sscell{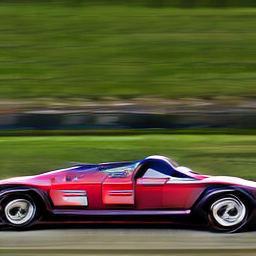}};
    \node at (3,4) {\sscell{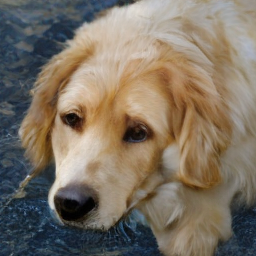}};
    \node at (4,4) {\sscell{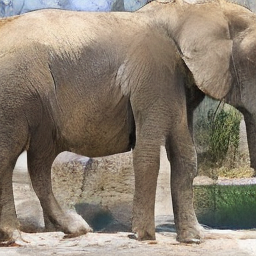}};
    \node at (5,4) {\sscell{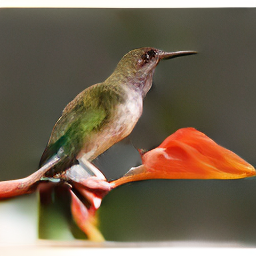}};
    \node at (0,5) {\sscell{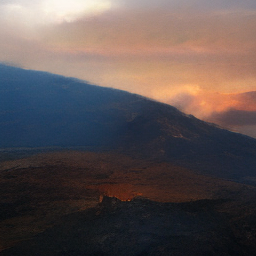}};
    \node at (1,5) {\sscell{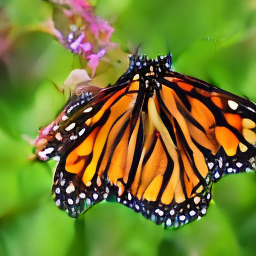}};
    \node at (2,5) {\sscell{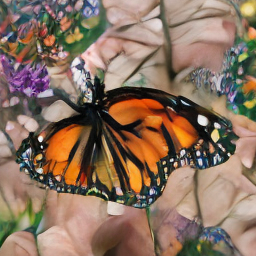}};
    \node at (3,5) {\sscell{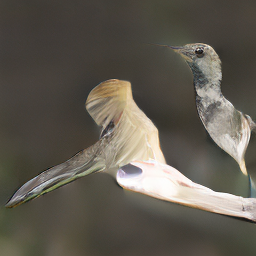}};
    \node at (4,5) {\sscell{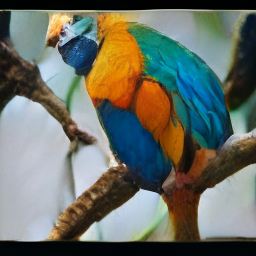}};
    \node at (5,5) {\sscell{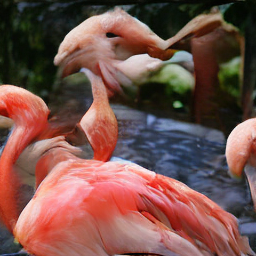}};
    \node at (0,6) {\sscell{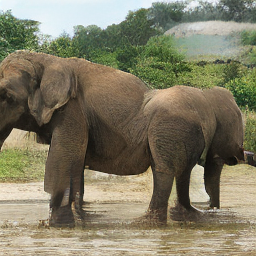}};
    \node at (1,6) {\sscell{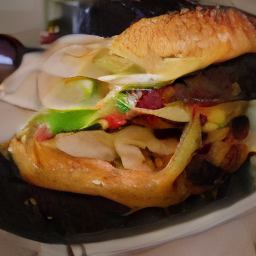}};
    \node at (2,6) {\sscell{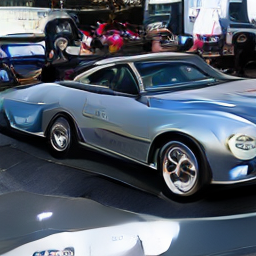}};
    \node at (3,6) {\sscell{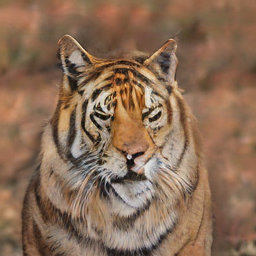}};
    \node at (4,6) {\sscell{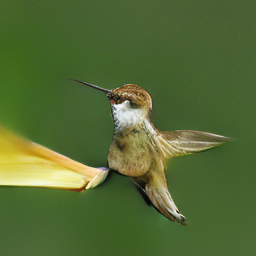}};
    \node at (5,6) {\sscell{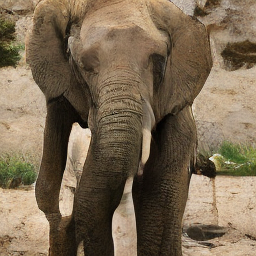}};
    \node at (0,7) {\sscell{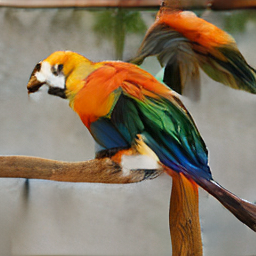}};
    \node at (1,7) {\sscell{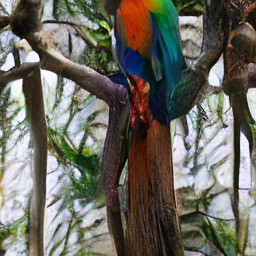}};
    \node at (2,7) {\sscell{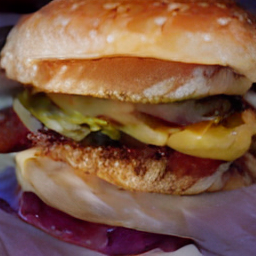}};
    \node at (3,7) {\sscell{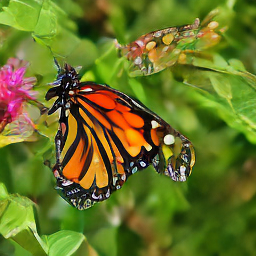}};
    \node at (4,7) {\sscell{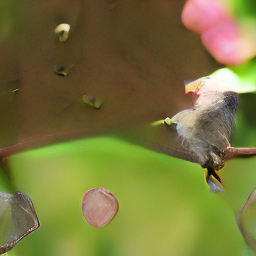}};
    \node at (5,7) {\sscell{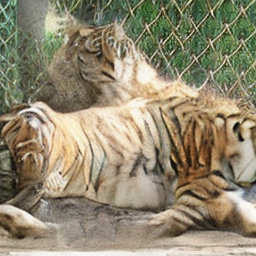}};
  \end{tikzpicture}
  \caption{Uncurated samples from the Sinkhorn-drift model at $30{,}000$ steps, CFG $w{=}1.167$. Each cell is an independently generated $256{\times}256$ image; samples from $10$ ImageNet classes (macaw, hummingbird, flamingo, golden retriever, tiger, monarch butterfly, African elephant, sports car, cheeseburger, volcano) are shuffled across the grid.}
  \label{fig:sinkhorn_panel}
\end{figure}

\subsection{Compute, memory, and implementation notes}
\label{app:image_exp:compute}

All ablations run on $16$ NVIDIA H100 GPUs. A full $30{,}000$-step ablation
takes roughly $10$--$11$ wall-clock hours depending on the loss (the
two-distribution Sinkhorn variant solves three Sinkhorn problems per step
when CFG is active). Checkpoints are evaluated at all
CFG scales; we report the best.

All drift variants share the same $\mathcal{O}(B^2 d)$ pairwise-distance cost; the Sinkhorn drift additionally runs $S$ Sinkhorn iterations on the entropic couplings $\pi_{qp}$ and $\pi_{qq}$. Table~\ref{tab:drift-compute} reports per-step wall-clock on a single H100, averaged over $100$ runs after $10$ warm-up iterations, for representative $(B, d)$ shapes.

\begin{table}[h!]
  \caption{Per-step drift wall-clock on a single H100, in milliseconds (mean $\pm$ std over $100$ iterations after $10$ warm-ups). The ``softmax'' column is the Laplacian-kernel softmax baseline; ``annealed (3R)'' is the Gaussian drift with the $3R$ exponential bandwidth schedule of Section~\ref{sec:annealing}. \textbf{Bold} marks the fastest variant in each row.}
  \label{tab:drift-compute}
  \centering
  \small
  \setlength{\tabcolsep}{8pt}
  \renewcommand{\arraystretch}{1.2}
  \begin{tabular}{@{}cc cccc@{}}
    \toprule
    \multicolumn{2}{c}{\textbf{Shape}} & \multicolumn{4}{c}{\textbf{Wall-clock per step (ms)}} \\
    \cmidrule(lr){1-2}\cmidrule(lr){3-6}
    $B$ & $d$ & softmax & Gaussian & annealed (3R) & Sinkhorn \\
    \midrule
    $64$    & $256$ & $\mathbf{0.761{\scriptstyle\,\pm\,0.020}}$ & $0.775{\scriptstyle\,\pm\,0.012}$ & $0.780{\scriptstyle\,\pm\,0.014}$ & $1.004{\scriptstyle\,\pm\,0.105}$ \\
    $256$   & $256$ & $1.214{\scriptstyle\,\pm\,0.004}$ & $1.265{\scriptstyle\,\pm\,0.005}$ & $1.261{\scriptstyle\,\pm\,0.004}$ & $\mathbf{1.042{\scriptstyle\,\pm\,0.011}}$ \\
    $1024$  & $256$ & $4.211{\scriptstyle\,\pm\,0.005}$ & $4.343{\scriptstyle\,\pm\,0.005}$ & $4.341{\scriptstyle\,\pm\,0.005}$ & $\mathbf{2.552{\scriptstyle\,\pm\,0.011}}$ \\
    \addlinespace[2pt]
    $64$    & $768$ & $\mathbf{0.899{\scriptstyle\,\pm\,0.005}}$ & $0.930{\scriptstyle\,\pm\,0.005}$ & $0.935{\scriptstyle\,\pm\,0.006}$ & $1.040{\scriptstyle\,\pm\,0.088}$ \\
    \bottomrule
  \end{tabular}
\end{table}

\paragraph{Analysis.}
Three observations stand out. (i) The three kernel-based variants (softmax, Gaussian, annealed Gaussian) are within $2$--$4\%$ of one another at every shape: the bandwidth schedule and kernel choice contribute negligible overhead on top of the shared $B^2 d$ pairwise-distance computation. (ii) The Sinkhorn drift carries a fixed iteration overhead from the $S$ Sinkhorn sweeps on the two coupling matrices, which dominates at small $B$: at $(B,d){=}(64,256)$ it is $\sim\!32\%$ slower than the softmax baseline. (iii) This overhead is rapidly amortized as $B$ grows. At $(256,256)$ Sinkhorn is already on par with the kernel variants ($1.04$ vs.\ $1.21$ ms), and at $(1024,256)$ it is in fact $\sim\!1.65{\times}$ \emph{faster} ($2.55$ vs.\ $4.21$ ms): the $B^2$ pairwise-distance and softmax-normalization passes that the kernel drifts execute eagerly are absorbed into the fused Sinkhorn iterations, whose constant factor is smaller. The picture is qualitatively unchanged at $d{=}768$: Sinkhorn pays a fixed-cost premium at $B{=}64$, and the gap is dominated by the iteration count rather than the embedding dimension. In summary, Sinkhorn is not a meaningful bottleneck at the batch sizes used in our image-generation experiments ($B \geq 256$), and is strictly cheaper than the kernel baselines once $B$ is large enough for the $B^2 d$ cost to dominate.

\end{document}

%% file: math_commands.tex
\usepackage{amsmath,amsfonts,bm}

\def\eqref#1{equation~\ref{#1}}

\def\1{\bm{1}}

\DeclareMathAlphabet{\mathsfit}{\encodingdefault}{\sfdefault}{m}{sl}
\SetMathAlphabet{\mathsfit}{bold}{\encodingdefault}{\sfdefault}{bx}{n}

\newcommand{\E}{\mathbb{E}}

\newcommand{\R}{\mathbb{R}}

\newcommand{\KL}{D_{\mathrm{KL}}}

\DeclareMathOperator*{\argmin}{arg\,min}